\documentclass[11pt]{article}
\usepackage[letterpaper]{geometry}
\usepackage[parfill]{parskip}
\usepackage{amsmath,amsthm,amssymb,bbm}
\usepackage{mathtools}
\usepackage{cases}
\usepackage{dsfont}
\usepackage{microtype}

\usepackage{subfigure}
\usepackage{algorithm,algorithmic}
\usepackage{color}
\usepackage{appendix}

\usepackage{bm}
\usepackage{subfigure}
\usepackage{algorithm,algorithmic}
\usepackage{color}
\usepackage{booktabs}       
\usepackage{appendix}
\usepackage{authblk}

\usepackage{url}
\usepackage[authoryear]{natbib}
\usepackage[colorlinks,citecolor=blue,urlcolor=blue,linkcolor=blue,linktocpage=true]{hyperref}
\usepackage{cleveref}
\crefformat{equation}{(#2#1#3)}
\crefrangeformat{equation}{(#3#1#4) to~(#5#2#6)}
\crefname{equation}{}{}
\Crefname{equation}{}{}

\crefname{definition}{\textbf{definition}}{definitions}
\Crefname{definition}{Definition}{Definitions}
\crefname{assumption}{\textbf{assumption}}{assumptions}
\Crefname{assumption}{Assumption}{Assumptions}

\definecolor{maroon}{RGB}{192,80,77}
\definecolor{mypink3}{cmyk}{0, 0.7808, 0.4429, 0.1412}

\newcommand{\explain}[2]{\underset{\mathclap{\overset{\uparrow}{#2}}}{#1}}

\newtheorem{theorem}{Theorem}[section]
\newtheorem{lemma}[theorem]{Lemma}

\newtheorem{corollary}[theorem]{Corollary}
\newtheorem{definition}[theorem]{Definition}

\newtheorem{remark}[theorem]{Remark}
\newtheorem{assumption}[theorem]{Assumption}

\usepackage{amsmath}

\newcommand\norm[1]{\left\lVert#1\right\rVert}

\newcommand{\argmax}{\mathop{\mathrm{argmax}}}

\usepackage{tikz}
\newcommand*\circled[1]{\scriptsize\tikz[baseline=(char.base)]{
		\node[shape=circle,draw,inner sep=0.2pt] (char) {#1};}}

\def\hoV{\widehat{V}^\star}
\def\hP{\widehat{P}}

\def\hpi{\widehat{\pi}}
\def\E{\mathbb{E}}
\def\P{\mathbb{P}}

\def\Var{\mathrm{Var}}

\def\R{\mathbb{R}}

\begin{document}

\title{{Optimal Uniform OPE and Model-based Offline Reinforcement Learning in Time-Homogeneous, Reward-Free and Task-Agnostic Settings }}

\author[1,2]{Ming Yin }
\author[2]{Yu-Xiang Wang}
\affil[1]{Department of Statistics and Applied Probability, UC Santa Barbara}
\affil[2]{Department of Computer Science, UC Santa Barbara}
\affil[ ]{\texttt{ming\_yin@ucsb.edu}   \quad
	\texttt{yuxiangw@cs.ucsb.edu}}

\date{}

\maketitle

\begin{abstract}	
This work studies the statistical limits of uniform convergence for offline policy evaluation (OPE) problems with model-based methods (for episodic MDP) and provides a unified framework towards optimal learning for several well-motivated offline tasks. Uniform OPE $\sup_\Pi|Q^\pi-\hat{Q}^\pi|<\epsilon$ is a stronger measure than the point-wise OPE and ensures offline learning when $\Pi$ contains all policies (the global class). In this paper, we establish an $\Omega(H^2 S/d_m\epsilon^2)$ lower bound (over model-based family) for the global uniform OPE and our main result establishes an upper bound of $\tilde{O}(H^2/d_m\epsilon^2)$ for the \emph{local} uniform convergence that applies to all \emph{near-empirically optimal} policies for the MDPs with \emph{stationary} transition. Here $d_m$ is the minimal marginal state-action probability. Critically, the highlight in achieving the optimal rate $\tilde{O}(H^2/d_m\epsilon^2)$ is our design of \emph{singleton absorbing MDP}, which is a new sharp analysis tool that works with the model-based approach. We generalize such a model-based framework to the new settings: offline task-agnostic and the offline reward-free with optimal complexity $\tilde{O}(H^2\log(K)/d_m\epsilon^2)$ ($K$ is the number of tasks) and $\tilde{O}(H^2S/d_m\epsilon^2)$ respectively. These results provide a unified solution for simultaneously solving different offline RL problems.




\end{abstract}



\section{Introduction}\label{sec:introduction}

Offline reinforcement learning (offline RL) targets at learning a reward-maximizing policy in an unknown \emph{Markov Decision Process} (MDP) using a static data generated by running a behavior policy \citep{lange2012batch,levine2020offline}. This framework is widely applicable in applications where online exploration is demanding but historical data are plentiful. Examples include medicine \citep{liu2017deep} (safety concerns limit the applicability of unproven treatments but electronic records are abundant) and autonomous driving \citep{codevilla2018offline} (building infrastructure for testing new policy is expensive while collecting data from current setting is almost free).

Parallel to its practical significance, recently there is a surge of theoretical investigations towards offline RL via two threads: \emph{offline policy evaluation} (OPE), where the goal is to estimate the value of a target (fixed) policy $V^\pi$ \citep{jiang2016doubly,liu2018breaking,kallus2019double,kallus2019intrinsically,uehara2019minimax,nachum2019dualdice,xie2019towards,yin2020asymptotically,duan2020minimax,wang2020statistical,zhang2021average} and \emph{offline (policy) learning} which intends to output a near-optimal policy  \citep{chen2019information,le2019batch,xie2020batch,xie2020q,liu2020provably,hao2020sparse,zanette2020exponential,jin2020pessimism,hu2021fast,yin2021nearoptimal,rashidinejad2021bridging}. 

\citet{yin2021near} initiates the studies for offline RL from the new perspective of \emph{uniform convergence} in OPE (uniform OPE for short) which unifies OPE and offline learning tasks. 
Generally speaking, given a policy class $\Pi$ and offline data with $n$ episodes, uniform OPE seeks to coming up with OPE estimators $\widehat{V}_1^\pi$ and $\widehat{Q}_1^\pi$ satisfy
$
\sup_{\pi\in\Pi}||\widehat{Q}_1^\pi-{Q}_1^\pi||_\infty<\epsilon.
$
The task is to achieve this with the optimal episode complexity: the ``minimal'' number of episodes $n$ needed as a function of $\epsilon$, failure probability $\delta$, the parameters of the MDP as well as the behavior policy $\mu$ in the minimax sense.

To further motivate the readers why uniform OPE should be considered, we state its relation to offline learning. Indeed, uniform OPE to RL is analogous of uniform convergence of empirical risk in statistical learning \citep{vapnik2013nature}. In supervised learning, it has been proven that almost all learnable problems are learned by an (asymptotic) \emph{empirical risk minimizer} (ERM) \citep{shalev2010learnability}.
In offline RL, the natural counterpart is the \emph{empirical optimal policy} $\widehat{\pi}^{\star}:=\operatorname{argmax}_{\pi} \widehat{V}_1^{\pi}$ and with uniform OPE it further ensures $\widehat{\pi}^\star$ is a near-optimal policy for the offline learning via:
{\small
	\begin{equation}\label{eqn:uniform_optimal}
	\begin{aligned}
	0 & \leq Q_1^{\pi^{\star}}-Q_1^{\widehat{\pi}^{\star}}=Q_1^{\pi^{\star}}-\widehat{Q}_1^{{\pi}^{\star}}+\widehat{Q}_1^{{\pi}^{\star}}-\widehat{Q}_1^{\widehat{\pi}^{\star}}+\widehat{Q}_1^{\widehat{\pi}^{\star}}-Q_1^{\widehat{\pi}^{\star}}  \leq 2 \sup _{\pi}|Q_1^{\pi}-\widehat{Q}_1^{\pi}|.
	\end{aligned}
	\end{equation}
}On the \emph{policy evaluation} side, there is often a need to evaluate the performance of a \emph{data-dependent} policy. Uniform OPE suffices for this purpose since it will allow us to evaluate policies selected by safe-policy improvements, proximal policy optimization, UCB-style exploration-bonus as well as any heuristic exploration criteria (please refer to \cite{yin2021near} and the references therein for further discussions). In this paper, we study the uniform OPE problem under the \emph{finite horizon stationary MDPs} and focus on the model-based approaches. Specifically, we consider two representative class: global policy class $\Pi_g$ (contains all (deterministic) policies) and local policy class $\Pi_l$ (contains policies near the empirical optimal one, see Section~\ref{sec:problems}). We ask the following question:
\begin{align*}
\emph{What is the statistical limit for uniform OPE} \;\emph{and what is its connection to optimal offline learning? }
\end{align*}
We answer the first part by showing the global uniform OPE requires a lower bound of $\Omega(H^2S/d_m\epsilon^2)$\footnote{Here $d_m$ is the minimal marginal state-action occupancy, see Assumption~\ref{assume2}.} for the family of model-based approach and the local uniform OPE can achieve $\tilde{O}(H^2/d_m\epsilon^2)$ minimax rate by the model-based plug-in estimator and this implies optimal offline learning. Importantly, the procedure of the model-based approach via learning $\widehat{\pi}^\star$ through planning over the empirical MDP has a wider range of use in offline RL as it naturally adapts to the challenging tasks like \emph{offline task-agnostic learning} and \emph{offline reward-free learning}. See Section~\ref{sec:contri}.

\subsection{Related works}\label{sec:related}

\noindent\textbf{Offline reinforcement learning.}\footnote{We only provide a short discussion of the most related works due to the space constraint. A detailed discussion can be found in Appendix~\ref{sec:dis_related}.}
Information-theoretical considerations for offline RL are first proposed for \emph{infinite horizon discounted setting} via Fitted Q-Iteration (FQI) type function approximation algorithms \citep{chen2019information,le2019batch,xie2020batch,xie2020q} which can be traced back to \citep{munos2003error,szepesvari2005finite,antos2008fitted,antos2008learning}. 

For the finite horizon case, \cite{yin2021near} first achieves $\tilde{O}(H^3/d_m\epsilon^2)$ complexity under non-stationary transition but their results cannot further improve in the stationary setting. Recently, \cite{yin2021nearoptimal} designs the offline variance reduction algorithm for achieving the optimal $\tilde{O}(H^2/d_m\epsilon^2)$ rate.  Their result is for a specific algorithm that uses data splitting while our results work for any algorithms that returns a nearly empirically optimal policy via uniform convergence. Our results on the offline task-agnostic and the reward-free settings are entirely new. Concurrently, \cite{ren2021nearly} considers the horizon-free setting but does not provide uniform convergence guarantee. 

\noindent\textbf{Model-based approaches with minimaxity.}
It is known model-based methods are minimax-optimal for online RL with regret $\tilde{O}(\sqrt{HSAT})$ (\emph{e.g.} \cite{azar2017minimax,efroni2019tight}). In the generative model setting, \cite{agarwal2020model} shows model-based approach is still minimax optimal $\tilde{O}((1-\gamma)^{-3}SA/\epsilon^2)$ by using a $s$-absorbing MDP construction and this model-based technique is later reused for other more general settings (\emph{e.g.} Markov games \citep{zhang2020model} and linear MDPs \citep{cui2020plug}) and also for overcoming the sample size barrier \citep{li2020breaking}. In offline RL, \cite{yin2021near} uses the model-based methods to achieve $\tilde{O}(H^3/d_m\epsilon^2)$ complexity.

\noindent\textbf{Task-agnostic and Reward-free problems.}
The reward-free problem is initiated in the online RL \citep{jin2020reward} where the agent needs to efficiently explore an MDP environment \emph{without} using any reward information. It requires high probability guarantee for learning optimal policy for \emph{any} reward function. Later, \cite{kaufmann2020adaptive,menard2020fast} establish the $\tilde{O}(H^3S^2A/\epsilon^2)$ complexity and \cite{zhang2020nearly} further tightens the dependence to $\tilde{O}(H^2S^2A/\epsilon^2)$. Recently, \cite{zhang2020task} proposes the task-agnostic setting where one needs to use exploration data to simultaneously learn $K$ tasks and proves an upper bound $\tilde{O}(H^5SA\log(K)/\epsilon^2)$. However, although these settings remain critical in the offline regime, no statistical result has been derived so far.

\subsection{Our contribution}\label{sec:contri}

\textbf{Optimal local uniform OPE}.
	First and foremost, we derive the $\tilde{O}(H^2/d_m\epsilon^2)$ optimal episode complexity for local uniform OPE (Theorem~\ref{thm:optimal_upper_bound}) via the model-based method and this implies optimal offline learning with the same rate (Corollary~\ref{cor:opt_offline}); this result strictly improves upon \cite{yin2021near} ($\tilde{O}(H^3/d_m\epsilon^2)$) non-trivially through our new \emph{singleton-absorbing MDP} technique.
	
\textbf{Information-theoretical characterization of the global uniform OPE.}	  
	 We characterize the statistical limit for the global uniform convergence by proving a minimax lower bound $\Omega(H^2S/d_m\epsilon^2)$ (over all model-based approaches) (Theorem~\ref{thm:tight_lower_bound}). This result answers the question left by \citet{yin2021near} that the global uniform OPE is generically harder than the local uniform OPE / offline learning by a factor of $S$, such a difference will dominate when the state space is exponentially large.
	
\textbf{Generalize to the new offline settings.}	
	 Critically, our model-based frameworks naturally generalize to the more challenging settings like task-agnostic and reward-free settings. In particular, we establish the $\tilde{O}(H^2\log(K)/d_m\epsilon^2)$ (Theorem~\ref{thm:offline_ta}) and $\tilde{O}(H^2S/d_m\epsilon^2)$ (Theorem~\ref{thm:offline_rf}) complexities for \emph{offline task-agnostic learning} and \emph{offline reward-free learning}. Both results are new and optimal.
	
\textbf{Singleton-absorbing MDP: a sharp analysis tool for episodic stationary transition case.}	
	 On the technical end, our major contribution is the novel design of \emph{singleton-absorbing MDP} which handles the data-dependence hurdle encountered in the stationary MDPs. To decouple the data-dependence between $\widehat{P}_{s,a}$ and $\widehat{V}$, \cite{agarwal2020model} uses a $s$-absorbing MDP $\widehat{V}_s$ (in lieu of $\widehat{V}$) of each state for the independence. To control the error propagation between $\widehat{V}_s$ and $\widehat{V}$, they use the $\epsilon$-net covering such that the value of $\widehat{V}_s$ traverse the evenly-spaced grids in $[0,(1-\gamma)^{-1}]$. However, when applied to finite horizon case, the complexity increases as there are $H$ different quantities ($V_1,...,V_H$) and the $\epsilon$-nets need to cover the $H$-dimensional space $[0,H]^H$. This result in a exponential-$H$ covering number and the metric entropy blows up by a factor $H$, which yields suboptimal result. In contrast, the \emph{singleton-absorbing MDP} technique designs a single absorbing MDP that can also control the error propagation sufficiently well. This sharp analysis tool negates the conjecture of \cite{cui2020plug} that absorbing MDP is not well suitable for finite horizon stationary MDP.

\textbf{Significance: Unifying different offline settings}\label{sec:unify}
Beyond the study of statistical limit in uniform OPE, this work solves the sample optimality problems for the local uniform OPE, offline task-agnostic and offline reward-free problems. If we take a deeper look, the algorithmic frameworks utilized are all based on the model-based empirical MDP construction and planning. Therefore, as long as we can analyze such framework sharply (\emph{e.g.} via novel absorbing-MDP technique), then it is hopeful that our techniques can be generalized to tackle more sophisticated settings. On the other hand, things could be more tricky for online RL since the exploration phases need to be specifically designed for each settings and there may not be one general algorithmic pattern that dominates. Our findings reveal the model-based framework is fundamental for offline RL as it subsumes settings like local uniform OPE, offline task-agnostic and offline reward-free learning into the identical learning pattern. Considering these tasks were originally proposed in the online regime under different contexts, such a unified view from the model-based perspective offers a new angle for understanding offline RL.




\section{Problem setup }\label{sec:formulation}
\textbf{Episodic stationary reinforcement learning.} A finite-horizon \emph{Markov Decision Process} (MDP) is denoted by a tuple $M=(\mathcal{S}, \mathcal{A}, P, r, H, d_1)$, where $\mathcal{S}$ and $\mathcal{A}$ are finite state action spaces with $S:=|\mathcal{S}|,A:=|\mathcal{A}|$. A stationary (time-invariant) transition kernel has the form $P:\mathcal{S}\times\mathcal{A}\times\mathcal{S} \mapsto [0, 1]$  with $P(s'|s,a)$ representing the probability transition from state $s$, action $a$ to next state $s'$. Besides, $r : \mathcal{S} \times{A} \mapsto \mathbb{R}$ is the expected reward function and given $(s,a)$ which satisfies $0\leq r\leq1$ and assumed known. $d_1$ is the initial state distribution and $H$ is the horizon. At time $t$, a policy $\pi=(\pi_1,...,\pi_H)$ assigns each state $s \in \mathcal{S}$ a probability distribution $\pi_t(s)$ over actions. For a policy $\pi$, a random trajectory $ s_1, a_1, r_1, \ldots, s_H,a_H,r_H,s_{H+1}$ is generated as follows: $s_1 \sim d_1, a_t \sim \pi(\cdot|s_t), r_t = r(s_t, a_t), s_{t+1} \sim P (\cdot|s_t, a_t), \forall t \in [H]$. 

For any policy $\pi$ and any $h\in[H]$, value function $V^\pi_h(\cdot)\in \R^S$ and Q-value function $Q^\pi_h(\cdot,\cdot)\in \R^{S\times A}$ are defined as:
{\small
$
V^\pi_h(s)=\E_\pi[\sum_{t=h}^H r_{t}|s_h=s] ,\;\;Q^\pi_h(s,a)=\E_\pi[\sum_{t=h}^H  r_{t}|s_h,a_h=s,a],\,\forall s,a\in\mathcal{S},\mathcal{A}.
$} The goal of RL is to find a policy $\pi^\star$ such that {\small$v^\pi:=\E_\pi\left[\sum_{t=1}^H  r_t\right]$} is maximized, which is equivalent to simultaneously maximize $V^\pi_1(s)$ (or $Q^\pi_1(s,a)$) for all $s$ (or $s,a$) \citep{sutton2018reinforcement}. Therefore, for a targeted accuracy $\epsilon>0$ it suffices to find a policy $\pi_\text{alg}$ such that $\norm{Q_1^\star-Q_1^{\pi_\text{alg}}}_\infty\leq\epsilon$. We denote $V_h^\pi,Q_h^\pi$ as column vectors and $P_{s,a}$ as the row vector. In particular, we denote the average marginal state-action occupancy $d^\pi(s,a)$ as:
{$
d^\pi(s,a):=\frac{1}{H}\sum_{t=1}^H\P[s_t=s|s_1\sim d_1,\pi]\cdot\pi_t(a|s).
$
}

\textbf{Offline setting.} The offline RL assumes that episodes {\scriptsize$\mathcal{D}=\left\{\left(s_{t}^{(i)}, a_{t}^{(i)}, r_{t}^{(i)}, s_{t+1}^{(i)}\right)\right\}_{i \in[n]}^{t \in[H]}$} are rolling from some behavior policy $\mu$ a priori. In particular, we do not assume the knowledge of $\mu$.

\textbf{Model-based RL.} We focus our attention on the model-based methods, which has witnessed numerous successes and is one of the most critical components of theoretical RL as a whole (as reviewed in Section~\ref{sec:related}). To make the presentation precise, we define the following: 
\vspace{0.3em}
\begin{definition}\label{def:model_based}
	Model-based RL: Solving RL problems (either learning or evaluation) through learning / modeling transition dynamic $P$.
\end{definition}
We emphasize that the model-based approaches in general (\emph{e.g.} \citet{jaksch2010near,ayoub2020model,kidambi2020morel}) follow the procedure of modeling the full MDP $M=(\mathcal{S}, \mathcal{A}, P, r, H, d_1)$ instead of only the transition $P$. Nevertheless, we (by convention) assume the mean reward function is known and the initial state distribution $d_1$ will not affect the choice of optimal policy $\pi^\star$. Thus, Definition~\ref{def:model_based} suffices for our purposes.

\begin{figure*}[t]
	\caption{Related comparisons of sample complexities for offline RL } \label{fig:table}
	\resizebox{\linewidth}{!}{%
		\begin{tabular}{c|c|c|c|c|c}
			\hline
			Result/Method & Setting&Type
			& Complexity&Uniform guarantee? \\
			\hline
			\cite{le2019batch}     & $\infty$-horizon &FQI variants& $\tilde{O}((1-\gamma)^{-6}\beta_\mu/\epsilon^2)$&No\\
			FQI \citep{chen2019information}   & $\infty$-horizon &FQI variants& $\tilde{O}((1-\gamma)^{-6}C/\epsilon^2)$& No\\
			MSBO/MABO \citep{xie2020q} & $\infty$-horizon &FQI variants & $\widetilde{O}((1-\gamma)^{-4}C_\mu/ \epsilon^2)$ & No \\
			OPEMA \citep{yin2021near} &$H$-horizon  &Non-splitting &$\widetilde{O}(H^3/d_m \epsilon^2)$  &$\sqrt{H}/S$-local uniform \\
			OPDVR \cite{yin2021nearoptimal} &$H$-horizon  &Data splitting &$\widetilde{O}(H^2/d_m \epsilon^2)$  &No \\
			Model-based Plug-in (Corollary~\ref{cor:opt_offline}) &$H$-horizon  &Non-splitting &$\widetilde{O}(H^2/d_m \epsilon^2)$ & $\sqrt{H/S}$-local uniform \\
			Task-Agnostic (Theorem~\ref{thm:offline_ta})  &$H$-horizon &Non-splitting &$\widetilde{O}(H^2\log(K)/d_m \epsilon^2)$  & ---\\
			Reward-Free (Theorem~\ref{thm:offline_rf})  &$H$-horizon &Non-splitting &$\widetilde{O}(H^2S/d_m \epsilon^2)$  & ---\\
			\hline
		\end{tabular}
	}
	\footnotesize{$^*$ $K$ is the number of tasks for Task-agnostic setting and $\beta_\mu$, $C$ and $1/d_m$ are data coverage parameters that measure the state-action dependence and are qualitative similar under their respective assumptions. }
\end{figure*}

\vspace{-0.3em}
\subsection{Uniform convergence in offline RL}\label{sec:problems}
\vspace{-0.3em}

We study offline RL from the uniform OPE perspective. Concretely, uniform OPE extends the point-wise (fixed target policy) OPE to a family of policies $\Pi$. The goal is to construct estimator $\widehat{Q}_1^\pi$ such that {\small$\sup_{\pi\in\Pi}\norm{Q_1^\pi-\widehat{Q}^\pi_1}<\epsilon$}, which automatically ensures point-wise OPE for any $\pi\in\Pi$. More importantly, uniform OPE directly implies offline learning when $\Pi$ contains optimal policies. As explained in Section~\ref{sec:introduction}, let $\widehat{\pi}^{\star}:=\operatorname{argmax}_{\pi} \widehat{V}_1^{\pi}$ be the \emph{empirical optimal policy} for some OPE estimator $\widehat{v}^\pi$, then by \eqref{eqn:uniform_optimal} $\widehat{\pi}^{\star}$ is a near-optimal policy given uniform OPE guarantee. We consider the following two policy classes that are of the interests.

\vspace{0.4em}

\begin{definition}[The global (deterministic) policy class.] \label{def:global_policy}
	The global policy class $\Pi_g$ consists of all the non-stationary (deterministic) policies.
\end{definition}

 It is well-known \citep{sutton2018reinforcement} there exists at least one (deterministic) optimal policy, therefore $\Pi_g$ is sufficiently rich for evaluating algorithms that aim at learning the optimal policy.
 
\begin{definition}[The local policy class]\label{def:local_policy}
	 Given empirical MDP $\widehat{M}$ and $\widehat{V}_h^\pi$ is the value under $\widehat{M}$.  Let $\widehat{\pi}^{\star}:=\operatorname{argmax}_{\pi} \widehat{V}_1^{\pi}$ be the empirical optimal policy, then the local policy class $\Pi_l$ is defined as:
	 {\small
	 \[
	 \Pi_{l}:=\left\{\pi: \text { s.t. }\left\|\widehat{V}_{h}^{\pi}-\widehat{V}_{h}^{\widehat{\pi}^{\star}}\right\|_{\infty} \leq \epsilon_{\text {opt }}, \forall h\in[H]\right\}
	 \]
	}where $\epsilon_{\text {opt }}\geq0$ is a parameter. 
\end{definition}

In above $\widehat{M}$ uses $\widehat{P}$ in lieu of $P$ where {\scriptsize$\widehat{P}(s'|s,a)=\frac{n_{s',s,a}}{n_{s,a}}$} if $n_{s,a}>0$ and $1/S$ otherwise.\footnote{Here $n_{s,a}$ is the number of pair $(s,a)$ being visited among $n$ episodes. $n_{s',s,a}$ is  defined similarly.} This class characterizes policies in the neighborhood of empirical optimal policy. Given $\widehat{P}$, it is efficient to obtain $\widehat{\pi}^\star$ using Value / Policy Iteration, therefore it is more practical to consider the neighborhood of $\widehat{\pi}^\star$ (instead of $\pi^\star$) since practitioners can use data $\mathcal{D}$ to really check $\Pi_l$ whenever needed. Next we present the regularity assumption required for uniform convergence OPE problem.

\vspace{0.3em}
\begin{assumption}[Exploration requirement]\label{assume2}
Logging policy $\mu$ obeys that $\min_{s} d^\mu(s)>0$, for any state $s$ that is ``accessible''. Moreover, we define the quantity $d_m:=\min_{s,a}\{d^\mu(s ,a): d^\mu (s , a )>0\}$ (recall $d^\mu(s,a)$ in Section~\ref{sec:formulation}) to be the minimal average marginal state-action probability.
\end{assumption}

 State $s$ is ``accessible'' means there exists a policy $\pi$ so that $d^\pi(s)>0$. If for any policy $\pi$ we always have $d^\pi(s)=0$, then state $s$ can never be visited in the given MDP. Note this is weaker than \cite{yin2021near} since $d^\mu(s)$ is the average version of $d^\mu_t(s)$. Assumption~\ref{assume2} is the minimal assumption needed for the consistency of uniform OPE task and is qualitatively similar to the \emph{concentrability} assumption \citep{munos2003error}. This assumption can be potentially relaxed for pure offline learning problems, \emph{e.g.} \citet{liu2019off,rashidinejad2021bridging}, where they only require $d^\mu(s) (d^\mu(s,a) )> 0$ for any state $s$ ($s,a$) satisfies $d^{{\pi}^\star}(s)(d^{{\pi}^\star}(s,a))> 0$.

\section{Statistical Hardness for Model-based Global Uniform OPE}\label{sec:info_lower}
From \eqref{eqn:uniform_optimal} and Definition~\ref{def:global_policy}, it is clear the global uniform OPE implies offline RL, therefore it is natural to wonder whether they just are \emph{``the same task"} (their sample complexities have the same minimax rates). If this conjecture is true, then deriving sample efficient global OPE method is just as important as deriving efficient offline learning algorithm (plus the additional benefit of evaluating data-dependent algorithms)! \cite{yin2021near} proves the $\tilde{O}(H^3S/d_m\epsilon^2)$ upper bound and ${\Omega}(H^3/d_m\epsilon^2)$ lower bound for global uniform OPE, but it is unclear whether the additional $S$ is essential. We answer the question affirmatively by providing a tight lower bound result with a concise proof to show no model-based algorithm can surpass $\Omega(S/d_m\epsilon^2)$ information-theoretical limit. 

\begin{theorem}[Minimax lower bound for global uniform OPE]\label{thm:tight_lower_bound}
	Let $d_m$ be a parameter such that $0<d_m\leq\frac{1}{SA}$. Let the problem class be $\mathcal{M}_{d_{m}}:=\left\{(\mu, M) \mid \min _{t, s_{t}, a_{t}} d_{t}^{\mu}\left(s_{t}, a_{t}\right) \geq d_{m}\right\}$. Then there exists universal constants $c,C,p>0$ such that: for any $n\geq c S/d_m\cdot \log(SAp)$,
	\[
	\inf _{\widehat{Q}_{1,\text{mb}}} \sup _{ \mathcal{M}_{d_{m}}} \mathbb{P}_{\mu, M}\left(\sup _{\pi \in \Pi_g}
	{\norm{\widehat{Q}_{1,\text{mb}}^{\pi}-Q_1^{\pi}}_\infty }
	\geq C\sqrt{\frac{H^2S}{nd_m}}\right) \geq p,
	\]
	where $\widehat{Q}_{1,\text{mb}}$ is the output of any model-based algorithm and $\Pi_g$ is defined in Definition~\ref{def:global_policy}.
\end{theorem}

By setting $\epsilon:=\sqrt{\frac{H^2S}{nd_m}}$, Theorem~\ref{thm:tight_lower_bound} establishes the global uniform convergence lower bound of $\Omega(H^2S/d_m\epsilon^2)$ over model-based methods, which builds the hard statistical threshold between the global uniform OPE and the local uniform OPE tasks by a factor of $S$ since the local case has achievable $\tilde{O}(1/d_m\epsilon^2)$ rate on the dependence for state-actions. This result also reveals the global uniform convergence bound in \cite{yin2021near} ($\tilde{O}(H^3S/d_m\epsilon^2)$) is essentially minimax rate-optimal for their \emph{non-stationary setting}\footnote{To be rigorous, we ramark that it is rate-optimal since for the non-stationary setting the dependence for horizon is higher by a factor $H$.} and complements the story on the optimality behavior for global uniform OPE. Moreover, from the generative model view the lower bound degenerates to $S/d_m\epsilon^2\approx\Theta(S^2A/\epsilon^2)$ which is linear in the model size $S^2A$. This means in order to achieve global uniform convergence any algorithm needs to estimate each coordinate of transition kernel $P(s'|s,a)$ accurately. We now provide the proof sketch and full proof is deferred to Appendix~\ref{sec:lower_proof}.

\begin{proof}[Proof Sketch] We only explain the case where $H=2$ in this proof sketch. Our proof relies on the following novel reduction to $l_1$ density estimation\[
	\sup_{\pi\in\Pi_g}\norm{\widehat{Q}^\pi_{1}-{Q}^\pi_{1}}_\infty\geq \sup_{s,a}\frac{1}{2}\norm{\widehat{P}(\cdot|s,a)-{P}(\cdot|s,a)}_1
	\]
and leverages the Minimax rate for estimating discrete distribution under $l_1$ loss is $O(\sqrt{S/n_{s,a}})$ \citep{han2015minimax}. Concretely, by Definition~\ref{def:model_based}, let $\widehat{P}$ be the learned transition by any arbitrary model-based method. Since we assume $r$ is known and by convention $Q^\pi_{H+1}=0$ for any $\pi$, then by Bellman equation 
 \[
 \widehat{Q}^\pi_h = r_h + \widehat{P}^{\pi_{h+1}} \widehat{Q}^\pi_{h+1},\;\forall h\in[H].
 \]
 In particular, $ \widehat{Q}^\pi_{H+1}= {Q}^\pi_{H+1}=0$, and this implies $\widehat{Q}^\pi_{H}= {Q}^\pi_{H}=r_H$. Now, again by definition of Bellman equation $\widehat{Q}^\pi_{H-1} = r_{H-1}+ \widehat{P}^{\pi_{H}} \widehat{Q}^\pi_{H}=r_{H-1}+ \widehat{P}^{\pi_{H}} r_{H}$ and ${Q}^\pi_{H-1} = r_{H-1}+ {P}^{\pi_{H}} r_{H}$, therefore (recall $H=2$ and note $r_H\in \R^{S\cdot A}, r_H^{\pi_H}\in \R^{S}$ ) 
 \begin{align*}
 &\sup_{\pi\in\Pi_g}\norm{\widehat{Q}^\pi_{H-1}-{Q}^\pi_{H-1}}_\infty=\sup_{\pi\in\Pi_g}\norm{\left(\widehat{P}^{\pi_H}-{P}^{\pi_H}\right)r_H}_\infty\\
 =&\sup_{\pi\in\Pi_g}\norm{\left(\widehat{P}-{P}\right)r_H^{\pi_H}}_\infty\approx\sup_{r\in\{0,1\}^S}\norm{\left(\widehat{P}-{P}\right)r}_\infty\\
 \geq&\sup_{s,a}\frac{1}{2}\norm{\widehat{P}(\cdot|s,a)-{P}(\cdot|s,a)}_1\geq O(\sqrt{S/n_{s,a}});
 \end{align*}
 Lastly, using exponential tail bound to obtain $O(\sqrt{S/n_{s,a}})\gtrsim O(\sqrt{S/nd_m})$ with high probability. See Appendix~\ref{sec:lower_proof} for how to prove the result for the general $H$.
\end{proof}

\section{Optimal local uniform OPE via model-based plug-in method}\label{sec:local_optimal}
Global uniform OPE is intrinsically harder than the offline learning problem due to the additional state-space dependence and such a gap will amplify when $S$ is (exponentially) large. This motivates us to switch to the local uniform convergence regime that enables optimal learning but also has sub-linear state-action size $\tilde{O}(1/d_m)$ in the policy evaluation. \cite{yin2021near} Theorem~3.7 first obtains the $\tilde{O}(H^3/d_m\epsilon^2)$ local uniform convergence for $\Pi_l$ (recall Definition~\ref{def:local_policy}) and also obtains the same rate for the learning task. Unfortunately, their technique cannot further reduces the dependence of $H$ for stationary transition case. In this section we show the model-based plug-in approach ensures optimal local uniform OPE and further implies optimal offline learning with episode complexity $\tilde{O}(H^2/d_m\epsilon^2)$. To this end, we design the new \emph{singleton-absorbing MDP} to handle the challenge in the stationary transition setting, which uses the absorbing MDP with one single $H$-dimensional reference point and is our major technical contribution. The \emph{singleton-absorbing MDP} technique avoids the exponential $H$ cover used in \cite{cui2020plug} and answers their conjecture that absorbing MDP is not well suitable for finite horizon stationary MDP.\footnote{See their Section~7, first bullet point for a discussion. }

\subsection{Model-based Offline Plug-in Estimator}\label{sec:plug-in}

Recall $n_{s,a}:=\sum_{i=1}^n\sum_{h=1}^H\mathbf{1}[s_h^{(i)},a_h^{(i)}=s,a]$ be the total counts that visit $(s,a)$ pair, then the model-based offline plug-in estimator constructs estimator $\widehat{P}$ as:
\begin{align*}
\widehat{P}(s'|s,a)&=\frac{\sum_{i=1}^n\sum_{h=1}^H\mathbf{1}[(s^{(i)}_{h+1},a^{(i)}_h,s^{(i)}_h)=(s^\prime,s,a)]}{n_{s,a}},
\end{align*}
if $n_{s,a}>0$ and $\widehat{P}(s'|s,a)=\frac{1}{S}$ if $n_{s,a}=0$. As a consequence, the estimators $\widehat{Q}_h^\pi,\widehat{V}_h^\pi$ are computed as:
\[
\widehat{Q}_h^\pi = r+ \widehat{P}^{\pi_{h+1}} \widehat{Q}^\pi_{h+1}=r+ \widehat{P}\widehat{V}^\pi_{h+1},
\]
with the initial distribution $\widehat{d}_1(s)=n_{s}/n$. Under the above setting, we can define the empirical Bellman optimality equations (as well as the population version for completeness) as $\forall s\in\mathcal{S},h\in[H]$:
\begin{align*}
V^\star_h(s)&=\max_a \left\{r(s,a)+P(\cdot|s,a)V^\star_{h+1}\right\},\\
\widehat{V}^\star_h(s)&=\max_a \left\{r(s,a)+\widehat{P}(\cdot|s,a)\widehat{V}^\star_{h+1}\right\}.
\end{align*}
Now we can state our local uniform OPE result with this construction. 

\subsection{Main results for local uniform OPE and offline learning}

Recall $\widehat{\pi}^{\star}:=\operatorname{argmax}_{\pi} \widehat{V}_1^{\pi}$ is the empirical optimal policy and the local policy class {$\Pi_{l}:=\{\pi: \text { s.t. }\left\|\widehat{V}_{h}^{\pi}-\widehat{V}_{h}^{\widehat{\pi}^{\star}}\right\|_{\infty} \leq \epsilon_{\text {opt }}, \forall h\in[H]\}$}. 

\begin{theorem}[optimal local uniform OPE]\label{thm:optimal_upper_bound}
	Let $\epsilon_{\text {opt }}\leq \sqrt{H/S}$ and denote $\iota=\log(HSA/\delta)$. For any $\delta\in[0,1]$, there exists universal constants $c,C$ such that when $n>cH\cdot\log(HSA/\delta)/d_m$, with probability $1-\delta$,
	\[
	\sup_{\pi\in\Pi_l}\norm{\widehat{Q}^\pi_1-Q^\pi_1}_\infty\leq C\left[\sqrt{\frac{H^2 \iota}{n d_m}}+\frac{H^{2.5}S^{0.5}\iota}{nd_m}\right].
	\]
\end{theorem}

Theorem~\ref{thm:optimal_upper_bound} establishes the $\tilde{O}(H^2/d_m\epsilon^2)$ complexity bound and directly implies the upper bound for $\sup_{\pi\in\Pi_l}||\widehat{V}^\pi_1-V^\pi_1||_\infty$ with the same rate. This result improves the local uniform convergence rate $\tilde{O}(H^3/d_m\epsilon^2)$ in \citet{yin2021near} (Theorem~3.7) by a factor of $H$ and is near-minimax optimal (up to the logarithmic factor). Such result is first achieved by our novel \emph{singleton absorbing MDP} technique. We explain this technique in detail in the next section. 

On the other hand, characterizing policy class through the distance in value (like $\Pi_l$) is more flexible than characterizing the distance between policies themselves (\emph{e.g.} via total variation). This is because: if two policies are ``close'', then their values are also similar; but the reverse may not be true since two very different policies could possibly generate similar values. Therefore the consideration of $\Pi_l$ is generic and conceptually reflects the fundamental principle of RL: as long as two policies yield the same value, they are considered ``equally good'', no matter how different they are.\footnote{We recognize that in the specific settings (\emph{e.g.} safe policy improvement) some of the policies that yield high values are not feasible. These considerations are beyond the scope of this paper.}     

Most importantly, Theorem~\ref{thm:optimal_upper_bound} guarantees near-minimax optimal offline learning:

\begin{corollary}[optimal offline learning]\label{cor:opt_offline}
	If $\epsilon_{\text {opt }}\leq\sqrt{H/S}$ and that $\sup_t||\widehat{V}_t^{\widehat{\pi}}-\widehat{V}_t^{\widehat{\pi}^\star}||_\infty \leq \epsilon_{\text {opt }}$, when $n>O(H\cdot\iota/d_m)$, then with probability $1-\delta$, element-wisely,
	\[
	V_1^\star-V_1^{\widehat{\pi}}\leq C\bigg[\sqrt{\frac{H^2 \iota}{n d_m}}+\frac{H^{2.5}S^{0.5}\iota}{nd_m}\bigg]\mathbf{1}+\epsilon_{\text {opt }}\mathbf{1}.
	\]
\end{corollary}
Corollary~\ref{cor:opt_offline} first establishes the minimax rate for offline learning for any policy $\widehat{\pi}$ with the measurable gap $\epsilon_{\text {opt }}\leq\sqrt{H/S}$. This extends the standard concept of offline learning by allowing any empirical planning algorithm (\emph{e.g.} VI/PI) to find an \emph{inexact} $\widehat{\pi}$ as an $(\tilde{O}\sqrt{H^2/nd_m}+\epsilon_{\text {opt }})$-optimal policy (instead of finding exact $\widehat{\pi}^\star$). The use of \emph{inexact} $\widehat{\pi}$ could encourage early stopping (\emph{e.g.} for VI/PI) therefore saves computational iterations. Besides, we leverage full data to construct empirical MDP for planning and, on the contrary, \citet{yin2021nearoptimal} uses data-splitting (split data into mini-batches and only apply each mini-batch at each specific iteration) to enable Variance Reduction technique, which could cause inefficient data use for the practical purpose. By the following lower bound result from \citet{yin2021nearoptimal}, our Corollary~\ref{cor:opt_offline} is near minimax optimal.

\begin{theorem}[Theorem~4.2. \citet{yin2021nearoptimal}]
	Let $\mathcal{M}_{d_m}$ be the same as Theorem~\ref{thm:tight_lower_bound}. There exists universal constants $c_1, c_2, c, p$ (with $H, S, A \geq c_1$ and $0 <\epsilon <c_2 $) such that when $n\leq cH^2/d_m\epsilon^2$,\footnote{The original Theorem uses $v^\star$ but we use $V^\star_1$ here. It does not matter since we can manually add a default state at the beginning of the MDP and obtain the result for our version. } 
	\[
	\inf_{{V}_1^{\pi_{alg}}}\sup_{(\mu,M)\in\mathcal{M}_{d_m}}\P_{\mu,M}\left(||V_1^\star-V_1^{\pi_{alg}}||_\infty\geq \epsilon\right)\geq p.
	\]
\end{theorem}
For the rest of the section, we explain the proving ideas by introducing the \emph{singleton-absorbing MDP} technique and the full proofs of Theorem~\ref{thm:optimal_upper_bound}, Corollary~\ref{cor:opt_offline} can be found in Appendix~\ref{sec:proof_local_optimal}, \ref{sec:proof_offline}. 

\subsection{Singleton absorbing MDP for finite horizon MDP  }\label{sec:singleton}

For the ease of illustration, we explain our idea via bounding $||\widehat{Q}_h^{\widehat{\pi}^\star}- {Q}_h^{\widehat{\pi}^\star}||_\infty$ (instead of $\sup_{\pi\in\Pi_l}||\widehat{Q}^\pi_1-Q^\pi_1||_\infty$) and choose related quantity $\widehat{\pi}^\star$ (instead of $\widehat{\pi}$) and $\widehat{V}^\star_h$ (instead of $\widehat{V}^{\widehat{\pi}}_h$) to discuss. Essentially, the key challenge in obtaining the optimal dependence in stationary setting is the need to decouple the dependence between $P-\widehat{P}$ and $\widehat{V}^\star_h$ as we aggregate all data for constructing both $\widehat{P}$ and $\widehat{V}^\star_h$. This issue is not encountered in the non-stationary setting in general due to the flexibility to estimate different transition $P_t$ at each time \citep{yin2021near} and $\widehat{P}_t$ and $\widehat{V}^\star_{t+1}$ preserve conditional independence. However, when confined to stationary case,  their complex $\tilde{O}(H^3/d_m\epsilon^2)$ becomes suboptimal. Moreover, the direct use of $s$-absorbing MDP in \cite{agarwal2020model} does not yield tight bounds for the finite horizon stationary setting, as it requires $s$-absorbing MDPs with $H$-dimensional fine-grid cover to make sure $\widehat{V}^\star_h$ is close to one of the elements in the cover (which has size $\approx H^H$ and it is not optimal \cite{cui2020plug}). We overcome this hurdle by choosing \emph{only one} delicate absorbing MDP to approximate $\widehat{V}^\star_h$ which will not incur additional dependence on horizon $H$ caused by the union bound. We begin with the general definition of absorbing MDP initialized in \citet{agarwal2020model} and then introduce the \emph{singleton absorbing MDP}.

\paragraph{Standard $s$-absorbing MDP in the finite horizon setting.} The general $s$-absorbing MDP is defined as follows: for a fixed state $s$ and a sequence $\{u_t\}_{t=1}^H$, MDP $M_{s,\{u_t\}_{t=1}^H}$ is identical to $M$ for all states except $s$, and state $s$ is absorbing in the sense $P_{M_{s,\{u_t\}_{t=1}^H}}(s|s,a)=1$ for all $a$, and the instantaneous reward at time $t$ is $r_t(s,a)=u_t$ for all $a\in\mathcal{A},t\in[H]$. For convenience, we use the shorthand notation $V^\pi_{\{s,u_t\}}$ to denote $V^\pi_{s,{M_{s,\{u_t\}_{t=1}^H}}}$ and similarly for $Q_t,r$ and transition $P$. Also, $V^\star_{\{s,u_t\}}$ ($Q^\star_{\{s,u_t\}}$) is the optimal value under $M_{s,\{u_t\}_{t=1}^H}$.

Before defining singleton absorbing MDP, we first present the following Lemma~\ref{lem:diff_m} and Lemma~\ref{lem:smdp_prop_m} which support the our design.

\begin{lemma}\label{lem:diff_m}
$ V^\star_t(s)-V^\star_{t+1}(s)\geq 0$,  $\forall s\in\mathcal{S},t\in[H]$.
\end{lemma}


\begin{lemma}\label{lem:smdp_prop_m}
	Fix a state $s$. If we choose $u_t^\star:= V^\star_t(s)-V^\star_{t+1}(s)$, then we have the following vector form equation 
	\[
	V^\star_{h,\{s,u_t^\star\}}=V^\star_{h,M}\quad \forall h\in[H].
	\]
	Similarly, if we choose $\hat{u}_t^\star:= \widehat{V}^\star_t(s)-\widehat{V}^\star_{t+1}(s)$, then $\widehat{V}^\star_{h,\{s,\hat{u}_t^\star\}}=\widehat{V}^\star_{h,M}$, $\forall h\in[H]$.
\end{lemma}

The proofs are deferred to Appendix~\ref{sec:proof_local_optimal}. Note by Lemma~\ref{lem:diff_m} the assignment of $u^\star_t(:=r_{t,\{s,u^\star_t\}})$ is well-defined. Lemma~\ref{lem:smdp_prop_m} is crucial since, under the specification of $u^\star_t$, the optimal value in $M_{s,\{u_t^\star\}_{t=1}^H}$ is identical to the optimal value in original $M$. Based on these, we can define the following:
	

	\begin{definition}[\textbf{Singleton-absorbing MDP}]\label{def:singleton_mdp}
		For each state $s$, the singleton-absorbing MDP is chosen to be $M_{s,\{u_t^\star\}_{t=1}^H}$, where $u_t^\star:= V^\star_t(s)-V^\star_{t+1}(s)$ for all $t\in[H]$.
	\end{definition}

Using Definition~\ref{def:singleton_mdp}, for each $(s,a)$ row the term $( \hP_{s,a}-P_{s,a} )\hoV_h$ can be substituted by $( \hP_{s,a}-P_{s,a} )\hoV_{h,\{s,u^\star_t\}}$, where $\hP_{s,a}$ and $\hoV_{h,\{s,u^\star_t\}}$ are independent by construction and Bernstein concentration applies. Furthermore, by the selection of $u^\star_t$, we can control the error of $||\hoV_h-\hoV_{h,\{s,u^\star_t\}}||_\infty$ to have rate $O(\sqrt{\frac{1}{n}})$ which forces the term $( \hP_{s,a}-P_{s,a} )(\hoV_h-\hoV_{h,\{s,u^\star_t\}})$ to have higher order error. These are the critical building blocks for bounding $||\widehat{Q}_h^{\widehat{\pi}^\star}- {Q}_h^{\widehat{\pi}^\star}||_\infty$. 

Indeed, by Bellman equations we have the decomposition:
{
\begin{align*}
&\widehat{Q}_h^{\widehat{\pi}^\star}- {Q}_h^{\widehat{\pi}^\star}
=\ldots=\sum_{t=h}^H\Gamma_{h+1:t}^{\widehat{\pi}^\star}\left(\widehat{P}-{P}\right)\widehat{V}_{t+1}^{\star},
\end{align*}
}where $\Gamma_{h+1:t}^\pi=\prod_{i=h+1}^t P^{\pi_i}$ is multi-step state-action transition and $\Gamma_{h+1:h}:=I$. Then for each $(s,a)$ row
{
\begin{equation}
\begin{aligned}
&( \hP_{s,a}-P_{s,a} )\hoV_h\\
=&  ( \hP_{s,a}-P_{s,a} )(\hoV_h-\hoV_{h,\{s,u^\star_t\}})+ ( \hP_{s,a}-P_{s,a} )\hoV_{h,\{s,u^\star_t\}}\\
\lesssim& ||\hP_{s,a}-P_{s,a}||_1||\hoV_h-\hoV_{h,\{s,u^\star_t\}}||_\infty+\sqrt{\frac{\Var_{s,a}(\hoV_{h,\{s,u^\star_t\}})}{n_{s,a}}}\\
\lesssim& \sqrt{\frac{S}{n_{s,a}}}\norm{\hoV_h-\hoV_{h,\{s,u^\star_t\}}}_\infty+\sqrt{\frac{\Var_{s,a}(\hoV_{h})}{n_{s,a}}}\quad (\star)\\
\end{aligned}
\end{equation}	
}where $(\star)$ is the place where the traditional technique uses the union bound over their \emph{exponential large} $\epsilon$-net and we do not have it! Next, by Lemma~\ref{lem:smdp_prop_m} and Lemma~\ref{lem:q_diff} in Appendix 
\begin{align*}
&||\hoV_h-\hoV_{h,\{s,u^\star_t\}}||_\infty= ||\hoV_{h,\{s,\hat{u}^\star_t\}}-\hoV_{h,\{s,{u}^\star_t\}}||_\infty\\
\leq& H\max_t\left|\hat{u}^\star_t-u^\star_t\right|\leq 2H\max_t|\widehat{V}^\star_t-V^\star_t|,
\end{align*}
by a crude bound (Lemma~\ref{lem:crude_u_b}), $\max_t|\widehat{V}^\star_t-V^\star_t|\lesssim H^2\sqrt{\frac{S}{n_{s,a}}}$ which makes $\sqrt{\frac{1}{n_{s,a}}}||\hoV_h-\hoV_{h,\{s,u^\star_t\}}||_\infty$ have order $1/n_{s,a}$. Finally, to reduce the horizon dependence we apply \newline $\sum_{t=h}^H\Gamma_{h+1:t}^{{\pi}}\sqrt{\Var_{s,a}\left(V^{{\pi}}_{t+1}\right)}\leq \sqrt{(H-h)^3}$ for any $\pi$. This (informally) bounds $\widehat{Q}_h^{\widehat{\pi}^\star}- {Q}_h^{\widehat{\pi}^\star}$ by
\[
||\widehat{Q}_h^{\widehat{\pi}^\star}- {Q}_h^{\widehat{\pi}^\star}||_\infty\lesssim \sqrt{\frac{H^3}{n_{s,a}}}+\frac{Poly(H,S)}{n_{s,a}}.
\]
Lastly, use $\min_{s,a} n_{s,a}\gtrsim H\cdot d_m$ to finish the proof.

\begin{remark}
	We emphasize the appropriate selection of $M_{s,\{u_t^\star\}_{t=1}^H}$ ($\widehat{M}_{s,\{u_t^\star\}_{t=1}^H}$) is the key for achieving optimality. It guarantees two things: 1. $\hoV_{h,\{s,u^\star_t\}}$ approximates $\hoV_h$ with sufficient accuracy (has rate $\sqrt{1/n_{s,a}}$); 2. it avoids the fine-grid design with exponential union bound in the dominate term ($\sqrt{\frac{\Var_{s,a}(\hoV_{h})\log(|U_{s,a}|/\delta)}{N}}$ with $|U_{s,a}|$ to be at least $H^H$ \cite{cui2020plug}.)
\end{remark}


\section{New settings: offline Task-agnostic and offline Reward-free learning}\label{sec:application}

From Corollary~\ref{cor:opt_offline}, our model-based offline learning algorithm has two steps: 1. constructing offline empirical MDP $\widehat{M}$ using the offline dataset {$\mathcal{D}=\{(s_{t}^{(i)}, a_{t}^{(i)}, r(s_{t}^{(i)}, a_{t}^{(i)}), s_{t+1}^{(i)})\}_{i \in[n]}^{t \in[H]}$}; 2. performing any accurate black-box \emph{planning} algorithm and returning $\widehat{\pi}^\star$(or $\widehat{\pi}$) as the final output. However, the only \emph{effective} data (data that contains stochasticity) is $\mathcal{D}'=\{(s_{t}^{(i)}, a_{t}^{(i)})\}_{i \in[n]}^{t \in[H]}$. This indicates we are essentially using the state-action space exploration data $\mathcal{D}'$ to solve the task-specific problem with reward $r$. With this perspective in mind, it is natural to ask: given only the offline exploration data $\mathcal{D}'$, can we efficiently learn a set of potentially conflicting $K$ tasks ($K$ rewards) simultaneously? Even more, can we efficiently learn all tasks simultaneously? This brings up the following definitions.
\vspace{0.1em}
\begin{definition}[Offline Task-agnostic Learning]\label{def:offline_ta}
	Given a offline exploration datatset $\mathcal{D}'=\{(s_{t}^{(i)}, a_{t}^{(i)})\}_{i \in[n]}^{t \in[H]}$ by $\mu$ with $n$ episodes. Given $K$ tasks with reward $\{r_k\}_{k=1}^K$ and the corresponding $K$ MDPs $M_k=(\mathcal{S}, \mathcal{A}, P, r_k, H, d_1)$. Can we use $\mathcal{D}'$ to output $\hat{\pi}_1,\ldots,\hat{\pi}_K$ such that 
	{\small
	$
	\P\left[\forall r_k,k\in[K], \norm{V^\star_{1,M_k}-V^{\hat{\pi}_k}_{1,M_k}}_\infty\leq\epsilon\right]\geq 1-\delta?
	$}
\end{definition}
\vspace{0.3em}
\begin{definition}[Offline Reward-free Learning]\label{def:offline_rf}
	Given a offline exploration datatset $\mathcal{D}'=\{(s_{t}^{(i)}, a_{t}^{(i)})\}_{i \in[n]}^{t \in[H]}$ by $\mu$ with $n$ episodes. For any reward $r$ and the corresponding MDP $M=(\mathcal{S}, \mathcal{A}, P, r, H, d_1)$. Can we use $\mathcal{D}'$ to output $\hat{\pi}$ such that 
	{\small
	$
	\P\left[\forall r, \norm{V^\star_{1,M}-V^{\hat{\pi}}_{1,M}}_\infty\leq\epsilon\right]\geq 1-\delta?
	$}
\end{definition}
Definition~\ref{def:offline_ta} and Definition~\ref{def:offline_rf} are the offline counterparts of \citet{zhang2020task} and \citet{jin2020reward} in online RL. Those settings are of practical interests in the offline regime as well since in practice reward functions are often iteratively engineered to encourage desired behavior via trial and error and using one shot of offline exploration data $\mathcal{D}'$ to tackle problems with different reward functions (different tasks) could help improve sample efficiency significantly. 

Our singleton absorbing MDP technique adapts to those settings and we have the following two theorems. The proofs of Theorem~\ref{thm:offline_ta}, \ref{thm:offline_rf} can be found in Appendix~\ref{sec:proof_ta}, \ref{sec:proof_rf}.

 \begin{theorem}[optimal offline task-agnostic learning]\label{thm:offline_ta}
 	Given $\mathcal{D}'=\{(s_{t}^{(i)}, a_{t}^{(i)})\}_{i \in[n]}^{t \in[H]}$ by $\mu$. Given $K$ tasks with reward $\{r_k\}_{k=1}^K$ and the corresponding $K$ MDPs $M_k=(\mathcal{S}, \mathcal{A}, P, r_k, H, d_1)$. Denote $\iota=\log(HSA/\delta)$. Let $\widehat{\pi}_k^{\star}:=\operatorname{argmax}_{\pi} \widehat{V}_{1,M_k}^{\pi}$ $\forall k\in[K]$, when $n>O(H\cdot[\iota+\log(K)]/d_m)$, then with probability $1-\delta$, 
 	$
 	\norm{V_{1,M_k}^\star-V_{1,M_k}^{\widehat{\pi}_k^{\star}}}_\infty\leq 
 	O\left[\sqrt{\frac{H^2 (\iota+\log(K))}{n d_m}}+\frac{H^{2.5}S^{0.5}(\iota+\log(K))}{nd_m}\right]
 	\;\;\forall k\in[K].
 	$
 \end{theorem}

\begin{theorem}[optimal offline reward-free learning]\label{thm:offline_rf}
	Given $\mathcal{D}'=\{(s_{t}^{(i)}, a_{t}^{(i)})\}_{i \in[n]}^{t \in[H]}$ by $\mu$. For any reward $r$ denote the corresponding MDP $M=(\mathcal{S}, \mathcal{A}, P, r, H, d_1)$. Denote $\iota=\log(HSA/\delta)$. Let $\widehat{\pi}^{\star}_M:=\operatorname{argmax}_{\pi} \widehat{V}_{1,M}^{\pi}$ $\forall r$, when $n>O(HS\cdot\iota/d_m)$, then with probability $1-\delta$, 
	$
	\norm{V_{1,M}^\star-V_{1,M}^{\widehat{\pi}^{\star}_M}}_\infty\leq O\left[\sqrt{\frac{H^2S \cdot\iota}{n d_m}}+\frac{H^{2}S\cdot\iota}{nd_m}\right],
	\;\;\forall r,M.
	$
\end{theorem}

By a direct translation of both theorems, we have sample complexity of order $\widetilde{O}(H^2\log(K)/d_m\epsilon^2)$ and $\widetilde{O}(H^2S/d_m\epsilon^2)$. All the parameters have the optimal rates, see the lower bounds in \citet{zhang2020task} and \citet{jin2020reward}.\footnote{We add a discussion in Appendix~\ref{sec:app_application} to explain more clearly why our rates are optimal for these problems. } The higher order dependence in Theorem~\ref{thm:offline_rf} is also tight comparing to Theorem~\ref{thm:offline_ta}. Such statistically optimal results reveal the model-based methods generalize well to those seemingly challenging problems in the offline regime. Changing to these harder problems would not affect the optimal statistical efficiency of the model-based approach.


\section{Further discussion: Extension to linear MDP with anchor representations} \label{sec:discussion}

The principle of our \emph{Singleton absorbing MDP} technique (with model-based construction) in decoupling the dependence between $\widehat{P}_{s,a}$ and $\widehat{V}^\star$ is not confined to tabular MDPs and therefore it is natural to generalize such idea for the episodic stationary transition setting for other problems. As an example, we further present a sharp result for the setting of finite horizon linear MDP with anchor points. We narrate by assuming a generative oracle (that allows sampling from $s'\sim P(\cdot|s,a)$) for the ease of exposition.

\begin{definition}[Linear MDP with anchor points \citep{yang2019sample,cui2020plug}]\label{def:anchor}
Let $\mathcal{S}$ be the exponential large space and $\mathcal{A}$ be the infinite (or even continuous) spaces. Assume there is feature map $\phi:\mathcal{S}\times\mathcal{A}\rightarrow \R^K$ (where $K\ll |\mathcal{S}|$), i.e.	$\phi(s,a)=[\phi_1(s,a),\ldots,\phi_K(s,a)]$. Transition $P$ admits a linear representation:
\[
P(s'|s,a)=\sum_{k\in[K]}\phi_k(s,a)\psi_k(s')
\]
where $\psi_1(\cdot),\ldots,\psi_K(\cdot)$ are unknowns. We further assume there exists a set of anchor state-action pairs $\mathcal{K}$ such that any $(s,a)$ can be represented as a convex combination of the anchors $\{(s_k,a_k)|k\in\mathcal{K}\}$:
\[
\exists\left\{\lambda_{k}^{s, a}\right\}: \phi(s, a)=\sum_{k \in \mathcal{K}} \lambda_{k}^{s, a} \phi\left(s_{k} , a_{k}\right), \sum_{k \in \mathcal{K}} \lambda_{k}^{s, a}=1, \lambda_{k} \geq 0, \forall k \in \mathcal{K},(s, a) \in(\mathcal{S}, \mathcal{A}).
\]
\end{definition} 
Under the definition, denote $N$ be the number of samples at each anchor pairs. Then we have the following ({see Appendix~\ref{sec:app_anchor} for the proof}): 
\begin{theorem}[Optimal sample complexity]\label{thm:anchor}
	Under Definition~\ref{def:anchor}, let $\widehat{\pi}^\star=\argmax_\pi \widehat{V}^\pi_1$. Then if $N\geq c{H^2}|\mathcal{S}|\log(KH/\delta)$, we have with probability $1-\delta$, $||Q^\star_1-Q^{\widehat{\pi}^\star}_1||_\infty\leq\widetilde{O}(\sqrt{H^3/N})$. 
\end{theorem}
Comparing to Theorem~4 of \cite{cui2020plug}, Theorem~\ref{thm:anchor} removes the additional dependence $\min\{|S|,K,H\}$. In term of the total sample complexity, Theorem~\ref{thm:anchor} gives $\tilde{O}(KH^3/\epsilon^2)$ while \cite{cui2020plug} has $\widetilde{O}(KH^4/\epsilon^2)$ (see their Section~7, first bullet point). Our result again reveals the model-based method is statistically optimal for the current setting.

\begin{remark}
	The rate $\tilde{O}(KH^3/\epsilon^2)$ with anchor point assumption has the linear dependence on $K$ and for the standard linear bandit \citep{lattimore2020bandit} $\Omega(\sqrt{d^2T})$ or the linear (mixture) MDP \citep{jin2020provably,zhou2020nearly} $\Omega(\sqrt{d^2H^2T})$ the lower bound dependence on the feature dimension $d$ is quadratic. We believe one reason for this to happen is that anchor representations assumption is somewhat strong as it abstracts the whole state action space by only finite points (via convex combination).  	
\end{remark}

\section{Conclusion and Future Works}\label{sec:conclusion}

This work studies the uniform convergence problems for offline policy evaluation (OPE) and provides complete answers for their optimality behaviors. We achieve the optimal sample complexity for stationary-transition case using a novel adaptation of the absorbing MDP trick, which is more generally applicable to the new offline task-agnostic and reward-free settings combining with the model-based approach and we hope it can be applied to a broader range of future problems. We end the section by two future directions.

 \noindent\textbf{On the higher order error term.}
 Our main result (Theorem~\ref{thm:optimal_upper_bound}) has an additional $\sqrt{HS}$ dependence in the higher order error term and we cannot further remove it based on our current technique. Nevertheless, this is already among the best higher order results to our knowledge. In fact, most state-of-the-art works (\emph{e.g.} \citet{azar2017minimax,dann2019policy,zhang2020reinforcement}) have additional $S$ dependence in the higher order and \cite{jin2018q} has only extra $\sqrt{S}$ in the higher order term but it also has additional $\sqrt{A}$ (see Table~1 of \cite{zhang2020reinforcement} for a clear reference). How to obtain optimality not only for the main term but also for the higher order error terms remains elusive for the community.

 \noindent\textbf{Uniform OPE and beyond.}
 The current study of uniform OPE derives results with expression using parameter dependence and deriving instance-dependent uniform convergence result will draw a clearer picture on the individual behaviors for each policy. Besides, this work concentrates on Tabular MDPs and generalizing uniform convergence to more practical settings like linear MDPs, game environments and multi-agent settings are promising future directions.  Specifically, general complexity measure (mirroring VC-dimensions and Rademacher complexities for statistical learning problems) that precisely captures local and global uniform convergence would be of great interest.

\subsection*{Acknowledgment}
The research is partially supported by NSF Awards \#2007117 and \#1934641. MY would like to thank Simon S. Du, Zihan Zhang for explaining a question in \cite{zhang2020nearly}; and Tongzheng Ren for the helpful discussions.

\bibliographystyle{plainnat}
\bibliography{sections/stat_rl}

\begin{thebibliography}{68}
\providecommand{\natexlab}[1]{#1}
\providecommand{\url}[1]{\texttt{#1}}
\expandafter\ifx\csname urlstyle\endcsname\relax
  \providecommand{\doi}[1]{doi: #1}\else
  \providecommand{\doi}{doi: \begingroup \urlstyle{rm}\Url}\fi

\bibitem[Agarwal et~al.(2019)Agarwal, Jiang, and
  Kakade]{agarwal2019reinforcement}
Alekh Agarwal, Nan Jiang, and Sham~M Kakade.
\newblock Reinforcement learning: Theory and algorithms.
\newblock \emph{CS Dept., UW Seattle, Seattle, WA, USA, Tech. Rep}, 2019.

\bibitem[Agarwal et~al.(2020)Agarwal, Kakade, and Yang]{agarwal2020model}
Alekh Agarwal, Sham Kakade, and Lin~F Yang.
\newblock Model-based reinforcement learning with a generative model is minimax
  optimal.
\newblock In \emph{Conference on Learning Theory}, pages 67--83, 2020.

\bibitem[Antos et~al.(2008{\natexlab{a}})Antos, Munos, and
  Szepesvari]{antos2008fitted}
Andras Antos, Remi Munos, and Csaba Szepesvari.
\newblock Fitted q-iteration in continuous action-space mdps.
\newblock In \emph{Advances in Neural Information Processing Systems}, pages
  9--16, 2008{\natexlab{a}}.

\bibitem[Antos et~al.(2008{\natexlab{b}})Antos, Szepesv{\'a}ri, and
  Munos]{antos2008learning}
Andr{\'a}s Antos, Csaba Szepesv{\'a}ri, and R{\'e}mi Munos.
\newblock Learning near-optimal policies with bellman-residual minimization
  based fitted policy iteration and a single sample path.
\newblock \emph{Machine Learning}, 71\penalty0 (1):\penalty0 89--129,
  2008{\natexlab{b}}.

\bibitem[Ayoub et~al.(2020)Ayoub, Jia, Szepesvari, Wang, and
  Yang]{ayoub2020model}
Alex Ayoub, Zeyu Jia, Csaba Szepesvari, Mengdi Wang, and Lin Yang.
\newblock Model-based reinforcement learning with value-targeted regression.
\newblock In \emph{International Conference on Machine Learning}, pages
  463--474. PMLR, 2020.

\bibitem[Azar et~al.(2017)Azar, Osband, and Munos]{azar2017minimax}
Mohammad~Gheshlaghi Azar, Ian Osband, and R{\'e}mi Munos.
\newblock Minimax regret bounds for reinforcement learning.
\newblock In \emph{Proceedings of the 34th International Conference on Machine
  Learning-Volume 70}, pages 263--272. JMLR. org, 2017.

\bibitem[Chen and Jiang(2019)]{chen2019information}
Jinglin Chen and Nan Jiang.
\newblock Information-theoretic considerations in batch reinforcement learning.
\newblock In \emph{International Conference on Machine Learning}, pages
  1042--1051, 2019.

\bibitem[Chernoff et~al.(1952)]{chernoff1952measure}
Herman Chernoff et~al.
\newblock A measure of asymptotic efficiency for tests of a hypothesis based on
  the sum of observations.
\newblock \emph{The Annals of Mathematical Statistics}, 23\penalty0
  (4):\penalty0 493--507, 1952.

\bibitem[Codevilla et~al.(2018)Codevilla, Lopez, Koltun, and
  Dosovitskiy]{codevilla2018offline}
Felipe Codevilla, Antonio~M Lopez, Vladlen Koltun, and Alexey Dosovitskiy.
\newblock On offline evaluation of vision-based driving models.
\newblock In \emph{Proceedings of the European Conference on Computer Vision
  (ECCV)}, pages 236--251, 2018.

\bibitem[Cui and Yang(2020)]{cui2020plug}
Qiwen Cui and Lin~F Yang.
\newblock Is plug-in solver sample-efficient for feature-based reinforcement
  learning?
\newblock In \emph{Advances in neural information processing systems}, 2020.

\bibitem[Dann et~al.(2019)Dann, Li, Wei, and Brunskill]{dann2019policy}
Christoph Dann, Lihong Li, Wei Wei, and Emma Brunskill.
\newblock Policy certificates: Towards accountable reinforcement learning.
\newblock In \emph{International Conference on Machine Learning}, pages
  1507--1516. PMLR, 2019.

\bibitem[Ding et~al.(2021)Ding, Wei, Yang, Wang, and
  Jovanovic]{ding2021provably}
Dongsheng Ding, Xiaohan Wei, Zhuoran Yang, Zhaoran Wang, and Mihailo Jovanovic.
\newblock Provably efficient safe exploration via primal-dual policy
  optimization.
\newblock In \emph{International Conference on Artificial Intelligence and
  Statistics}, pages 3304--3312. PMLR, 2021.

\bibitem[Duan et~al.(2020)Duan, Jia, and Wang]{duan2020minimax}
Yaqi Duan, Zeyu Jia, and Mengdi Wang.
\newblock Minimax-optimal off-policy evaluation with linear function
  approximation.
\newblock In \emph{International Conference on Machine Learning}, pages
  8334--8342, 2020.

\bibitem[Efroni et~al.(2019)Efroni, Merlis, Ghavamzadeh, and
  Mannor]{efroni2019tight}
Yonathan Efroni, Nadav Merlis, Mohammad Ghavamzadeh, and Shie Mannor.
\newblock Tight regret bounds for model-based reinforcement learning with
  greedy policies.
\newblock In \emph{Advances in Neural Information Processing Systems}, 2019.

\bibitem[Han et~al.(2015)Han, Jiao, and Weissman]{han2015minimax}
Yanjun Han, Jiantao Jiao, and Tsachy Weissman.
\newblock Minimax estimation of discrete distributions under $l_1$ loss.
\newblock \emph{IEEE Transactions on Information Theory}, 61\penalty0
  (11):\penalty0 6343--6354, 2015.

\bibitem[Hao et~al.(2020)Hao, Duan, Lattimore, Szepesv{\'a}ri, and
  Wang]{hao2020sparse}
Botao Hao, Yaqi Duan, Tor Lattimore, Csaba Szepesv{\'a}ri, and Mengdi Wang.
\newblock Sparse feature selection makes batch reinforcement learning more
  sample efficient.
\newblock \emph{arXiv preprint arXiv:2011.04019}, 2020.

\bibitem[Hu et~al.(2021)Hu, Kallus, and Uehara]{hu2021fast}
Yichun Hu, Nathan Kallus, and Masatoshi Uehara.
\newblock Fast rates for the regret of offline reinforcement learning.
\newblock \emph{arXiv preprint arXiv:2102.00479}, 2021.

\bibitem[Jaksch et~al.(2010)Jaksch, Ortner, and Auer]{jaksch2010near}
Thomas Jaksch, Ronald Ortner, and Peter Auer.
\newblock Near-optimal regret bounds for reinforcement learning.
\newblock \emph{Journal of Machine Learning Research}, 11\penalty0 (4), 2010.

\bibitem[Jiang(2018)]{jiang2018notes}
Nan Jiang.
\newblock Notes on tabular methods.
\newblock 2018.

\bibitem[Jiang and Li(2016)]{jiang2016doubly}
Nan Jiang and Lihong Li.
\newblock Doubly robust off-policy value evaluation for reinforcement learning.
\newblock In \emph{Proceedings of the 33rd International Conference on
  International Conference on Machine Learning-Volume 48}, pages 652--661.
  JMLR. org, 2016.

\bibitem[Jin et~al.(2018)Jin, Allen-Zhu, Bubeck, and Jordan]{jin2018q}
Chi Jin, Zeyuan Allen-Zhu, Sebastien Bubeck, and Michael~I Jordan.
\newblock Is q-learning provably efficient?
\newblock In \emph{Advances in Neural Information Processing Systems}, pages
  4863--4873, 2018.

\bibitem[Jin et~al.(2020{\natexlab{a}})Jin, Krishnamurthy, Simchowitz, and
  Yu]{jin2020reward}
Chi Jin, Akshay Krishnamurthy, Max Simchowitz, and Tiancheng Yu.
\newblock Reward-free exploration for reinforcement learning.
\newblock In \emph{International Conference on Machine Learning}, pages
  4870--4879. PMLR, 2020{\natexlab{a}}.

\bibitem[Jin et~al.(2020{\natexlab{b}})Jin, Yang, Wang, and
  Jordan]{jin2020provably}
Chi Jin, Zhuoran Yang, Zhaoran Wang, and Michael~I Jordan.
\newblock Provably efficient reinforcement learning with linear function
  approximation.
\newblock In \emph{Conference on Learning Theory}, pages 2137--2143. PMLR,
  2020{\natexlab{b}}.

\bibitem[Jin et~al.(2020{\natexlab{c}})Jin, Yang, and Wang]{jin2020pessimism}
Ying Jin, Zhuoran Yang, and Zhaoran Wang.
\newblock Is pessimism provably efficient for offline rl?
\newblock \emph{arXiv preprint arXiv:2012.15085}, 2020{\natexlab{c}}.

\bibitem[Kallus and Uehara(2019)]{kallus2019intrinsically}
Nathan Kallus and Masatoshi Uehara.
\newblock Intrinsically efficient, stable, and bounded off-policy evaluation
  for reinforcement learning.
\newblock In \emph{Advances in Neural Information Processing Systems}, pages
  3325--3334, 2019.

\bibitem[Kallus and Uehara(2020)]{kallus2019double}
Nathan Kallus and Masatoshi Uehara.
\newblock Double reinforcement learning for efficient off-policy evaluation in
  markov decision processes.
\newblock In \emph{International Conference on Machine Learning}, pages
  1922--1931, 2020.

\bibitem[Kaufmann et~al.(2020)Kaufmann, M{\'e}nard, Domingues, Jonsson,
  Leurent, and Valko]{kaufmann2020adaptive}
Emilie Kaufmann, Pierre M{\'e}nard, Omar~Darwiche Domingues, Anders Jonsson,
  Edouard Leurent, and Michal Valko.
\newblock Adaptive reward-free exploration.
\newblock \emph{arXiv preprint arXiv:2006.06294}, 2020.

\bibitem[Kidambi et~al.(2020)Kidambi, Rajeswaran, Netrapalli, and
  Joachims]{kidambi2020morel}
Rahul Kidambi, Aravind Rajeswaran, Praneeth Netrapalli, and Thorsten Joachims.
\newblock Morel: Model-based offline reinforcement learning.
\newblock \emph{arXiv preprint arXiv:2005.05951}, 2020.

\bibitem[Lange et~al.(2012)Lange, Gabel, and Riedmiller]{lange2012batch}
Sascha Lange, Thomas Gabel, and Martin Riedmiller.
\newblock Batch reinforcement learning.
\newblock In \emph{Reinforcement learning}, pages 45--73. Springer, 2012.

\bibitem[Lattimore and Szepesv{\'a}ri(2020)]{lattimore2020bandit}
Tor Lattimore and Csaba Szepesv{\'a}ri.
\newblock \emph{Bandit algorithms}.
\newblock Cambridge University Press, 2020.

\bibitem[Le et~al.(2019)Le, Voloshin, and Yue]{le2019batch}
Hoang Le, Cameron Voloshin, and Yisong Yue.
\newblock Batch policy learning under constraints.
\newblock In \emph{International Conference on Machine Learning}, pages
  3703--3712, 2019.

\bibitem[Levine et~al.(2020)Levine, Kumar, Tucker, and Fu]{levine2020offline}
Sergey Levine, Aviral Kumar, George Tucker, and Justin Fu.
\newblock Offline reinforcement learning: Tutorial, review, and perspectives on
  open problems.
\newblock \emph{arXiv preprint arXiv:2005.01643}, 2020.

\bibitem[Li et~al.(2020)Li, Wei, Chi, Gu, and Chen]{li2020breaking}
Gen Li, Yuting Wei, Yuejie Chi, Yuantao Gu, and Yuxin Chen.
\newblock Breaking the sample size barrier in model-based reinforcement
  learning with a generative model.
\newblock \emph{Advances in Neural Information Processing Systems}, 33, 2020.

\bibitem[Liu et~al.(2018)Liu, Li, Tang, and Zhou]{liu2018breaking}
Qiang Liu, Lihong Li, Ziyang Tang, and Dengyong Zhou.
\newblock Breaking the curse of horizon: Infinite-horizon off-policy
  estimation.
\newblock In \emph{Advances in Neural Information Processing Systems}, pages
  5361--5371, 2018.

\bibitem[Liu et~al.(2020{\natexlab{a}})Liu, Yu, Bai, and Jin]{liu2020sharp}
Qinghua Liu, Tiancheng Yu, Yu~Bai, and Chi Jin.
\newblock A sharp analysis of model-based reinforcement learning with
  self-play.
\newblock \emph{arXiv preprint arXiv:2010.01604}, 2020{\natexlab{a}}.

\bibitem[Liu et~al.(2019)Liu, Swaminathan, Agarwal, and Brunskill]{liu2019off}
Yao Liu, Adith Swaminathan, Alekh Agarwal, and Emma Brunskill.
\newblock Off-policy policy gradient with state distribution correction.
\newblock In \emph{Uncertainty in Artificial Intelligence}, 2019.

\bibitem[Liu et~al.(2020{\natexlab{b}})Liu, Swaminathan, Agarwal, and
  Brunskill]{liu2020provably}
Yao Liu, Adith Swaminathan, Alekh Agarwal, and Emma Brunskill.
\newblock Provably good batch reinforcement learning without great exploration.
\newblock \emph{arXiv preprint arXiv:2007.08202}, 2020{\natexlab{b}}.

\bibitem[Liu et~al.(2017)Liu, Logan, Liu, Xu, Tang, and Wang]{liu2017deep}
Ying Liu, Brent Logan, Ning Liu, Zhiyuan Xu, Jian Tang, and Yangzhi Wang.
\newblock Deep reinforcement learning for dynamic treatment regimes on medical
  registry data.
\newblock In \emph{2017 IEEE International Conference on Healthcare Informatics
  (ICHI)}, pages 380--385. IEEE, 2017.

\bibitem[Mannor and Tsitsiklis(2004)]{mannor2004sample}
Shie Mannor and John~N Tsitsiklis.
\newblock The sample complexity of exploration in the multi-armed bandit
  problem.
\newblock \emph{Journal of Machine Learning Research}, 5\penalty0
  (Jun):\penalty0 623--648, 2004.

\bibitem[Menard et~al.(2020)Menard, Domingues, Jonsson, Kaufmann, Leurent, and
  Valko]{menard2020fast}
Pierre Menard, Omar~Darwiche Domingues, Anders Jonsson, Emilie Kaufmann,
  Edouard Leurent, and Michal Valko.
\newblock Fast active learning for pure exploration in reinforcement learning.
\newblock \emph{arXiv preprint arXiv:2007.13442}, 2020.

\bibitem[Munos(2003)]{munos2003error}
R{\'e}mi Munos.
\newblock Error bounds for approximate policy iteration.
\newblock In \emph{International Conference on Machine Learning}, pages
  560--567, 2003.

\bibitem[Nachum et~al.(2019)Nachum, Chow, Dai, and Li]{nachum2019dualdice}
Ofir Nachum, Yinlam Chow, Bo~Dai, and Lihong Li.
\newblock Dualdice: Behavior-agnostic estimation of discounted stationary
  distribution corrections.
\newblock \emph{arXiv preprint arXiv:1906.04733}, 2019.

\bibitem[Rashidinejad et~al.(2021)Rashidinejad, Zhu, Ma, Jiao, and
  Russell]{rashidinejad2021bridging}
Paria Rashidinejad, Banghua Zhu, Cong Ma, Jiantao Jiao, and Stuart Russell.
\newblock Bridging offline reinforcement learning and imitation learning: A
  tale of pessimism.
\newblock \emph{arXiv preprint arXiv:2103.12021}, 2021.

\bibitem[Ren et~al.(2021)Ren, Li, Dai, Du, and Sanghavi]{ren2021nearly}
Tongzheng Ren, Jialian Li, Bo~Dai, Simon~S Du, and Sujay Sanghavi.
\newblock Nearly horizon-free offline reinforcement learning.
\newblock \emph{arXiv preprint arXiv:2103.14077}, 2021.

\bibitem[Shalev-Shwartz et~al.(2010)Shalev-Shwartz, Shamir, Srebro, and
  Sridharan]{shalev2010learnability}
Shai Shalev-Shwartz, Ohad Shamir, Nathan Srebro, and Karthik Sridharan.
\newblock Learnability, stability and uniform convergence.
\newblock \emph{The Journal of Machine Learning Research}, 11:\penalty0
  2635--2670, 2010.

\bibitem[Sridharan(2002)]{sridharan2002gentle}
Karthik Sridharan.
\newblock A gentle introduction to concentration inequalities.
\newblock \emph{Dept. Comput. Sci., Cornell Univ., Tech. Rep}, 2002.

\bibitem[Sutton and Barto(2018)]{sutton2018reinforcement}
Richard~S Sutton and Andrew~G Barto.
\newblock \emph{Reinforcement learning: An introduction}.
\newblock MIT press, 2018.

\bibitem[Szepesv{\'a}ri and Munos(2005)]{szepesvari2005finite}
Csaba Szepesv{\'a}ri and R{\'e}mi Munos.
\newblock Finite time bounds for sampling based fitted value iteration.
\newblock In \emph{Proceedings of the 22nd international conference on Machine
  learning}, pages 880--887, 2005.

\bibitem[Tropp et~al.(2011)]{tropp2011freedman}
Joel Tropp et~al.
\newblock Freedman's inequality for matrix martingales.
\newblock \emph{Electronic Communications in Probability}, 16:\penalty0
  262--270, 2011.

\bibitem[Uehara and Jiang(2019)]{uehara2019minimax}
Masatoshi Uehara and Nan Jiang.
\newblock Minimax weight and q-function learning for off-policy evaluation.
\newblock \emph{arXiv preprint arXiv:1910.12809}, 2019.

\bibitem[Vapnik(2013)]{vapnik2013nature}
Vladimir Vapnik.
\newblock \emph{The nature of statistical learning theory}.
\newblock Springer science \& business media, 2013.

\bibitem[Wang et~al.(2020{\natexlab{a}})Wang, Du, Yang, and
  Salakhutdinov]{wang2020reward}
Ruosong Wang, Simon~S Du, Lin~F Yang, and Ruslan Salakhutdinov.
\newblock On reward-free reinforcement learning with linear function
  approximation.
\newblock \emph{arXiv preprint arXiv:2006.11274}, 2020{\natexlab{a}}.

\bibitem[Wang et~al.(2020{\natexlab{b}})Wang, Foster, and
  Kakade]{wang2020statistical}
Ruosong Wang, Dean~P Foster, and Sham~M Kakade.
\newblock What are the statistical limits of offline rl with linear function
  approximation?
\newblock \emph{arXiv preprint arXiv:2010.11895}, 2020{\natexlab{b}}.

\bibitem[Xie and Jiang(2020{\natexlab{a}})]{xie2020batch}
Tengyang Xie and Nan Jiang.
\newblock Batch value-function approximation with only realizability.
\newblock \emph{arXiv preprint arXiv:2008.04990}, 2020{\natexlab{a}}.

\bibitem[Xie and Jiang(2020{\natexlab{b}})]{xie2020q}
Tengyang Xie and Nan Jiang.
\newblock Q* approximation schemes for batch reinforcement learning: A
  theoretical comparison.
\newblock In \emph{Uncertainty in Artificial Intelligence}, pages 550--559,
  2020{\natexlab{b}}.

\bibitem[Xie et~al.(2019)Xie, Ma, and Wang]{xie2019towards}
Tengyang Xie, Yifei Ma, and Yu-Xiang Wang.
\newblock Towards optimal off-policy evaluation for reinforcement learning with
  marginalized importance sampling.
\newblock In \emph{Advances in Neural Information Processing Systems}, pages
  9668--9678, 2019.

\bibitem[Yang and Wang(2019)]{yang2019sample}
Lin Yang and Mengdi Wang.
\newblock Sample-optimal parametric q-learning using linearly additive
  features.
\newblock In \emph{International Conference on Machine Learning}, pages
  6995--7004. PMLR, 2019.

\bibitem[Yin and Wang(2020)]{yin2020asymptotically}
Ming Yin and Yu-Xiang Wang.
\newblock Asymptotically efficient off-policy evaluation for tabular
  reinforcement learning.
\newblock In \emph{International Conference on Artificial Intelligence and
  Statistics}, pages 3948--3958. PMLR, 2020.

\bibitem[Yin et~al.(2021{\natexlab{a}})Yin, Bai, and Wang]{yin2021near}
Ming Yin, Yu~Bai, and Yu-Xiang Wang.
\newblock Near-optimal provable uniform convergence in offline policy
  evaluation for reinforcement learning.
\newblock In \emph{International Conference on Artificial Intelligence and
  Statistics}, pages 1567--1575. PMLR, 2021{\natexlab{a}}.

\bibitem[Yin et~al.(2021{\natexlab{b}})Yin, Bai, and Wang]{yin2021nearoptimal}
Ming Yin, Yu~Bai, and Yu-Xiang Wang.
\newblock Near-optimal offline reinforcement learning via double variance
  reduction.
\newblock \emph{arXiv preprint arXiv:2102.01748}, 2021{\natexlab{b}}.

\bibitem[Yu et~al.(2020)Yu, Thomas, Yu, Ermon, Zou, Levine, Finn, and
  Ma]{yu2020mopo}
Tianhe Yu, Garrett Thomas, Lantao Yu, Stefano Ermon, James Zou, Sergey Levine,
  Chelsea Finn, and Tengyu Ma.
\newblock Mopo: Model-based offline policy optimization.
\newblock \emph{arXiv preprint arXiv:2005.13239}, 2020.

\bibitem[Zanette(2020)]{zanette2020exponential}
Andrea Zanette.
\newblock Exponential lower bounds for batch reinforcement learning: Batch rl
  can be exponentially harder than online rl.
\newblock \emph{arXiv preprint arXiv:2012.08005}, 2020.

\bibitem[Zhang et~al.(2020{\natexlab{a}})Zhang, Kakade, Ba{\c{s}}ar, and
  Yang]{zhang2020model}
Kaiqing Zhang, Sham~M Kakade, Tamer Ba{\c{s}}ar, and Lin~F Yang.
\newblock Model-based multi-agent rl in zero-sum markov games with near-optimal
  sample complexity.
\newblock \emph{arXiv preprint arXiv:2007.07461}, 2020{\natexlab{a}}.

\bibitem[Zhang et~al.(2021)Zhang, Wan, Sutton, and Whiteson]{zhang2021average}
Shangtong Zhang, Yi~Wan, Richard~S Sutton, and Shimon Whiteson.
\newblock Average-reward off-policy policy evaluation with function
  approximation.
\newblock \emph{arXiv preprint arXiv:2101.02808}, 2021.

\bibitem[Zhang et~al.(2020{\natexlab{b}})Zhang, Singla, et~al.]{zhang2020task}
Xuezhou Zhang, Adish Singla, et~al.
\newblock Task-agnostic exploration in reinforcement learning.
\newblock \emph{Advances in Neural Information Processing Systems},
  2020{\natexlab{b}}.

\bibitem[Zhang et~al.(2020{\natexlab{c}})Zhang, Du, and Ji]{zhang2020nearly}
Zihan Zhang, Simon~S Du, and Xiangyang Ji.
\newblock Nearly minimax optimal reward-free reinforcement learning.
\newblock \emph{arXiv preprint arXiv:2010.05901}, 2020{\natexlab{c}}.

\bibitem[Zhang et~al.(2020{\natexlab{d}})Zhang, Ji, and
  Du]{zhang2020reinforcement}
Zihan Zhang, Xiangyang Ji, and Simon~S Du.
\newblock Is reinforcement learning more difficult than bandits? a near-optimal
  algorithm escaping the curse of horizon.
\newblock \emph{arXiv preprint arXiv:2009.13503}, 2020{\natexlab{d}}.

\bibitem[Zhou et~al.(2020)Zhou, Gu, and Szepesvari]{zhou2020nearly}
Dongruo Zhou, Quanquan Gu, and Csaba Szepesvari.
\newblock Nearly minimax optimal reinforcement learning for linear mixture
  markov decision processes.
\newblock \emph{arXiv preprint arXiv:2012.08507}, 2020.

\end{thebibliography}

\appendix

\clearpage
\begin{center}
	 {\LARGE \textbf{Appendix}}
\end{center}


\begin{figure}[H]
	\centering     
	\subfigure[Covering-based Absorbing: $1$-D case]{\label{fig:1}\includegraphics[width=65mm]{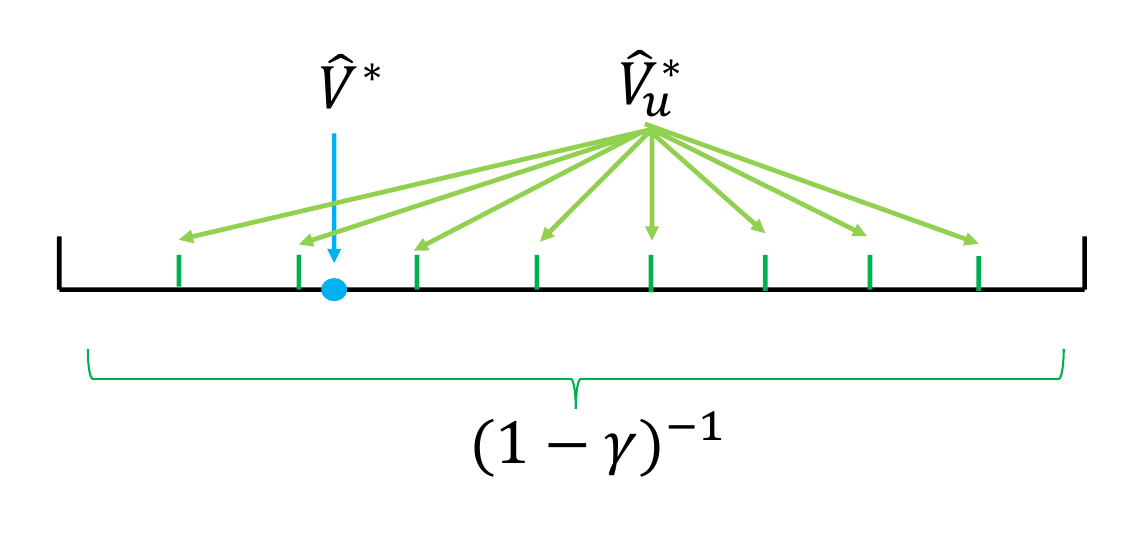}}
	\subfigure[Covering-based Absorbing: $2$-D case]{\label{fig:2}\includegraphics[width=65mm]{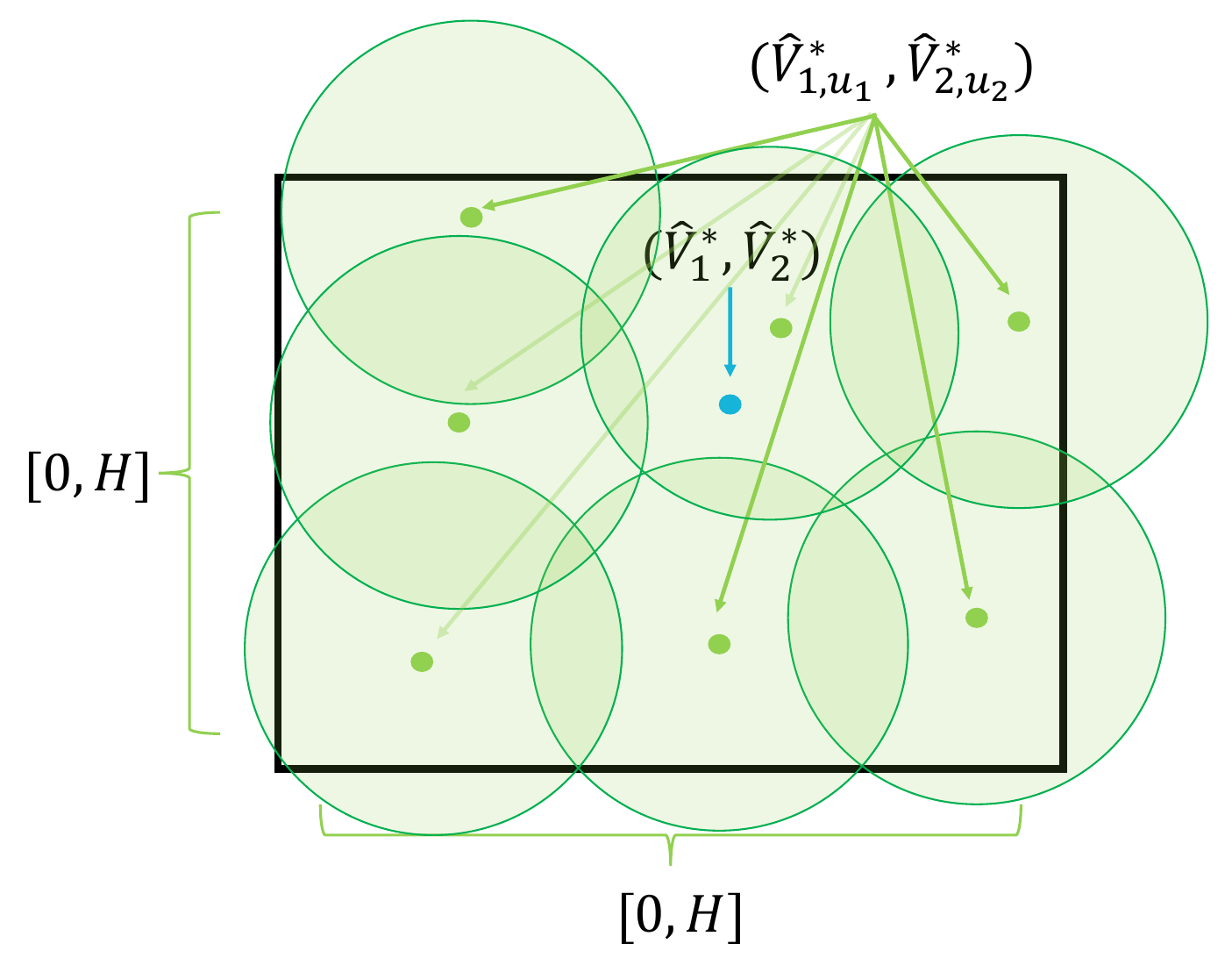}}
	\subfigure[Singleton Absorbing: $1$-D case]{\label{fig:3}\includegraphics[width=65mm]{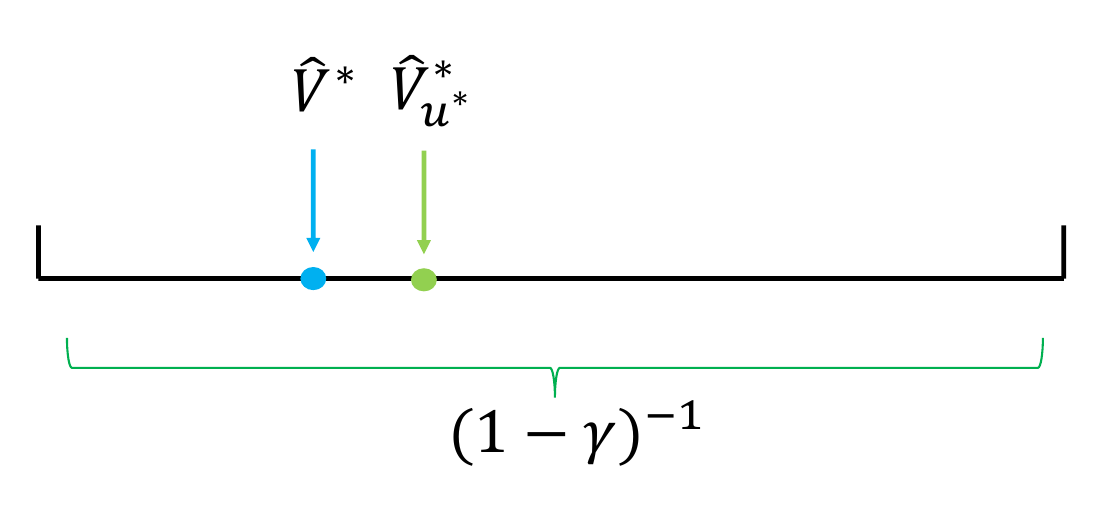}}
	\subfigure[Singleton Absorbing: $2$-D case]{\label{fig:4}\includegraphics[width=65mm]{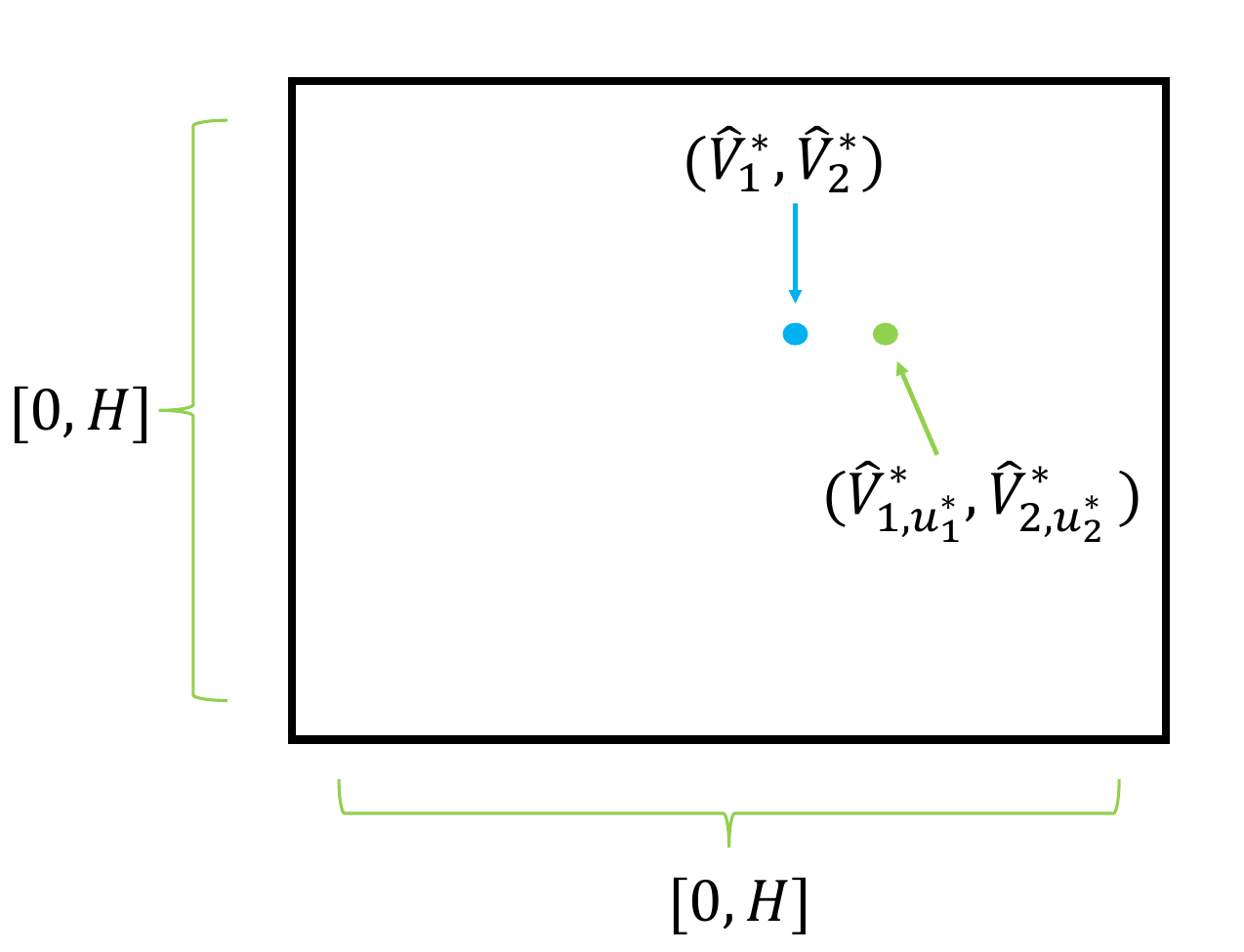}}
	\caption{Visualization of singleton absorbing MDP technique}
	\label{fig:main}
\end{figure}

We provide a visualization (Figure~\ref{fig:main}) for understanding the \emph{singleton absorbing MDP} technique at the beginning of Appendix. \ref{fig:1}, \ref{fig:3} demonstrate the infinite horizon case and \ref{fig:2}, \ref{fig:4} demonstrate the finite horizon case. In particular, it should be a $H$-dimensional hypercube $[0,H]^H$ (that contains $\widehat{V}^\star_1,\ldots,\widehat{V}^\star_h$) instead of only the square $[0,H]\times[0,H]$ ($\widehat{V}^\star_1,\widehat{V}^\star_2$). This is only for the ease of visualization.

The standard absorbing MDP technique \cite{agarwal2020model,cui2020plug} leverages a set of absorbing MDPs to cover the range of value functions (following the standard covering principle) to make sure $\widehat{V}^\star$ is close to one of the element (absorbing MDP) in the set (Figure~\ref{fig:1},\ref{fig:2}). The size of the covering set (\emph{i.e.} the covering number) grows exponentially in $H$ \ref{fig:2} in the finite horizon setting and this is due to the fact that there are $\widehat{V}^\star_1,\widehat{V}^\star_2,\ldots,\widehat{V}^\star_H$ quantities to cover. This results in the metric entropy (the $\log$ of the covering number) to blow up by a factor of $H$ and incurs suboptimality. On the other hand, by the nifty chosen singleton absorbing MDP $\widehat{V}^\star_{h,u^\star}$ (Figure~\ref{fig:3},\ref{fig:4}), we completely get rid of the covering issue (covering the $H$-dimensional space requires exponential in $H$ size), maintain the independence and  control the error propagation ($\norm{\widehat{V}^\star-\widehat{V}^\star_{u^\star}}_\infty$ is sufficiently small). See Section~\ref{sec:proof_local_optimal} for all the technical details.    

\section{Discussion on Related works}\label{sec:dis_related}

\noindent\textbf{Offline reinforcement learning.} Information-theoretical considerations for offline RL are first proposed for \emph{infinite horizon discounted setting} via Fitted Q-Iteration (FQI) type function approximation algorithms \citep{chen2019information,le2019batch,xie2020batch,xie2020q} which can be traced back to \citep{munos2003error,szepesvari2005finite,antos2008fitted,antos2008learning}. Later, \cite{xie2020batch} considers the offline RL under only the \emph{realizability} assumption and \cite{liu2020provably} considers the offline RL \emph{without good exploration}. Those are all challenging problems but with they only provide suboptimal polynomial complexity in terms of $(1-\gamma)^{-1}$.

For the finite horizon case, \cite{yin2021near} first achieves $\tilde{O}(H^3/d_m\epsilon^2)$ complexity under non-stationary transition but their results cannot be further improved in the stationary setting. Concurrent to our work, a recently released work \cite{yin2021nearoptimal} designs the offline variance reduction algorithm for achieving the optimal $\tilde{O}(H^2/d_m\epsilon^2)$ rate.  Their result is for a specific algorithm that uses data splitting while our results work for any algorithms that returns a nearly empirically optimal policy via a uniform convergence guarantee. Our results on the offline task-agnostic and the reward-free settings are entirely new. Another concurrent work \cite{ren2021nearly} considers the horizon-free setting but does not provide uniform convergence guarantee. Even more recently, \cite{rashidinejad2021bridging} considers the single concentrability coefficient $C^\star:=\max_{s,a}\frac{d^{\pi^\star}(s,a)}{d^\mu(s,a)}$ and obtains the sample complexity $\tilde{O}[(1-\gamma)^{-5}SC^\star/\epsilon^2]$.

 In the linear MDP case, \cite{jin2020pessimism} studies the pessimism-based algorithms for offline policy optimization under the weak compliance assumption and \cite{wang2020statistical,zanette2020exponential} provide some negative results (exponential lower bound) for offline RL with linear MDP structure.

\noindent\textbf{Model-based approaches with minimaxity.}
It is known model-based methods are minimax-optimal for online RL with regret $\tilde{O}(\sqrt{HSAT})$ (\emph{e.g.} \cite{azar2017minimax,efroni2019tight}). For linear MDP, In the generative model setting, \cite{agarwal2020model} shows model-based approach is still minimax optimal $\tilde{O}((1-\gamma)^{-3}SA/\epsilon^2)$ by using a $s$-absorbing MDP construction and this model-based technique is later reused for other more general settings (\emph{e.g.} Markov games \citep{zhang2020model} and linear MDPs \citep{cui2020plug}) and also for improving the sample size barrier \citep{li2020breaking}. In offline RL, \cite{yu2020mopo,kidambi2020morel} use model-based approaches for continuous policy optimization and \cite{yin2021near} uses the model-based methods to achieve $\tilde{O}(H^3/d_m\epsilon^2)$ complexity.

\noindent\textbf{Task-agnostic and Reward-free problems.}
The reward-free problem is initiated in the online RL \citep{jin2020reward} where the agent needs to efficiently explore an MDP environment \emph{without} using any reward information. It requires high probability guarantee for learning optimal policy for \emph{any} reward function, which is strictly stronger than the standard learning task that one only needs to learn to optimal policy for a fixed reward. Later, \cite{kaufmann2020adaptive,menard2020fast} establish the $\tilde{O}(H^3S^2A/\epsilon^2)$ complexity and \cite{zhang2020nearly} further tightens the dependence to $\tilde{O}(H^2S^2A/\epsilon^2)$.\footnote{We translate \citep{zhang2020nearly} their dimension-free result to $\tilde{O}(H^2S^2A/\epsilon^2)$ under the standard assumption $r\in[0,1]$.} Recently, \cite{zhang2020task} proposes the task-agnostic setting where one needs to use exploration data to simultaneously learn $K$ tasks and provides a upper bound with complexity $\tilde{O}(H^5SA\log(K)/\epsilon^2)$. For linear MDP setting, \cite{wang2020reward} achieves the sample complexity $\tilde{O}(d^3H^6/\epsilon^2)$ and \cite{liu2020sharp} considers such problem in the online two-player Markov game. However, although these settings remain critical in the offline regime, no statistical result has been formally derived so far.

\section{Proof of optimal local uniform convergence}\label{sec:proof_local_optimal}

\subsection{Model-based Offline Plug-in Estimator}\label{subsec:method}

Recall the model-based estimator uses empirical estimator $\widehat{P}$ for estimating $P$ and the estimator is calculated accordingly:
\[
\widehat{Q}_h^\pi = r+ \widehat{P}^{\pi_{h+1}} Q^\pi_{h+1}=r+ \widehat{P}V^\pi_{H+1}, 
\]

where $\widehat{P}(s'|s,a)$ can be expressed as: 

\begin{align*}
\widehat{P}(s'|s,a)&=\frac{\sum_{i=1}^n\sum_{h=1}^H\mathbf{1}[(s^{(i)}_{h+1},a^{(i)}_h,s^{(i)}_h)=(s^\prime,s,a)]}{n_{s,a}},\quad n_{s,a} = \sum_{h=1}^H\sum_{i=1}^n \mathbf{1}[(s^{(i)}_h,a^{(i)}_h)=(s,a)].
\end{align*}

and $\widehat{P}(s^\prime|s,a)=\frac{1}{S}$, if $n_{s,a}=0$. The initial distribution is also constructed as $\widehat{d}_1^\pi(s)=n_{s}/n$.

First of all, we have by definition the Bellman optimality equation

\begin{equation}\label{eqn:Bellman_opt}
V^\star_t(s)=\max_a \left\{r(s,a)+\sum_{s'}P(s'|s,a)V^\star_{t+1}(s')\right\},\quad \forall s\in\mathcal{S}.
\end{equation}
and similarly the empirical version 
\[
\widehat{V}^\star_t(s)=\max_a \left\{r(s,a)+\sum_{s'}\widehat{P}(s'|s,a)\widehat{V}^\star_{t+1}(s')\right\},\quad \forall s\in\mathcal{S}.
\]
The key difficulty in obtaining the optimal dependence in stationary setting is decoupling the dependence of $P-\widehat{P}$ and $\widehat{V}^\star$. This issue is not encountered in the non-stationary setting due to the possibility to estimate different transition at each time \citep{yin2021near}, but it cannot further reduce the sample complexity on $H$. Moreover, the direct use of $s$-absorbing MDP in \cite{agarwal2020model} is not sharp for finite horizon stationary setting, as it requires $s$-absorbing MDPs with $H$-dimensional cover (which has size $\approx e^H$ and it is not optimal). We design the \emph{singleton-absorbing MDP} to get rid of the issue.  

\subsection{General absorbing MDP} 

The general absorbing MDP is defined as follows: for a fixed state $s$ and a sequence $\{u_t\}_{t=1}^H$, MDP $M_{s,\{u_t\}_{t=1}^H}$ is identical to $M$ for all states except $s$, and state $s$ is absorbing in the sense $P_{M_{s,\{u_t\}_{t=1}^H}}(s|s,a)=1$ for all $a$, and the instantaneous reward at time $t$ is $r_t(s,a)=u_t$ for all $a\in\mathcal{A}$. Also, we use the shorthand notation $V^\pi_{\{s,u_t\}}$ for $V^\pi_{s,{M_{s,\{u_t\}_{t=1}^H}}}$ and similarly for $Q_{\{s,u_t\}}$ and transition $P_{\{s,u_t\}}$. Then the following properties hold:
\vspace{1em}
\begin{lemma}\label{lem:absorb_value}
	\[
	V^\star_{h,\{s,u_t\}}(s)=\sum_{t=h}^H u_t.
	\]
\end{lemma}
\begin{proof}
	We prove by backward induction. For $h=H$, under $M_{s,\{u_t\}_{t=1}^H}$ state $s$ is absorbing (and by convention $V^\star_{H+1,\{s,u_t\}}=0$) therefore 
	\[
	V^\star_{H,\{s,u_t\}}(s)=\max_a\left\{r_{H,\{s,u_t\}}(s,a)+\sum_{s'}P_{\{s,u_t\}}(s'|s,a)V^\star_{H+1,\{s,u_t\}}(s')\right\}=\max_a\left\{r_{H,\{s,u_t\}}(s,a)\right\}=u_H
	\]
	for general $h$, note $\sum_{s'}P_{\{s,u_t\}}(s'|s,a)V^\star_{h+1,\{s,u_t\}}(s')=1\cdot V^\star_{h+1,\{s,u_t\}}(s)$, therefore using induction property $V^\star_{h+1,\{s,u_t\}}(s)=\sum_{t=h+1}^H u_t$ we can similarly obtain $V^\star_{h,\{s,u_t\}}(s)=\sum_{t=h}^H u_t$.
	
\end{proof}

\begin{lemma}\label{lem:q_diff}
	Fix state $s$. For two different sequences $\{u_t\}_{t=1}^H$ and $\{u_t^\prime\}_{t=1}^H$, we have 
	\[
	\max_h\norm{Q^\star_{h,\{s,u_t\}}-Q^\star_{h,\{s,u_t^\prime\}}}_\infty\leq H\cdot\max_{t\in[H]}\left|u_t-u^\prime_t\right|.
	\]
	
\end{lemma}

\begin{proof}
	Let $\pi^\star_{\{s,u_t\}}$ be the optimal policy in $M_{\{s,u_t\}}$. Then (by convention $\prod_{a=h+1}^{h}P^{\pi_a}=I$)
	\begin{align*}
	&Q^\star_{h,\{s,u_t\}}-Q^\star_{h,\{s,u_t^\prime\}}=Q^\star_{h,\{s,u_t\}}-\max_\pi \sum_{i=h}^H \left(\prod_{a=h+1}^{i} P^{\pi_{a}}_{\{s,u'_t\}}\right)r_{i,\{s,u'_t\}}\\
	\leq& Q^\star_{h,\{s,u_t\}}- \sum_{i=h}^H \left(\prod_{a=h+1}^{i} P^{\pi^\star_{a,\{s,u_t\}}}_{\{s,u'_t\}}\right)r_{i,\{s,u'_t\}}=\sum_{i=h}^H \left(\prod_{a=h+1}^{i} P^{\pi^\star_{a,\{s,u_t\}}}_{\{s,u'_t\}} \right)\left(r_{i,\{s,u_t\}}-r_{i,\{s,u'_t\}}\right)\\
	&\leq \sum_{i=h}^H\max_{s,a}\norm{\left(\prod_{a=h+1}^{i} P^{\pi^\star_{a,\{s,u_t\}}}_{\{s,u'_t\}} \right)^{i-h}(\cdot|s,a)}_1\cdot \norm{r_{i,\{s,u_t\}}-r_{i,\{s,u'_t\}}}_\infty\cdot \mathbf{1}=(H-h+1)\cdot \max_t\left|u_t-u'_t\right|\cdot\mathbf{1}\\
	\end{align*}
	where the first equal sign uses the definition of $Q^\star$, the second equal sign uses $P_{\{s,u_t\}}$ only depends $s$ but not the specification of $u_t$'s and the last equal sign comes from $r_{i,\{s,u_t\}}(s,a)=u_i$ for any $a\in\mathcal{A}$ and $r_{i,\{s,u_t\}}(\tilde{s},a)=r_{i,\{s,u'_t\}}(\tilde{s},a)$ for any $\tilde{s}\neq s$. Lastly by symmetry we finish the proof.
\end{proof}

\subsection{Singleton-absorbing MDP}

The direct of transfer of absorbing technique created in \cite{agarwal2020model} will require each $u_t$ to fill in the range of $[0,H]$ using evenly spaced elements. For finite horizon MDP there are $H$ layers, therefore the total number of $H$-tuples $(u_1,\ldots,u_H)$ has order $|U_{s}|=Poly(H)^H$, therefore when apply the union bound, it will incur the additional $H$ factor. We get rid of this issue by choosing one single point in $H$-dimensional space $[0,H]^H$. We first give the following two lemmas.  
\vspace{1em}
\begin{lemma}\label{lem:pos_diff}
	$ V^\star_t(s)-V^\star_{t+1}(s)\geq 0$, for all state $s\in\mathcal{S}$ and all $t\in[H]$.	
\end{lemma}

\begin{proof}
	Let the optimal policy for $V^\star_{t+1}$ be $\pi^\star_{t+1:H}$, \emph{i.e.} $V^\star_{t+1}=V^{\pi^\star_{t+1:H}}_{t+1}$, then artificially construct a policy $\pi_{t:H}$ such that $\pi_{t:H-1}=\pi^\star_{t+1:H}$ and $\pi_H$ is arbitrary, then by the definition of optimal value 
	\begin{align*}
	V^\star_{t}(s)&\geq V^{\pi_{t:H}}_t(s)=\E^{\pi_{t:H}}\left[\sum_{i=t}^H r(s_i,a_i)\middle|s_t=s\right]\\
	&=\E^{\pi_{t:H-1}}\left[\sum_{i=t}^{H-1} r(s_i,a_i)\middle|s_t=s\right]+\E^{\pi_{t:H}}\left[ r(s_H,a_H)\middle|s_t=s\right]\\
	&=\E^{\pi^\star_{t+1:H}}\left[\sum_{i=t+1}^{H} r(s_i,a_i)\middle|s_{t+1}=s\right]+\E^{\pi_{t:H}}\left[ r(s_H,a_H)\middle|s_t=s\right]\\
	&\geq \E^{\pi^\star_{t+1:H}}\left[\sum_{i=t+1}^{H} r(s_i,a_i)\middle|s_{t+1}=s\right]+0=V^\star_{t+1}(s),
	\end{align*}
	where the third equal sign uses exactly that \emph{$P$ is a STATIONARY transition} and definition $\pi_{t:H-1}=\pi^\star_{t+1:H}$. The last inequality uses assumption that reward is always non-negative.
\end{proof}

\begin{remark}
	Lemma~\ref{lem:pos_diff} leverages $P$ is stationary and above may not be true in the non-stationary setting. This enables us to establish the following lemma, which is the key for singleton-absorbing MDP.
\end{remark}

\begin{lemma}\label{lem:smdp_prop}
	Fix a state $s$. If we choose $u_t^\star:= V^\star_t(s)-V^\star_{t+1}(s)$ $\forall t\in[H]$, then we have the following vector form equation 
	\[
	V^\star_{h,\{s,u_t^\star\}}=V^\star_{h,M}\quad \forall h\in[H].
	\]
	Similarly, if we choose $\widehat{u}_t^\star:= \widehat{V}^\star_t(s)-\widehat{V}^\star_{t+1}(s)$, then $\widehat{V}^\star_{h,\{s,\hat{u}_t^\star\}}=\widehat{V}^\star_{h,M}$, $\forall h\in[H]$.
\end{lemma}

\begin{proof} 
	We focus on the first claim. Note by Lemma~\ref{lem:pos_diff} the assignment of $u^\star_t(:=r_{t,\{s,u^\star_t\}})$ is well-defined. Next recall $V^\star_{h,M}$ is the optimal value under true MDP $M$ and $V^\star_{h,\{s,u_t^\star\}}$ is the optimal value under the assimilating MDP ${M_{s,\{u_t^\star\}_{t=1}^H}}$. We prove by backward induction. 
	
	For $h=H$, note by convention $V^\star_{H+1}=0$, therefore $u^\star_H=V^\star_{H}(s)-V^\star_{H+1}(s)=V^\star_{H}(s)-0=V^\star_{H}(s)$ and Bellman optimality equation becomes 
	\[
	V^\star_H(\tilde{s})=\max_a \left\{r(\tilde{s},a)\right\},\quad \forall \tilde{s}\in\mathcal{S}.
	\]
	Under ${M_{s,\{u_t^\star\}_{t=1}^H}}$, for state $s$ by Lemma~\ref{lem:absorb_value} we have
	$
	V^\star_{H,\{s,u_t^\star\}}(s)=u^\star_H=V^\star_{H}(s),
	$
	for other states $\tilde{s}\neq s$, reward in ${M_{s,\{u_t^\star\}_{t=1}^H}}=M$ so we also have $V^\star_{H,\{s,u_t^\star\}}(\tilde{s})=V^\star_{H}(\tilde{s})$ for all $\tilde{s}\neq s$.
	
	Now for general $h$, for state $s$ by Lemma~\ref{lem:absorb_value}
	\[
	V^\star_{h,\{s,u_t^\star\}}(s)=\sum_{t=h}^H u_t^\star=\sum_{t=h}^H \left(V^\star_t(s)-V^\star_{t+1}(s)\right)=V^\star_h(s),
	\]
	for state $\tilde{s}\neq s$, by Bellman optimality equation
	\begin{align*}
	V^\star_{h,\{s,u_t^\star\}}(\tilde{s})&=\max_a \left\{r_{\{s,u_t^\star\}}(\tilde{s},a)+\sum_{s'}P_{\{s,u_t^\star\}}(s'|\tilde{s},a)V^\star_{h+1,\{s,u_t^\star\}}(s')\right\}\\
	&=\max_a \left\{r(\tilde{s},a)+\sum_{s'}P(s'|\tilde{s},a)V^\star_{h+1,\{s,u_t^\star\}}(s')\right\}\\
	&=\max_a \left\{r(\tilde{s},a)+\sum_{s'}P(s'|\tilde{s},a)V^\star_{h+1}(s')\right\}=V^\star_h(\tilde{s}),\\
	\end{align*}
	where the second equal sign uses when $\tilde{s}\neq s$, ${M_{s,\{u_t^\star\}_{t=1}^H}}$ is identical to $M$ and the third equal sign uses induction assumption that element-wisely $V^\star_{h+1,\{s,u_t^\star\}}=V^\star_{h+1}$. Similar result can be derived for $\hat{u}^\star$ version and this completes the proof.
	
\end{proof}

The singleton MDP we used is exactly $M_{s,\{u^\star_t\}_{t=1}^H}$ (or $\widehat{M}_{s,\{u^\star_t\}_{t=1}^H}$). 

\subsection{Proof for local uniform convergence}\label{sec:dec}

Recall the local policy class
\[
\Pi_{l}:=\left\{\pi: \text { s.t. }\left\|\widehat{V}_{h}^{\pi}-\widehat{V}_{h}^{\widehat{\pi}^{\star}}\right\|_{\infty} \leq \epsilon_{\text {opt }}, \forall h\in[H]\right\}.
\] 
For ease of exposition, we denote $N:=\min_{s,a}n_{s,a}$. Note $N$ itself is a random variable, therefore for the rest of proof we first conditional on $N$. Later we shall remove the conditional on $N$ ({see Section~\ref{sec:final_pf}}).

For any $\widehat{\pi}\in\Pi_l$, by (empirical) Bellman equation we have element-wisely:
\begin{align*}
\widehat{Q}_h^{\widehat{\pi}}- {Q}_h^{\widehat{\pi}}&=r_h+\widehat{P}^{\widehat{\pi}_{h+1}} \widehat{Q}_{h+1}^{\widehat{\pi}}-r_h-{P}^{\widehat{\pi}_{h+1}} {Q}_{h+1}^{\widehat{\pi}}\\
&=\left(\widehat{P}^{\widehat{\pi}_{h+1}}-{P}^{\widehat{\pi}_{h+1}}\right)\widehat{Q}_{h+1}^{\widehat{\pi}} + {P}^{\widehat{\pi}_{h+1}} \left(\widehat{Q}_{h+1}^{\widehat{\pi}}-{Q}_{h+1}^{\widehat{\pi}}\right)\\
&=\left(\widehat{P}-{P}\right)\widehat{V}_{h+1}^{\widehat{\pi}}+ {P}^{\widehat{\pi}_{h+1}} \left(\widehat{Q}_{h+1}^{\widehat{\pi}}-{Q}_{h+1}^{\widehat{\pi}}\right)\\
&=\ldots=\sum_{t=h}^H\Gamma_{h+1:t}^{\widehat{\pi}}\left(\widehat{P}-{P}\right)\widehat{V}_{t+1}^{\widehat{\pi}}\\
&\leq \underbrace{\sum_{t=h}^H\Gamma_{h+1:t}^{\widehat{\pi}}\left|\left(\widehat{P}-{P}\right)\widehat{V}_{t+1}^{\widehat{\pi}^\star}\right|}_{(\text{$\star$})}+\underbrace{\sum_{t=h}^H\Gamma_{h+1:t}^{\widehat{\pi}}\left|\left(\widehat{P}-{P}\right)\left(\widehat{V}_{t+1}^{\widehat{\pi}}-\widehat{V}_{t+1}^{\widehat{\pi}^\star}\right)\right|}_{(\text{$\star \star$})}\\
\end{align*}

where $\Gamma_{h+1:t}^\pi=\prod_{i=h+1}^t P^{\pi_i}$ is multi-step state-action transition and $\Gamma_{h+1:h}:=I$.

\subsection{Analyzing (\text{$\star \star$})} 

Term (\text{$\star \star$}) can be readily bounded using the following lemma.

\begin{lemma}\label{lem:epi_opt_bound}
	Fix $N>0$, we have with probability $1-\delta$, for all $t=1,...,H-1$
	\[
	\sum_{t=h}^H \Gamma_{h+1:t}^{\hpi}\left|(\widehat{P}-{P}) (\widehat{V}^{\widehat{\pi}^\star}_{h+1}-\widehat{V}^{{\widehat{\pi}}}_{h+1})\right|\leq C\epsilon_\mathrm{opt}\cdot \sqrt{\frac{H^2S\log(SA/\delta)}{N}}\cdot\mathbf{1}
	\]
	where $C$ absorb the higher order term and absolute constants.
\end{lemma}

\begin{proof}

First, by vector induced matrix norm\footnote{For $A$ a matrix and $x$ a vector we have $\norm{Ax}_\infty\leq\norm{A}_\infty\norm{x}_\infty$.} we have 
\begin{align*}
\norm{\sum_{t=h}^H \Gamma_{h+1:t}^{\widehat{\pi}}\cdot\left|(\widehat{P}-{P}) (\widehat{V}^{\widehat{\pi}^\star}_{t+1}-\widehat{V}^{\widehat{\pi}}_{t+1})\right|}_\infty&\leq H\cdot \sup_t\norm{\Gamma_{h+1:t}^{\widehat{\pi}}}_\infty\norm{|(\widehat{P}-{P}) (\widehat{V}^{\widehat{\pi}^\star}_{t+1}-\widehat{V}^{{\widehat{\pi}}}_{t+1})|}_\infty\\
&\leq H\cdot \sup_t\norm{|(\widehat{P}-{P}) (\widehat{V}^{\widehat{\pi}^\star}_{t+1}-\widehat{V}^{{\widehat{\pi}}}_{t+1})|}_\infty\\
&= H\cdot \sup_{t,s,a}\left|(\widehat{P}-{P})(\cdot|s,a) (\widehat{V}^{\widehat{\pi}^\star}_{t+1}-\widehat{V}^{{\widehat{\pi}}}_{t+1})\right|\\
&\leq H\cdot \sup_{t,s,a} \norm{(\widehat{P}-{P})(\cdot|s,a)}_1\cdot \norm{\widehat{V}^{\widehat{\pi}^\star}_{t+1}-\widehat{V}^{{\widehat{\pi}}}_{t+1}}_\infty\cdot\mathbf{1}
\end{align*}
where the second inequality uses multi-step transition $\Gamma_{t+1:h-1}^\pi$ is row-stochastic. Note given $N$, therefore by Lemma~\ref{lem:l1_upper} and a union bound we have with probability $1-\delta$,
\[
\sup_{s,a} \norm{(\widehat{P}-{P})(\cdot|s,a)}_1\leq {C}(\sqrt{\frac{S\log(SA/\delta)}{N}}),
\]
(where $C$ absorb the higher order term and absolute constants) and using definition of $\Pi_l$ we have $\sup_t\norm{\widehat{V}^{\widehat{\pi}^\star}_{t}-\widehat{V}^{{\widehat{\pi}}}_{t}}_\infty\leq \epsilon_{\text {opt }}$.
This indicates
\[
\sup_{t,s,a} \norm{(\widehat{P}-{P})(\cdot|s,a)}_1\cdot \norm{\widehat{V}^{\widehat{\pi}^\star}_{t+1}-\widehat{V}^{{\widehat{\pi}}}_{t+1}}_\infty\cdot\mathbf{1}\leq C(\epsilon_{\text {opt }}\sqrt{\frac{S\log(SA/\delta)}{N}}\cdot\mathbf{1}),
\]
where $\mathbf{1}\in\R^{S}$ is all-one vector. Then multiple by $H$ to get the stated result.

\end{proof}

\subsection{Analyzing (\text{$\star$})} 
\textbf{Concentration on $\left(\widehat{P}-P\right)\hoV_h$.}\footnote{Here we use $\hoV_h$ instead of $\hoV_t$ since we later have $\hoV_{h,\{s,{u}^\star_t\}}$. We avoid the same $t$ twice in the expression to prevent confusion.} Since $\hP$ aggregates all data from different step so that $\hP$ and $\hoV_h$ are on longer independent, Bernstein inequality cannot be directly applied. We use the singleton-absorbing MDP ${M_{s,\{u_t^\star\}_{t=1}^H}}$ to handle the case (recall $u_t^\star:= V^\star_t(s)-V^\star_{t+1}(s)$ $\forall t\in[H]$).  Again, let us fix a state $s$ and $a\in\mathcal{A}$ be any action. Also, we use $P_{s,a}$ to denote row vector to avoid long expression. Then we have:
\begin{equation}\label{eqn:bern_absorb}
\begin{aligned}
&\left( \hP_{s,a}-P_{s,a} \right)\hoV_h=\left( \hP_{s,a}-P_{s,a} \right)\left(\hoV_h-\hoV_{h,\{s,u^\star_t\}}+\hoV_{h,\{s,u^\star_t\}}\right)\\
 =& \left( \hP_{s,a}-P_{s,a} \right)\left(\hoV_h-\hoV_{h,\{s,u^\star_t\}}\right)+ \left( \hP_{s,a}-P_{s,a} \right)\hoV_{h,\{s,u^\star_t\}}\\
\leq& \norm{\hP_{s,a}-P_{s,a}}_1\norm{\hoV_h-\hoV_{h,\{s,u^\star_t\}}}_\infty+\sqrt{\frac{2\log(4/\delta)}{N}}\sqrt{\Var_{s,a}(\hoV_{h,\{s,u^\star_t\}})}+\frac{2H\log(1/\delta)}{3N}\\
\leq& \norm{\hP_{s,a}-P_{s,a}}_1\norm{\hoV_h-\hoV_{h,\{s,u^\star_t\}}}_\infty+\sqrt{\frac{2\log(4/\delta)}{N}}\left(\sqrt{\Var_{s,a}(\hoV_{h})}+\sqrt{\Var_{s,a}(\hoV_{h,\{s,u^\star_t\}}-\hoV_{h})}\right)+\frac{2H\log(1/\delta)}{3N}\\
\leq& \norm{\hP_{s,a}-P_{s,a}}_1\norm{\hoV_h-\hoV_{h,\{s,u^\star_t\}}}_\infty+\sqrt{\frac{2\log(4/\delta)}{N}}\left(\sqrt{\Var_{s,a}(\hoV_{h})}+\sqrt{\norm{\hoV_{h,\{s,u^\star_t\}}-\hoV_{h}}^2_\infty}\right)+\frac{2H\log(1/\delta)}{3N}\\
=& \left(\norm{\hP_{s,a}-P_{s,a}}_1+\sqrt{\frac{2\log(4/\delta)}{N}}\right)\norm{\hoV_h-\hoV_{h,\{s,u^\star_t\}}}_\infty+\sqrt{\frac{2\log(4/\delta)}{N}}\sqrt{\Var_{s,a}(\hoV_{h})}+\frac{2H\log(1/\delta)}{3N}\\
\end{aligned}
\end{equation}

where the first inequality uses Bernstein inequality (Lemma~\ref{lem:bernstein_ineq}), the second inequality uses $\sqrt{\Var(\cdot)}$ is norm (norm triangle inequality). Now we treat $\norm{\hP_{s,a}-P_{s,a}}_1$ and $\norm{\hoV_h-\hoV_{h,\{s,u^\star_t\}}}_\infty$ separately.

\textbf{For $\norm{\hP_{s,a}-P_{s,a}}_1$.} Indeed, by Lemma~\ref{lem:l1_upper} again $\norm{\hP_{s,a}-P_{s,a}}_1\leq\tilde{O}(\sqrt{\frac{S\log(S/\delta)}{N}})$ and by a union bound we obtain w.p., $1-\delta$
\begin{equation}\label{eqn:l1}
\sup_{s,a}\norm{\hP_{s,a}-P_{s,a}}_1\leq C\sqrt{\frac{S\log(SA/\delta)}{N}}.
\end{equation}
where $C$ absorbs the higher order term and constants.

\textbf{For $\norm{\hoV_h-\hoV_{h,\{s,u^\star_t\}}}_\infty$.} Note if we set $\widehat{u}^\star_t=\widehat{V}^\star_t(s)-\widehat{V}^\star_{t+1}(s)$, then by Lemma~\ref{lem:smdp_prop}
\[
\hoV_h=\hoV_{h,\{s,\hat{u}^\star_t\}}
\]
Next since $\hoV_{h,\{s,\hat{u}^\star_t\}}(\tilde{s})=\max_a \widehat{Q}^\star_{h,\{s,\hat{u}^\star_t\}}(\tilde{s},a)$ $\forall \tilde{s}\in\mathcal{S}$, by generic inequality $|\max f-\max g|\leq \max|f-g|$, we have $|\hoV_{h,\{s,\hat{u}^\star_t\}}(\tilde{s})-\hoV_{h,\{s,{u}^\star_t\}}(\tilde{s})|\leq \max_a |\widehat{Q}^\star_{h,\{s,\hat{u}^\star_t\}}(\tilde{s},a)-\widehat{Q}^\star_{h,\{s,{u}^\star_t\}}(\tilde{s},a)|$, taking $\max_{\tilde{s}}$ on both sides, we obtain exactly
\[
\norm{\hoV_{h,\{s,\hat{u}^\star_t\}}-\hoV_{h,\{s,{u}^\star_t\}}}_\infty\leq \norm {\widehat{Q}^\star_{h,\{s,\hat{u}^\star_t\}}-\widehat{Q}^\star_{h,\{s,{u}^\star_t\}}}_\infty
\]
then by Lemma~\ref{lem:q_diff}, 
\begin{equation}\label{eqn:higher_diff}
\norm{\hoV_h-\hoV_{h,\{s,u^\star_t\}}}_\infty\leq \norm {\widehat{Q}^\star_{h,\{s,\hat{u}^\star_t\}}-\widehat{Q}^\star_{h,\{s,{u}^\star_t\}}}_\infty\leq H\max_t\left|\hat{u}^\star_t-u^\star_t\right|,
\end{equation}
 Recall 
\[
\hat{u}^\star_t-u^\star_t=\widehat{V}^\star_t(s)-\widehat{V}^\star_{t+1}(s)-\left(V^\star_t(s)-V^\star_{t+1}(s)\right).
\]
Now we denote 
\[
\Delta_s:= \max_t|\hat{u}^\star_t-u^\star_t|=\max_t\left|\widehat{V}^\star_t(s)-\widehat{V}^\star_{t+1}(s)-\left(V^\star_t(s)-V^\star_{t+1}(s)\right)\right|,
\]
then $\Delta_s$ itself is a scalar and a random variable.

To sum up, by \eqref{eqn:bern_absorb}, \eqref{eqn:l1} and \eqref{eqn:higher_diff} and a union bound we have 
\begin{lemma}\label{lem:inter_bern}
	Fix $N>0$. With probability $1-\delta$, element-wisely, for all $h\in[H]$,
	\begin{align*}
	\left|\left( \hP-P \right)\hoV_h\right|\leq C\sqrt{\frac{S\log(HSA/\delta)}{N}}\cdot H\max_s\Delta_s\cdot \mathbf{1}+\sqrt{\frac{2\log(4HSA/\delta)}{N}}\sqrt{\Var_{P}(\hoV_{h})}+\frac{2H\log(HSA/\delta)}{3N}\cdot\mathbf{1}
	\end{align*}
	
\end{lemma}

Now plug Lemma~\ref{lem:inter_bern} back into ($\star$) and combine Lemma~\ref{lem:epi_opt_bound}, we receive:

\begin{align*}
&\left|\widehat{Q}^{\widehat{\pi}}_h-Q^{\widehat{\pi}}_h\right|\\
\leq& 
\sum_{t=h}^H\Gamma_{h+1:t}^{\widehat{\pi}}\left(C\sqrt{\frac{S\log(HSA/\delta)}{N}}\cdot H\max_s\Delta_s\cdot \mathbf{1}+\sqrt{\frac{2\log(4HSA/\delta)}{N}}\sqrt{\Var_{P}(\hoV_{t+1})}+\frac{2H\log(HSA/\delta)}{3N}\cdot\mathbf{1}\right)\\
+&C\epsilon_\mathrm{opt}\cdot \sqrt{\frac{H^2S\log(SA/\delta)}{N}}\cdot\mathbf{1}\\
\leq& \sum_{t=h}^H\Gamma_{h+1:t}^{\widehat{\pi}}\sqrt{\frac{2\log(4HSA/\delta)}{N}}\sqrt{\Var_{P}(\hoV_{t+1})} +CH^2\sqrt{\frac{S\log(HSA/\delta)}{N}}\cdot \max_s\Delta_s\cdot \mathbf{1}+\frac{2H^2\log(HSA/\delta)}{3N}\cdot\mathbf{1}\\
+&C\epsilon_\mathrm{opt}\cdot \sqrt{\frac{H^2S\log(SA/\delta)}{N}}\cdot\mathbf{1}\\
\end{align*}

Next note
\begin{equation}\label{eqn:var}
\begin{aligned}
&\sqrt{\Var_{P}(\hoV_{h})}:=\sqrt{\Var_{P}\left(\widehat{V}^{\widehat{\pi}^\star}_{h}\right)}=\sqrt{\Var_{P}\left(\widehat{V}^{\widehat{\pi}^\star}_{h}-\widehat{V}^{\widehat{\pi}}_h+\widehat{V}^{\widehat{\pi}}_h\right)}\\
\leq& \sqrt{\Var_{P}\left(\widehat{V}^{\widehat{\pi}}_h\right)}+\sqrt{\Var_{P}\left(\widehat{V}^{\widehat{\pi}^\star}_{h}-\widehat{V}^{\widehat{\pi}}_h\right)}\leq \sqrt{\Var_{P}\left(\widehat{V}^{\widehat{\pi}}_h\right)}+\norm{\widehat{V}^{\widehat{\pi}^\star}_{h}-\widehat{V}^{\widehat{\pi}}_h}_\infty\\
\leq& \sqrt{\Var_{P}\left(\widehat{V}^{\widehat{\pi}}_h\right)}+\epsilon_{\text {opt }}\cdot\mathbf{1}\leq \sqrt{\Var_{P}\left(V^{\widehat{\pi}}_h\right)}+\sqrt{\Var_{P}\left(\widehat{V}^{\widehat{\pi}}_{h}-{V}^{\widehat{\pi}}_h\right)}+\epsilon_{\text {opt }}\cdot\mathbf{1}\\
\leq& \sqrt{\Var_{P}\left(V^{\widehat{\pi}}_h\right)}+\norm{\widehat{V}^{\widehat{\pi}}_{h}-{V}^{\widehat{\pi}}_h}_\infty+\epsilon_{\text {opt }}\cdot\mathbf{1}\leq \sqrt{\Var_{P}\left(V^{\widehat{\pi}}_h\right)}+\norm{\widehat{Q}^{\widehat{\pi}}_{h}-{Q}^{\widehat{\pi}}_h}_\infty+\epsilon_{\text {opt }}\cdot\mathbf{1}\\
\end{aligned}
\end{equation}

Plug \eqref{eqn:var} back to above we obtain $\forall h\in[H]$,
\begin{equation}\label{eqn:decomp}
\begin{aligned}
&\left|\widehat{Q}^{\widehat{\pi}}_h-Q^{\widehat{\pi}}_h\right|
\leq \sqrt{\frac{2\log(4HSA/\delta)}{N}}\sum_{t=h}^H\Gamma_{h+1:t}^{\widehat{\pi}}\left(\sqrt{\Var_{P}\left(V^{\widehat{\pi}}_{t+1}\right)}+\norm{\widehat{Q}^{\widehat{\pi}}_{t+1}-{Q}^{\widehat{\pi}}_{t+1}}_\infty+\epsilon_{\text {opt }}\cdot\mathbf{1}\right)\\
& +CH^2\sqrt{\frac{S\log(HSA/\delta)}{N}}\cdot \max_s\Delta_s\cdot \mathbf{1}+\frac{2H^2\log(HSA/\delta)}{3N}\cdot\mathbf{1}+C\epsilon_\mathrm{opt}\cdot \sqrt{\frac{H^2S\log(SA/\delta)}{N}}\cdot\mathbf{1}\\
&\leq \sqrt{\frac{2\log(4HSA/\delta)}{N}}\sum_{t=h}^H\Gamma_{h+1:t}^{\widehat{\pi}}\sqrt{\Var_{P}\left(V^{\widehat{\pi}}_{t+1}\right)}+\sqrt{\frac{2\log(4HSA/\delta)}{N}}\sum_{t=h}^H\norm{\widehat{Q}^{\widehat{\pi}}_{t+1}-{Q}^{\widehat{\pi}}_{t+1}}_\infty\\
& +CH^2\sqrt{\frac{S\log(HSA/\delta)}{N}}\cdot \max_s\Delta_s\cdot \mathbf{1}+\frac{2H^2\log(HSA/\delta)}{3N}\cdot\mathbf{1}+C_1\epsilon_\mathrm{opt}\cdot \sqrt{\frac{H^2S\log(SA/\delta)}{N}}\cdot\mathbf{1}\\
\end{aligned}
\end{equation}
Apply Lemma~\ref{lem:H3toH2} and the coarse uniform bound (Lemma~\ref{lem:crude_u_b}) we obtain the following lemma:

\begin{lemma}\label{lem:recursion}
	Given $N>0$ and $\epsilon_{\text {opt }}\leq \sqrt{H/S}$.  With probability $1-\delta$, for all $h\in[H]$, 
	\begin{align*}
	\norm{\widehat{Q}^{\widehat{\pi}}_h-Q^{\widehat{\pi}}_h}_\infty
	\leq \sqrt{\frac{C_0 H^3\log(4HSA/\delta)}{N}}+\sqrt{\frac{2\log(4HSA/\delta)}{N}}\sum_{t=h}^H\norm{\widehat{Q}^{\widehat{\pi}}_{t+1}-{Q}^{\widehat{\pi}}_{t+1}}_\infty +C'H^4\frac{S\log(HSA/\delta)}{N}
	\end{align*}
\end{lemma}

\begin{proof}
Since
\begin{equation}\label{eqn:delta_cb}
\begin{aligned}
\Delta_s&:= \max_t|\hat{u}^\star_t-u^\star_t|=\max_t\left|\widehat{V}^\star_t(s)-\widehat{V}^\star_{t+1}(s)-\left(V^\star_t(s)-V^\star_{t+1}(s)\right)\right|\\
&\leq 2\cdot \max_t \left|\widehat{V}^\star_t(s)-V^\star_t(s)\right|\\
&= 2\cdot \max_t \left|\max_\pi \widehat{V}^\pi_t(s)-\max_\pi V^\pi_t(s)\right|\\
&\leq 2\cdot \max_{\pi\in\Pi_g,t\in[H]}\norm{\widehat{V}_t^\pi-V_t^\pi}_\infty\leq C\cdot H^2 \sqrt{\frac{S\log(HSA/\delta)}{N}}
\end{aligned}
\end{equation}
where the last inequality uses Lemma~\ref{lem:crude_u_b}. Then apply union bound w.p. $1-\delta/2$, we obtain $\max_s \Delta_s\leq C\cdot H^2 \sqrt{\frac{S\log(HSA/\delta)}{N}}$. Note \eqref{eqn:decomp} holds with probability $1-\delta/2$, therefore plug above into \eqref{eqn:decomp} we obtain w.p. $1-\delta$,
\begin{align*}
	&\left|\widehat{Q}^{\widehat{\pi}}_h-Q^{\widehat{\pi}}_h\right|
	\leq \sqrt{\frac{2\log(4HSA/\delta)}{N}}\sum_{t=h}^H\Gamma_{h+1:t}^{\widehat{\pi}}\sqrt{\Var_{P}\left(V^{\widehat{\pi}}_{t+1}\right)}+\sqrt{\frac{2\log(4HSA/\delta)}{N}}\sum_{t=h}^H\norm{\widehat{Q}^{\widehat{\pi}}_{t+1}-{Q}^{\widehat{\pi}}_{t+1}}_\infty\\
	& +C'H^4\frac{S\log(HSA/\delta)}{N}\cdot \mathbf{1}+C_1\epsilon_\mathrm{opt}\cdot \sqrt{\frac{H^2S\log(SA/\delta)}{N}}\cdot\mathbf{1}\\
	&\leq \left[\sqrt{\frac{C_0H^3\log(4HSA/\delta)}{N}}+\sqrt{\frac{2\log(4HSA/\delta)}{N}}\sum_{t=h}^H\norm{\widehat{Q}^{\widehat{\pi}}_{t+1}-{Q}^{\widehat{\pi}}_{t+1}}_\infty +C'H^4\frac{S\log(HSA/\delta)}{N}\right]\cdot \mathbf{1},\\
\end{align*}
where the last inequality uses Lemma~\ref{lem:H3toH2} and $\epsilon_{\text {opt }}\leq \sqrt{H/S}$ and renames $C'=C'+C_1$. Take $\norm{\cdot}_\infty$ then obtain the result.
\end{proof}

\begin{lemma}\label{lem:inter_final}
	Given $N>0$. Define $C^{\prime\prime}:=2\cdot\max(\sqrt{C_0},C')$ where $C'$ is the universal constant in Lemma~\ref{lem:recursion}. When $N\geq 8H^2\log(4HSA/\delta)$, then with probability $1-\delta$, $\forall h\in[H]$,
		\begin{equation}\label{eqn:inter_final}
		\begin{aligned}
	\norm{\widehat{Q}^{\widehat{\pi}}_h-Q^{\widehat{\pi}}_h}_\infty
	\leq C^{\prime\prime}\sqrt{\frac{H^3\log(4HSA/\delta)}{N}}+C^{\prime\prime}\frac{H^4S\log(HSA/\delta)}{N}.\\
	\norm{\widehat{Q}^{{\pi}^\star}_h-Q^{{\pi}^\star}_h}_\infty
	\leq C^{\prime\prime}\sqrt{\frac{H^3\log(4HSA/\delta)}{N}}+C^{\prime\prime}\frac{H^4S\log(HSA/\delta)}{N}.\\
	\end{aligned}
	\end{equation}
	
\end{lemma}

\begin{proof}
	We prove by backward induction. For $h=H$, by Lemma~\ref{lem:recursion}
	\begin{align*}
		\norm{\widehat{Q}^{\widehat{\pi}}_H-Q^{\widehat{\pi}}_H}_\infty
	&\leq \sqrt{\frac{C_0H^3\log(4HSA/\delta)}{N}}+\sqrt{\frac{2\log(4HSA/\delta)}{N}}\norm{\widehat{Q}^{\widehat{\pi}}_{H+1}-{Q}^{\widehat{\pi}}_{H+1}}_\infty +C'H^4\frac{S\log(HSA/\delta)}{N}\\
	&=\sqrt{\frac{C_0H^3\log(4HSA/\delta)}{N}}+0 +C'H^4\frac{S\log(HSA/\delta)}{N}\\
	&\leq C^{\prime\prime}\sqrt{\frac{H^3\log(4HSA/\delta)}{N}}+C''H^4\frac{S\log(HSA/\delta)}{N},\\
	\end{align*}
	for general $h$, by condition we have $H\sqrt{\frac{2\log(4HSA/\delta)}{N}}\leq 1/2$, therefore by Lemma~\ref{lem:recursion}
	\begin{align*}
	&\norm{\widehat{Q}^{\widehat{\pi}}_h-Q^{\widehat{\pi}}_h}_\infty
	\leq \sqrt{\frac{C_0H^3\log(4HSA/\delta)}{N}}+\sqrt{\frac{2\log(4HSA/\delta)}{N}}\sum_{t=h}^H\norm{\widehat{Q}^{\widehat{\pi}}_{t+1}-{Q}^{\widehat{\pi}}_{t+1}}_\infty +C'H^4\frac{S\log(HSA/\delta)}{N}\\
	&\leq \sqrt{\frac{C_0H^3\log(4HSA/\delta)}{N}}+H\sqrt{\frac{2\log(4HSA/\delta)}{N}}\max_{t+1}\norm{\widehat{Q}^{\widehat{\pi}}_{t+1}-{Q}^{\widehat{\pi}}_{t+1}}_\infty +C'H^4\frac{S\log(HSA/\delta)}{N}\\
	&\leq \sqrt{\frac{C_0H^3\log(4HSA/\delta)}{N}} +C'H^4\frac{S\log(HSA/\delta)}{N}\\
	&+\frac{1}{2}\left(C^{\prime\prime}\sqrt{\frac{H^3\log(4HSA/\delta)}{N}}+C^{\prime\prime}\frac{H^4S\log(HSA/\delta)}{N}\right)\\
	&\leq C^{\prime\prime}\sqrt{\frac{H^3\log(4HSA/\delta)}{N}}+C^{\prime\prime}\frac{H^4S\log(HSA/\delta)}{N}
	\end{align*}
	
	The proof of the second claim is even easier since $\pi^\star$ is no longer a random policy and it is really just a non-uniform point-wise OPE. There are multiple ways to prove it and we leave it as an exercise to avoid redundancy: 1. Follow the same proving pipeline as $\norm{\widehat{Q}^{\widehat{\pi}}_h-Q^{\widehat{\pi}}_h}_\infty$ used; 2. Mimic the procedure of point-wise OPE result in Lemma~3.4. in \cite{yin2021near}. 
\end{proof}

\begin{remark}
Note the higher order term has dependence $H^4S$, which is somewhat unsatisfactory. We use the \emph{recursion-back} trick to further reduce it to $H^{3.5}S^{0.5}$. 
\end{remark}

\begin{lemma}\label{lem:final}
	Given $N>0$. There exists universal constants $C_1,C_2$ such that when $N\geq C_1H^2\log(HSA/\delta)$, then with probability $1-\delta$, $\forall h\in[H]$,
	\begin{equation}\label{eqn:final}
	\norm{\widehat{Q}^{\widehat{\pi}}_h-Q^{\widehat{\pi}}_h}_\infty
	\leq C_2\sqrt{\frac{H^3\log(HSA/\delta)}{N}}+C_2\frac{H^3\sqrt{HS}\log(HSA/\delta)}{N}.
	\end{equation}
	and 
	\begin{align*}
	\norm{\widehat{Q}^{{\pi}^\star}_h-Q^{{\pi}^\star}_h}_\infty
	\leq C_2\sqrt{\frac{H^3\log(HSA/\delta)}{N}}+C_2\frac{H^3\sqrt{HS}\log(HSA/\delta)}{N}.
	\end{align*}
\end{lemma}

\begin{proof}

	Note 
	\begin{equation}\label{eqn:as}
	\begin{aligned}
	\widehat{V}^\star_t(s)-V^\star_t(s)&:=\widehat{V}^{\widehat{\pi}^\star}_t(s)-V^{\pi^\star}_t(s)\\
	&=\widehat{V}^{\widehat{\pi}^\star}_t(s)-{V}^{\widehat{\pi}^\star}_t(s)+{V}^{\widehat{\pi}^\star}_t(s)-V^{\pi^\star}_t(s)\\
	&\leq \widehat{V}^{\widehat{\pi}^\star}_t(s)-{V}^{\widehat{\pi}^\star}_t(s)\leq \left|\widehat{V}^{\widehat{\pi}^\star}_t(s)-{V}^{\widehat{\pi}^\star}_t(s)\right|
	\end{aligned}
	\end{equation}
	and similarly $V^\star_t(s)-\widehat{V}^\star_t(s)\leq \left|\widehat{V}^{{\pi}^\star}_t(s)-{V}^{{\pi}^\star}_t(s)\right|$, therefore by Lemma~\ref{lem:inter_final} (and use $||\widehat{V}^\pi_t-V^\pi_t||_\infty\leq ||\widehat{Q}^\pi_t-Q^\pi_t||_\infty $), with probability $1-\delta$,
	\begin{align*}
	\Delta_s\leq 2\cdot\sup_{t}\norm{V^\star_t-\widehat{V}^\star_t}\leq 2\max_{\widehat{\pi}^\star,\pi^\star} \sup_t\norm{\widehat{V}^\pi_t-{V}^\pi_t}_\infty\leq C_2\sqrt{\frac{H^3\log(HSA/\delta)}{N}}+C_2\frac{H^4S\log(HSA/\delta)}{N},
	\end{align*}
	where the second inequality uses \eqref{eqn:as}. This replaces the crude bound of $O(\sqrt{H^4S\log(HSA/\delta)/N})$ for $\max_s\Delta_s$ (recall \eqref{eqn:delta_cb}) by $O(\sqrt{H^3\log(HSA/\delta)/N})$.
	
	Plug this back to \eqref{eqn:decomp} and repeat the similar analysis we end up with \eqref{eqn:final}. The second result is similarly proved.

\end{proof}

\subsection{Proof of Theorem~\ref{thm:optimal_upper_bound}}\label{sec:final_pf}
\begin{proof}[Proof of Theorem~\ref{thm:optimal_upper_bound}]
Note $n_{s,a}=\sum_{i=1}^n\sum_{t=1}^H \mathbf{1}[s^{(i)}_t=s,a^{(i)}_t=a]$, which implies
\[
\E[n_{s,a}]=\E\left[\sum_{i=1}^n\sum_{t=1}^H \mathbf{1}[s^{(i)}_t=s,a^{(i)}_t=a]\right]=n\cdot\sum_{t=1}^H d^\mu_t(s,a).
\]
Or equivalently, $n_{s,a}$ follows Binomial$(n,\sum_{t=1}^H d^\mu_t(s,a))$. Then apply the first result of Lemma~\ref{lem:chernoff_multiplicative} by taking $\theta=1/2$, we have when $n>1/d_m\cdot\log(HSA/\delta)$\footnote{The exact sufficient condition for applying Lemma~\ref{lem:chernoff_multiplicative} is $n>1/\sum_{t=1}^H d_t(s,a)\cdot\log(HSA/\delta)$ for all $s,a$. However, since $\sum_{t=1}^H d_t(s,a)\geq Hd_m\geq d_m$, our condition $n>1/d_m\cdot\log(HSA/\delta)$ used here is a much stronger version thus Lemma~\ref{lem:chernoff_multiplicative} apply. }, then with probability $1-\delta$, 
\[
n_{s,a}\geq \frac{1}{2}n\cdot\sum_{t=1}^H d^\mu_t(s,a),\quad \forall s\in\mathcal{S}, a\in\mathcal{A}.
\]
This further implies w.p. $1-\delta$, $n_{s,a}\geq \frac{1}{2}n\cdot\sum_{t=1}^H d^\mu_t(s,a)=\frac{1}{2}n\cdot H\cdot d^\mu(s,a)\geq \frac{1}{2}nH\cdot d_m$ and further ensures
\[
N:=\min_{s,a}n_{s,a}\geq \frac{1}{2}nH\cdot d_m.
\]
Finally, apply above to Lemma~\ref{lem:final}, we can get over with the condition on $N$ and obtain the stated result. 
\end{proof}

\section{Proof of minimax lower bound for model-based global uniform OPE}\label{sec:lower_proof}

\begin{proof}[Proof of Theorem~\ref{thm:tight_lower_bound}]
	In particular, we first focus on the case where $H=2$ and extend the result of $H=2$ to the general $H\geq 3$ at the end. 
	
	First of all, by Definition~\ref{def:model_based} let $\widehat{P}$ be the learned transition by certain model-based method. Since we assume $r_h$ is known and by convention $Q^\pi_{H+1}=0$ for any $\pi$, then by Bellman equation 
	\[
	\widehat{Q}^\pi_h = r_h + \widehat{P}^{\pi_{h+1}} \widehat{Q}^\pi_{h+1},\;\forall h\in[H].
	\]
	In particular, $ \widehat{Q}^\pi_{H+1}= {Q}^\pi_{H+1}=0$, and this implies 
	\[
	\widehat{Q}^\pi_{H}=r_H + \widehat{P}^{\pi_{H+1}} \widehat{Q}^\pi_{H+1}=r_H; \quad {Q}^\pi_{H}=r_H + {P}^{\pi_{H+1}} {Q}^\pi_{H+1}=r_H +0=r_H
	\]
	Now, again by definition of Bellman equation 
	\begin{align*}
	\widehat{Q}^\pi_{H-1} = r_{H-1}+ \widehat{P}^{\pi_{H}} \widehat{Q}^\pi_{H}= r_{H-1}+ \widehat{P}^{\pi_{H}} r_{H}\\
	{Q}^\pi_{H-1} = r_{H-1}+ {P}^{\pi_{H}} {Q}^\pi_{H}= r_{H-1}+ {P}^{\pi_{H}} r_{H}
	\end{align*}
	Therefore 
	\begin{align*}
	&\sup_{\pi\in\Pi_g}\norm{\widehat{Q}^\pi_{H-1}-{Q}^\pi_{H-1}}_\infty=\sup_{\pi\in\Pi_g}\norm{\left(\widehat{P}^{\pi_H}-{P}^{\pi_H}\right)r_H}_\infty\\
	=&\sup_{\pi\in\Pi_g}\norm{\left(\widehat{P}-{P}\right)r_H^{\pi_H}}_\infty=\sup_{\pi\in\Pi_g}\sup_{s,a}\left|\left(\widehat{P}(\cdot|s,a)-{P}(\cdot|s,a)\right)r_H^{\pi_H}\right|\\
	=&\sup_{s,a}\sup_{\pi\in\Pi_g}\left|\left(\widehat{P}(\cdot|s,a)-{P}(\cdot|s,a)\right)r_H^{\pi_H}\right|,\\
	\end{align*}
	where ${P}^{\pi_H}\in\R^{S\cdot A\times S\cdot A}, r_{H}\in\R^{S\cdot A}, {P}\in\R^{S\cdot A\times S}$ and $r_H^{\pi_H}\in\R^{S}$.
	Note $A\geq 2$, so we can choose an instance of $r_H$ as (there are at least two actions since $A\geq 2$)
	\[
	(r_H(s,a_1),r_H(s,a_2),...):=(1,0,...)\qquad \forall s\in\mathcal{S}.
	\]
	Above implies: if $\pi_H(s)=a_1$, then $r_H^{\pi_H}(s)=1$; if $\pi_H(s)=a_2$, then $r_H^{\pi_H}(s)=0$; ... 
	
	Hence, if $\Pi_g$ is the global deterministic policy class, then $r_H^{\pi_H}$ can traverse all the $S$-dimensional vectors with either $0$ or $1$ in each coordinate, which is exactly 
	\[
	\left\{r_H^{\pi_H}\in\R^S: \pi_H\in\Pi_g \right\}\supset\{0,1\}^S.
	\]
	
	Now let us first consider fixed $s,a$. Then with this choice of $r$, above implies
	\begin{align*}
	&\sup_{\pi\in\Pi_g}\left|\left(\widehat{P}(\cdot|s,a)-{P}(\cdot|s,a)\right)r_H^{\pi_H}\right|\geq\sup_{r\in\{0,1\}^S}\left|\left(\widehat{P}(\cdot|s,a)-{P}(\cdot|s,a)\right)\cdot r\right|\\
	=&\sup_{r\in\{0,1\}^S}\left|\sum_{i:r_i=1}\left(\widehat{P}(s_i|s,a)-{P}(s_i|s,a)\right)\right|
	\end{align*}
	Let $I_+:=\{i\in[S]:s.t.\;\;\widehat{P}(s_i|s,a)-{P}(s_i|s,a)>0\}$ be the set of indices where $\widehat{P}(s_i|s,a)-{P}(s_i|s,a)$ are positive and $I_-:=\{i\in[S]:s.t.\;\;\widehat{P}(s_i|s,a)-{P}(s_i|s,a)<0\}$ be the set of indices where $\widehat{P}(s_i|s,a)-{P}(s_i|s,a)$ are negative, then we further have 
	\begin{align*}
	&\sup_{r\in\{0,1\}^S}\left|\sum_{i:r_i=1}\left(\widehat{P}(s_i|s,a)-{P}(s_i|s,a)\right)\right|\\
	\geq &\max\bigg\{\left|\sum_{i\in I+}[\widehat{P}(s_i|s,a)-{P}(s_i|s,a)]\right|,\left|\sum_{i\in I-}[\widehat{P}(s_i|s,a)-{P}(s_i|s,a)]\right|\bigg\}\\
	=&\max\bigg\{\sum_{i\in I+}\left|\widehat{P}(s_i|s,a)-{P}(s_i|s,a)\right|,\sum_{i\in I-}\left|\widehat{P}(s_i|s,a)-{P}(s_i|s,a)\right|\bigg\}\\
	\end{align*}
	On the other hand, we have
	\[
	\sum_{i\in I+}\left|\widehat{P}(s_i|s,a)-{P}(s_i|s,a)\right|+\sum_{i\in I-}\left|\widehat{P}(s_i|s,a)-{P}(s_i|s,a)\right|=\norm{\widehat{P}(\cdot|s,a)-{P}(\cdot|s,a)}_1
	\]
	since $\widehat{P}(s_i|s,a)-{P}(s_i|s,a)=0$ contributes nothing to the $l_1$ norm.	Combine all the steps together, we obtain
	\begin{equation}\label{eqn:lower_b}
	\sup_{\pi\in\Pi_g}\norm{\widehat{Q}^\pi_{H-1}-{Q}^\pi_{H-1}}_\infty\geq \sup_{s,a}\frac{1}{2}\norm{\widehat{P}(\cdot|s,a)-{P}(\cdot|s,a)}_1{\explain{\geq}{{\circled{1}}} c\cdot \sup_{s,a}\sqrt{\frac{S}{n_{s,a}}}\explain{\geq}{\circled{2}} c'\sqrt{\frac{S}{nd_m}}}
	\end{equation}
	holds with constant probability $p$. Here $n_{s,a}=\sum_{h=1}^H\sum_{i=1}^n\mathbf{1}[s^{(i)}_h=s,a^{(i)}_h=a]$ is the number of data pieces visited $(s,a)$ in $n$ episodes. Now we explain how to obtain $\circled{1}$ and $\circled{2}$. In particular, we first explain $\circled{2}$.
	
	\textbf{Explain $\circled{2}$.} Recall we consider the case $H=2$. Then 
	\[
	\E\left[n_{s,a}\right]=\E\left[\sum_{h=1}^H\sum_{i=1}^n\mathbf{1}[s^{(i)}_h=s,a^{(i)}_h=a]\right]=n\sum_{i=1}^2\E\left[\mathbf{1}[s^{(1)}_h=s,a^{(1)}_h=a]\right]=n\sum_{h=1}^2 d^\mu_h(s,a)
	\]
	\emph{i.e.} $n_{s,a}$ is a Binomial random variable with parameter $n$ and $\sum_{h=1}^2 d^\mu_h(s,a)$. Then by Lemma~\ref{lem:chernoff_multiplicative}, choose $\theta=\frac{1}{2}$, apply the second result, we obtain when $n>(1/2d_m) \cdot\log(SA/\delta)$\footnote{By Lemma~\ref{lem:chernoff_multiplicative},the inequality holds as long as $n\geq 1/\sum_{h=1}^2 d^\mu_h(s,a) \log(SA/\delta)$, here $n>(1/2d_m) \cdot\log(SA/\delta)$ is a stronger sufficient condition. }, with probability $1-\delta$
	\[
	n_{s,a}\leq \frac{3}{2}n\cdot\sum_{h=1}^2 d^\mu_h(s,a), \quad \forall s,a
	\]
	Next, similar to the lower bound proof (Theorem~G.2.) of \cite{yin2021near}, we can choose $\mu$ and $M$ {(\textbf{near uniform} but not exact uniform)} such that $d^\mu_h(s,a)\leq C\cdot d_m$, which further implies $n_{s,a}\leq C\cdot n\cdot d_m,\;\forall s,a$. Summarize above we end up with the following Lemma:
	\begin{lemma}\label{lem:a_1}
		Suppose $n\geq (1/2d_m)\cdot \log(SA/\delta)$, then 
		\[
		\sup_{\mu,M}\P\left[\sqrt{\frac{1}{n_{s,a}}}\geq C\cdot \sqrt{\frac{1}{n\cdot d_m}},\;\forall s,a\right]\geq 1-\delta
		\]
	\end{lemma}

	\textbf{Explain $\circled{1}$.} To make the explanation rigorous, we first fix a pair $(s,a)$ and conditional on $n_{s,a}$. Then by a direct translation of Lemma~\ref{lem:l1_lower}, we have 
	\[
	\inf _{\widehat{P}} \sup _{P(\cdot|s,a) \in \mathcal{M}_{S}} \mathbb{P}\left[\|\widehat{P}(\cdot|s,a)-P(\cdot|s,a)\|_{1} \geq \frac{1}{8} \sqrt{\frac{e S}{2 n_{s,a}}}-o\left(\cdot\right) \middle| n_{s,a}\geq \frac{e}{32}S\right] \geq p,
	\]
	where $o(\cdot)$ is some exponentially small term in $S,n$. Now we consider everything under the condition $n\geq \frac{e}{32}\cdot S/d_m\log(SA/\delta)$. Next again take $\theta=1/2$, then by the first result of Lemma~\ref{lem:chernoff_multiplicative}, with probability $1-\delta$, 
	
	\[
	n_{s,a}\geq\frac{1}{2}n\cdot \sum_{h=1}^2 d^\mu_h(s,a)\geq n\cdot d_m\geq \frac{e}{32} S\log(SA/\delta).
	\]
	where the last inequality uses the condition $n\geq \frac{e}{32}\cdot S/d_m\log(SA/\delta)$. Therefore this implies
	\begin{align*}
	&\inf _{\widehat{P}} \sup _{P(\cdot|s,a) \in \mathcal{M}_{S}} \mathbb{P}\left[\|\widehat{P}(\cdot|s,a)-P(\cdot|s,a)\|_{1} \geq \frac{1}{8} \sqrt{\frac{e S}{2 n_{s,a}}}-o\left(\cdot\right) \right] \\
	=&\inf _{\widehat{P}} \sup _{P(\cdot|s,a) \in \mathcal{M}_{S}} \bigg(\mathbb{P}\left[\|\widehat{P}(\cdot|s,a)-P(\cdot|s,a)\|_{1} \geq \frac{1}{8} \sqrt{\frac{e S}{2 n_{s,a}}}-o\left(\cdot\right) \middle| n_{s,a}\geq \frac{e}{32}S\right]\cdot\P\left[n_{s,a}\geq \frac{e}{32}S\right] \\
	+&\mathbb{P}\left[\|\widehat{P}(\cdot|s,a)-P(\cdot|s,a)\|_{1} \geq \frac{1}{8} \sqrt{\frac{e S}{2 n_{s,a}}}-o\left(\cdot\right) \middle| n_{s,a}\leq \frac{e}{32}S\right]\cdot\P\left[n_{s,a}\leq \frac{e}{32}S\right]\bigg)\\
	\geq&\inf _{\widehat{P}} \sup _{P(\cdot|s,a) \in \mathcal{M}_{S}}\mathbb{P}\left[\|\widehat{P}(\cdot|s,a)-P(\cdot|s,a)\|_{1} \geq \frac{1}{8} \sqrt{\frac{e S}{2 n_{s,a}}}-o\left(\cdot\right) \middle| n_{s,a}\geq \frac{e}{32}S\right]\cdot\P\left[n_{s,a}\geq \frac{e}{32}S\right]\\
	\geq&p\cdot (1-\delta),
	\end{align*}
	To sum up, we have the following lemma:
	\begin{lemma}\label{lem:a_2}
		Let $n\geq \frac{e}{32} S/d_m\cdot \log(SA/\delta)$, then there exists a $0<p<1$,
		\[
		\inf _{\widehat{P}} \sup _{P(\cdot|s,a) \in \mathcal{M}_{S}} \mathbb{P}\left[\|\widehat{P}(\cdot|s,a)-P(\cdot|s,a)\|_{1} \geq \frac{1}{8} \sqrt{\frac{e S}{2 n_{s,a}}}-o\left(\cdot\right) \right]\geq p\cdot (1-\delta).
		\]
	\end{lemma}
	Now we finish the proof for the case where $H=2$. First note by \eqref{eqn:lower_b}, 
	\[
	\sup_{\pi\in\Pi_g}\norm{\widehat{Q}^\pi_{H-1}-{Q}^\pi_{H-1}}_\infty\geq \sup_{s,a}\frac{1}{2}\norm{\widehat{P}(\cdot|s,a)-{P}(\cdot|s,a)}_1
	\]
	with probability $1$, therefore by \eqref{eqn:lower_b}, Lemma~\ref{lem:a_1}, Lemma~\ref{lem:a_2} we have 
	\[
	\inf _{\widehat{P}} \sup _{P \in \mathcal{M}_{S}}\mathbb{P}\left[\sup_{\pi\in\Pi_g}\norm{\widehat{Q}^\pi_{H-1}-{Q}^\pi_{H-1}}_\infty\geq C\cdot\sqrt{\frac{S}{nd_m}}\right]\geq p(1-\delta)-\delta
	\]
	when $n\geq c\cdot S/d_m\log(SA/\delta)$ for some $c\geq \frac{e}{32}$. Above holds for any $\delta$. 
	
	It is easy to check $\frac{3}{2}\frac{p}{1+p}\leq 1$, therefore, in particular we set $\delta=\frac{3}{2}\frac{p}{1+p}$, direct calculation shows
	\[
	p(1-\delta)-\delta=\frac{p}{2},
	\]
	which completes the proof for $H=2$.
	
	\textbf{Extend to the general $H\geq 3$.}  
	
	\textbf{Step 1.} Similar to the decomposition in section~\ref{sec:dec}, we also have: 
	\[\hat{Q}^\pi_t-Q_t^\pi=\sum_{h=t}^{H} \hat{\Gamma}_{t+1: h}^{{\pi}}(\widehat{P}-P) {V}_{h+1}^{{\pi}}
	\]
	\textbf{Step 2.} Now choosing rewards recursively from back (with $||r_H||_\infty= c$ sufficiently small) such that $1\geq r_h\geq (||r_{h+1}||_\infty+\ldots+||r_H||_\infty)$ element-wisely $\forall h$, and $\max_{s,a}r_h(s,a)=3\min_{s,a}r_h(s,a)$. We denote $r_{h,max}:=\max_{s,a}r_h(s,a)$ and $r_{h,min}:=\min_{s,a}r_h(s,a)$. This choice guarantees: 
	\[
	r_{h,min}:=\min_{s,a}r_h(s,a)>||P^{\pi_{h+1}} r_{h+1}+..+P^{\pi_{h+1:H}}r_H||_\infty
	\]
	since $P^{\pi_h}$ is row-stochastic.

	\textbf{Step 3.} Next note ${V}_{h}^{{\pi}}= r_h+P^{\pi_{h+1}} r_{h+1}+..+P^{\pi_{h+1:H}}r_H$, so set $\left(r_{h}\left(s, a_{1}\right), r_{h}\left(s, a_{2}\right), \ldots\right):=(\max_{s,a}r_h(s,a),\min_{s,a}r_h(s,a), \ldots) $, then choose $\pi_h$ similar to the $H=2$ case and use \textbf{Step 1} and \textbf{Step 2}  we have
	 \begin{align*}
	 |(\hat{P}_{s,a}-P_{s,a}){V}_{h}^{{\pi}}|\geq &\frac{1}{2}||\hat{P}_{s,a}-P_{s,a}||_1\cdot(r_{h,max}-r_{h,min}-(P^{\pi_{h+1}} r_{h+1}+..+P^{\pi_{h+1:H}}r_H))\\
	 \geq &\frac{1}{2}||\hat{P}_{s,a}-P_{s,a}||_1 \cdot r_{h,min}\geq \frac{1}{2}||\hat{P}_{s,a}-P_{s,a}||_1\cdot c
	\end{align*}
	
	where the reasoning of the first inequality is similar to the case of $H=2$. Next use $\hat{\Gamma}_{t+1: h}^{{\pi}}$ is row-stochastic then from \textbf{Step 1} and take the sum we have 
	\[
	||\hat{Q}^\pi_1-Q_1^\pi||_\infty\geq\frac{1}{2} c\cdot H \min_{s,a}||\hat{P}_{s,a}-P_{s,a}||_1.
	\]
	for such choice of rewards and $\pi$.
	
	\textbf{Step 4.} However, in the above construction $c$ actually depends on $H$ due to the design $1\geq r_h\geq (||r_{h+1}||_\infty+\ldots+||r_H||_\infty)$. To get a universal constant $c$ we could use the bound $||\hat{Q}^\pi_1-Q_1^\pi||_\infty\gtrsim  r_{\frac{H}{2},min}\cdot\frac{H}{2} \min_{s,a}||\hat{P}_{s,a}-P_{s,a}||_1$ instead, where
	$r_{\frac{H}{2},min}$ in {Step 2} is universally lower bounded. Then we apply $||\hat{P}_{s,a}-P_{s,a}||_1\gtrsim \Omega(\sqrt{S/nd_m})$ to obtain the lower bound $\Omega(\sqrt{H^2S/nd_m})$.

\end{proof}

\begin{remark}
	We point out while our lower bound of $\Omega(H^2S/d_m\epsilon^2)$ for uniform OPE appears to be qualitatively similar to the lower bound of $\Omega(H^2S^2A/ \epsilon^2)$ derived for the online reward-free RL setting \citep{jin2020reward}, our result is not implied by theirs and cannot be proven by directly adapting their construction. 
Those two results are in principle different since: the result in \citep{jin2020reward} is learning-oriented where they define the problem class on $O(S)$ states and forcing $\Omega(SA/\epsilon^2)$ episodes in each state and end up with $O(S^2A/\epsilon^2)$ complexity; our result is evaluation-oriented where we need reduce the uniform evaluation problem to estimating probability distribution in $\ell_1$-error. The global uniform OPE and the reward-free setting are also different tasks (one cannot imply the other): the former deals with uniform convergence over all policies but with a fixed reward while the latter aims at learning simultaneously over all rewards.
\end{remark}

\section{Proof for optimal offline learning (Corollary~\ref{cor:opt_offline})}\label{sec:proof_offline}

\begin{proof}
This is a corollary of Theorem~\ref{thm:optimal_upper_bound}. Indeed, by taking $\widehat{\pi}=\widehat{\pi}^\star$, we first have 

\[
\norm{\widehat{V}_1^{\widehat{\pi}^\star}-V_1^{\widehat{\pi}^\star}}_\infty\leq \norm{\widehat{Q}_1^{\widehat{\pi}^\star}-Q_1^{\widehat{\pi}^\star}}_\infty\leq C\left[\sqrt{\frac{H^2 \iota}{n d_m}}+\frac{H^{2.5}S^{0.5}\iota}{nd_m}\right].
\]
Similar to the second result in Lemma~\ref{lem:final}, we also have

\[
\norm{\widehat{V}_1^{{\pi}^\star}-V_1^{{\pi}^\star}}_\infty\leq \norm{\widehat{Q}_1^{{\pi}^\star}-Q_1^{{\pi}^\star}}_\infty\leq C\left[\sqrt{\frac{H^2 \iota}{n d_m}}+\frac{H^{2.5}S^{0.5}\iota}{nd_m}\right].
\]
Next, recall the definition of $\widehat{\pi}\in\Pi_l$ that
\[
\norm{\widehat{V}_1^{\widehat{\pi}^\star}-\widehat{V}_1^{\widehat{\pi}}}_\infty\leq\epsilon_{\text {opt }},
\]
and Theorem~\ref{thm:optimal_upper_bound} again that
\[
\norm{\widehat{V}_1^{\widehat{\pi}}-V_1^{\widehat{\pi}}}_\infty\leq \norm{\widehat{Q}_1^{\widehat{\pi}}-Q_1^{\widehat{\pi}}}_\infty\leq C\left[\sqrt{\frac{H^2 \iota}{n d_m}}+\frac{H^{2.5}S^{0.5}\iota}{nd_m}\right].
\]
Therefore
\begin{align*}
{V}_1^{{\pi}^\star}-V_1^{\widehat{\pi}}&={V}_1^{{\pi}^\star}-V_1^{\widehat{\pi}^\star}+V_1^{\widehat{\pi}^\star}-V_1^{\widehat{\pi}}\\
&\leq \max_{\widehat{\pi}^\star,{\pi}^\star}\norm{\widehat{V}^\pi_1-V^\pi_1}_\infty+V_1^{\widehat{\pi}^\star}-V_1^{\widehat{\pi}}\\
&=\max_{\widehat{\pi}^\star,{\pi}^\star}\norm{\widehat{V}^\pi_1-V^\pi_1}_\infty+\left(V_1^{\widehat{\pi}^\star}-\widehat{V}_1^{\widehat{\pi}^\star}\right)+\left(\widehat{V}_1^{\widehat{\pi}^\star}-\widehat{V}_1^{\widehat{\pi}}\right)+\left(\widehat{V}_1^{\widehat{\pi}}-{V}_1^{\widehat{\pi}}\right)\\
&\leq 3C\left[\sqrt{\frac{H^2 \iota}{n d_m}}+\frac{H^{2.5}S^{0.5}\iota}{nd_m}\right] + \norm{\widehat{V}_1^{\widehat{\pi}^\star}-\widehat{V}_1^{\widehat{\pi}}}_\infty\cdot\mathbf{1}\\
&\leq 3C\left[\sqrt{\frac{H^2 \iota}{n d_m}}+\frac{H^{2.5}S^{0.5}\iota}{nd_m}\right] + \epsilon_{\text {opt }}\cdot\mathbf{1}.\\
\end{align*}
This completes the proof.

\end{proof}

\section{Proof for optimal offline Task-agnostic learning (Theorem~\ref{thm:offline_ta})}\label{sec:proof_ta}
\begin{proof}
	Recall the definition of offline task-agnostic setting, where $K$ tasks corresponds to $K$ MDPs $M_k=(\mathcal{S}, \mathcal{A}, P, r_k, H, d_1)$ with different mean reward functions $r_k$'s. Since the incremental number of rewards do not incur randomness, therefore by Corollary~\ref{cor:opt_offline}, choose $\widehat{\pi}_k=\widehat{\pi}_k^\star$ and apply a union bound we obtain with probability $1-\delta$,
	
		\begin{align*}
		\sup_{k\in[K]}||V_{1,M_k}^\star-V_{1,M_k}^{\widehat{\pi}_k^{\star}}||_\infty&\leq {O}\left[\sqrt{\frac{H^2 \log(HSAK/\delta)}{n d_m}}+\frac{H^{2.5}S^{0.5}\log(HSAK/\delta)}{nd_m}\right] \\
		&={O}\left[\sqrt{\frac{H^2 (\iota+\log(K))}{n d_m}}+\frac{H^{2.5}S^{0.5}(\iota+\log(K))}{nd_m}\right],
		\end{align*}
	which completes the proof.
	
\end{proof}

\begin{remark}
	We stress that Section~3 of \cite{zhang2020task} claims the definition of task-agnostic RL setting embraces one challenge that $r_k^{(i)}$'s are the observed random realizations and the need to accurately estimate mean rewards $r_k$'s causes the additional $\log(K)$ dependence. However, for offline case, this is not essential since, by straightforward calculation, estimating $r_k^{(i)}$'s accurately only requires $\tilde{O}(\log(K)/d_m\epsilon^2)$ samples, which is of lower order comparing to $\tilde{O}(H^2\log(K)/d_m\epsilon^2)$ learning bound. Therefore, in Definition~\ref{def:offline_ta} we do not incorporate the random version statement for reward $r_k$. 
\end{remark}

\subsection{Offline Learning in the Constrained MDPs (CMDP)}

Recently, there is a line of studies in the Constrained Markov Decision Processes (CMDP) (\emph{e.g.} \cite{ding2021provably}), where the MDP $M =(\mathcal{S}, \mathcal{A}, P, H, d_1)$. When the reward is set to be $r$, it defines the objective
function $V^\pi_r$ and there is another utility function $g$ that defines the constraint. To be concrete, the
objective formualted as:
\begin{equation}\label{eqn:CMDP}
\underset{\pi \in \Delta(\mathcal{A} \mid \mathcal{S}, H)}{\operatorname{maximize}} V_{r, 1}^{\pi}\left(x_{1}\right) \text { subject to } V_{g, 1}^{\pi}\left(x_{1}\right) \geq b
\end{equation}
where $b\in(0,H]$ is some constraint threshold. In addition, the formulation needs a Slater condition that: there
exists $\gamma>0$ and $\bar{\pi}\in\Delta(\mathcal{A}|\mathcal{S},H)$ such that $V^{\bar{\pi}}_{g,1}(x_1)\geq b+\gamma$.

Let $\pi^\star$ be the optimal solution that is compatible with the programming \eqref{eqn:CMDP} (note this is \textbf{different} from the optimal policy that maximizes $V_{r, 1}^{\pi}$ only), then by feasibility it satisfies $V^{\pi^\star}_{g,1}\geq b$. 

Now let $\hat{\pi}^\star$ be the solution of the empirical program:
\begin{equation}\label{eqn:empirical_CMDP}
\underset{\pi \in \Delta(\mathcal{A} \mid \mathcal{S}, H)}{\operatorname{maximize}} \widehat{V}_{r, 1}^{\pi}\left(x_{1}\right) \text { subject to } \widehat{V}_{g, 1}^{\pi}\left(x_{1}\right) \geq b
\end{equation}
then we can show $\hat{\pi}^\star$ is a near-optimal solution for \eqref{eqn:CMDP} via the local uniform convergence guarantee (Theorem~\ref{thm:optimal_upper_bound}).

Indeed, define a surrogate program:
\begin{equation}\label{eqn:surrogate_CMDP}
\underset{\pi \in \Delta(\mathcal{A} \mid \mathcal{S}, H)}{\operatorname{maximize}} \widehat{V}_{r, 1}^{\pi}\left(x_{1}\right) \text { subject to } {V}_{g, 1}^{\pi}\left(x_{1}\right) \geq b
\end{equation}
and let $\bar{\pi}^\star$ be the solution for \eqref{eqn:surrogate_CMDP}. Then apparently $\bar{\pi}^\star$ satisfies ${V}_{g, 1}^{\bar{\pi}^\star}\left(x_{1}\right) \geq b$. Moreover, we have 
\begin{align*}
V_{r,1}^{\pi^\star}-V_{r,1}^{\bar{\pi}^\star}=&V_{r,1}^{\pi^\star}-\widehat{V}_{r,1}^{\pi^\star}+\widehat{V}_{r,1}^{\pi^\star}-\widehat{V}_{r,1}^{\bar{\pi}^\star}+\widehat{V}_{r,1}^{\bar{\pi}^\star}-V_{r,1}^{\bar{\pi}^\star}\\
\leq&V_{r,1}^{\pi^\star}-\widehat{V}_{r,1}^{\pi^\star}+0+\widehat{V}_{r,1}^{\bar{\pi}^\star}-V_{r,1}^{\bar{\pi}^\star}\\
\leq& 2\sup_\pi|V_{r,1}^{\pi}-\widehat{V}_{r,1}^{\pi}|
\end{align*}

On the other hand, by local uniform convergence guarantee, $|V_{g,1}^{\pi}-\widehat{V}_{g,1}^{\pi}|\leq \tilde{O}(\sqrt{H^2/nd_m})$ for all $\pi$ in the $\sqrt{H/S}$-neighborhood of $\hat{\pi}^\star$ (w.r.t $g$). This implies
\[
V_{r,1}^{\pi^\star}-V_{r,1}^{\hat{\pi}^\star}\leq 2\sup_\pi|V_{r,1}^{\pi}-\widehat{V}_{r,1}^{\pi}|+\tilde{O}(\sqrt{H^2/nd_m})
\] 
and the violation of the constraint is bounded by $\tilde{O}(\sqrt{H^2/nd_m})$. This means any approach that solves \eqref{eqn:empirical_CMDP} is near-optimal for the original constrained MDP task given the uniform convergence guarantee.

\section{Proof for optimal offline Reward-free learning (Theorem~\ref{thm:offline_rf})}\label{sec:proof_rf}

Similar to before, recall $n_{s,a}=\sum_{h=1}^H\sum_{i=1}^n\mathbf{1}[s^{(i)}_h=s,a^{(i)}_h=a]$. We first prove two lemmas which essentially provide a version of \emph{``Maximal Bernstein inequality''}. We first fix a pair $(s,a)$ and then conditional on $n_{s,a}$.

\begin{lemma}\label{lem:error_bound}
We define $\epsilon_1=\sqrt{\frac{1}{HS^2}}$. Let  $\mathcal{G}=\{[i_1\epsilon_1,i_2\epsilon_1,\ldots,i_S\epsilon_1]^\top|i_1,i_2,\ldots,i_S\in\mathbb{Z}\}\cap [0,H]^S$ be the $S$-dimensional grid. Next define $\iota_1=\log[(\sqrt{H^3S^2})^S /\delta]$. Then with probability $1-\delta$,
\[
\left|(P_{s,a}-\widehat{P}_{s,a})w\right|\leq \sqrt{\frac{2\Var_{s,a}(w)\iota_1}{n_{s,a}}}+\frac{2H\iota_1}{3n_{s,a}},\quad \forall w\in\mathcal{G}.
\]
\end{lemma}
This is by the direct application of Bernstein inequality with a union bound, where the cardinality of $\mathcal{G}$ is
\[
\left(\frac{H}{\epsilon_1}\right)^S=\left(\sqrt{H^3S^2}\right)^S.
\]

\begin{lemma}\label{lem:maximal_bern}
	Let the $S$-dimensional grid be $\mathcal{G}=\{[i_1\epsilon_1,i_2\epsilon_1,\ldots,i_S\epsilon_1]^\top|i_1,i_2,\ldots,i_S\in\mathbb{Z}\}\cap [0,H]^S$ and define $\iota_1=\log[(\sqrt{H^3S^2})^S /\delta]$. It holds with probability $1-\delta$, 
	\[
	\left|(P_{s,a}-\widehat{P}_{s,a})v\right|\leq \sqrt{\frac{2\Var_{s,a}(v)\iota_1}{n_{s,a}}}+C\sqrt{\frac{\iota_1}{n_{s,a}HS}}+\frac{2H\iota_1}{3n_{s,a}},\quad \forall\; v\in[0,H]^S.
	\]
\end{lemma}

\begin{proof}
	Let $z:=\text{Proj}_\mathcal{G}(v)$. Then by design of $\mathcal{G}$ we have 
	\[
	\norm{z-v}_\infty\leq \epsilon_1=\sqrt{\frac{1}{HS^2}}.
	\]
	Therefore we obtain $\forall v\in[0,H]^S$,
	\begin{align*}
	\left|(P_{s,a}-\widehat{P}_{s,a})v\right|&\leq \left|(P_{s,a}-\widehat{P}_{s,a})(v-z)\right|+\left|(P_{s,a}-\widehat{P}_{s,a})z\right|\\
	&\leq \norm{P_{s,a}-\widehat{P}_{s,a}}_1\norm{z-v}_\infty+\left|(P_{s,a}-\widehat{P}_{s,a})z\right|\\
	&\leq c\sqrt{\frac{S}{n_{s,a}}}\norm{z-v}_\infty+\sqrt{\frac{2\Var_{s,a}(z)\iota_1}{n_{s,a}}}+\frac{2H\iota_1}{3n_{s,a}}\\
	&\leq c\sqrt{\frac{S}{n_{s,a}}}\norm{z-v}_\infty+\sqrt{\frac{2\norm{z-v}^2_\infty\iota_1}{n_{s,a}}}+\sqrt{\frac{2\Var_{s,a}(v)\iota_1}{n_{s,a}}}+\frac{2H\iota_1}{3n_{s,a}}\\
	&\leq C\sqrt{\frac{S\iota_1}{n_{s,a}}}\norm{z-v}_\infty+\sqrt{\frac{2\Var_{s,a}(v)\iota_1}{n_{s,a}}}+\frac{2H\iota_1}{3n_{s,a}}\\
	&\leq C\sqrt{\frac{\iota_1}{n_{s,a}HS}}+\sqrt{\frac{2\Var_{s,a}(v)\iota_1}{n_{s,a}}}+\frac{2H\iota_1}{3n_{s,a}}.\\
	\end{align*}
	where the third inequality uses Lemma~\ref{lem:error_bound} and Lemma~\ref{lem:l1_upper}.
\end{proof}
Then recall $N:=\min_{s,a}n_{s,a}$, by Lemma~\ref{lem:maximal_bern} and a union bound we obtain with probability $1-\delta$, element-wisely,
\begin{align}\label{eqn:maximal_bern}
\left|(P-\widehat{P})v\right|\leq C\cdot\left(\sqrt{\frac{2\Var_{s,a}(v)\iota_2}{N}}+2\sqrt{\frac{\iota_2}{N\cdot HS}}+\frac{2H\iota_2}{3N}\right)\cdot\mathbf{1},\quad \forall\; v\in[0,H]^S,
\end{align}
where $\iota_2=S\log(HSA/\delta)$.

\begin{remark}
	Equation~\ref{eqn:maximal_bern} is a form of maximal Bernstein inequality as it keeps validity for all $v\in[0,H]^S$. The price for this stronger result is the extra $S$ factor (coming from $\iota_2$) in the dominate term. 
	
\end{remark}

Now, for \emph{any} reward $r$, by (empirical) Bellman equation we have element-wisely:
\begin{align*}
\widehat{Q}_h^{\widehat{\pi}^\star}- {Q}_h^{\widehat{\pi}^\star}&=r_h+\widehat{P}^{\widehat{\pi}^\star_{h+1}} \widehat{Q}_{h+1}^{\widehat{\pi}^\star}-r_h-{P}^{\widehat{\pi}^\star_{h+1}} {Q}_{h+1}^{\widehat{\pi}^\star}\\
&=\left(\widehat{P}^{\widehat{\pi}^\star_{h+1}}-{P}^{\widehat{\pi}^\star_{h+1}}\right)\widehat{Q}_{h+1}^{\widehat{\pi}^\star} + {P}^{\widehat{\pi}^\star_{h+1}} \left(\widehat{Q}_{h+1}^{\widehat{\pi}^\star}-{Q}_{h+1}^{\widehat{\pi}^\star}\right)\\
&=\left(\widehat{P}-{P}\right)\widehat{V}_{h+1}^{\widehat{\pi}^\star}+ {P}^{\widehat{\pi}^\star_{h+1}} \left(\widehat{Q}_{h+1}^{\widehat{\pi}^\star}-{Q}_{h+1}^{\widehat{\pi}^\star}\right)\\
&=\ldots=\sum_{t=h}^H\Gamma_{h+1:t}^{\widehat{\pi}^\star}\left(\widehat{P}-{P}\right)\widehat{V}_{t+1}^{\widehat{\pi}^\star}\\
\end{align*}

where $\Gamma_{h+1:t}^\pi=\prod_{i=h+1}^t P^{\pi_i}$ is multi-step state-action transition and $\Gamma_{h+1:h}:=I$.

\textbf{Concentration on $\left(\widehat{P}-P\right)\hoV_h$.} Now by \eqref{eqn:maximal_bern}, we have the following:
\begin{equation}
\begin{aligned}
&\left( \hP_{s,a}-P_{s,a} \right)\hoV_h\\
&\leq C\cdot\left(\sqrt{\frac{2\Var_{s,a}(\hoV_h)\iota_2}{N}}+2\sqrt{\frac{\iota_2}{N\cdot HS}}+\frac{2H\iota_2}{3N}\right)\\
&\leq C\cdot\left(\sqrt{\frac{2\Var_{s,a}(V^{\widehat{\pi}^\star}_h)\iota_2}{N}}+2\sqrt{\frac{\iota_2}{N\cdot HS}}+\sqrt{\frac{2\iota_2}{N}}\cdot \norm{\widehat{V}_h^{\widehat{\pi}^\star}-{V}_h^{\widehat{\pi}^\star}}_\infty+\frac{2H\iota_2}{3N}\right)\\
&\leq C\cdot\left(\sqrt{\frac{2\Var_{s,a}(V^{\widehat{\pi}^\star}_h)\iota_2}{N}}+2\sqrt{\frac{\iota_2}{N\cdot HS}}+\sqrt{\frac{2\iota_2}{N}}\cdot H^2\sqrt{\frac{S}{N}}+\frac{2H\iota_2}{3N}\right)\\
&\leq C'\cdot\left(\sqrt{\frac{2\Var_{s,a}(V^{\widehat{\pi}^\star}_h)\iota_2}{N}}+2\sqrt{\frac{\iota_2}{N\cdot HS}}+{\frac{2H^2S\log(HSA/\delta)}{N}} \right),\\
\end{aligned}
\end{equation}

where the third inequality uses Lemma~\ref{lem:crude_u_b}\footnote{Note the use of Lemma~\ref{lem:crude_u_b} also works for any rewards since the only high probability result they used is for $||P-\hat{P}||_1$. Therefore conditional on the concentration for $||P-\hat{P}||_1$, the argument follows for any arbitrary reward as well.}. Then above implies

\begin{align*}
&\widehat{Q}_h^{\widehat{\pi}^\star}- {Q}_h^{\widehat{\pi}^\star}\\
\leq &C'\sum_{t=h}^H \Gamma_{h+1:t}^{\widehat{\pi}^\star}\cdot\left(\sqrt{\frac{2\Var_{s,a}(V^{\widehat{\pi}^\star}_h)\iota_2}{N}}+2\sqrt{\frac{\iota_2}{N\cdot HS}}+{\frac{2H^2S\log(HSA/\delta)}{N}} \right)\\
\leq& C'\left[\sum_{t=h}^H \Gamma_{h+1:t}^{\widehat{\pi}^\star}\cdot\sqrt{\frac{2\Var_{s,a}(V^{\widehat{\pi}^\star}_h)\iota_2}{N}}+2\sqrt{\frac{H\log(HSA/\delta)}{N}}+{\frac{2H^3S\log(HSA/\delta)}{N}}\right] \\
\leq& C'\left[\sqrt{\frac{2H^3S\log(HSA/\delta)}{N}}+2\sqrt{\frac{H\log(HSA/\delta)}{N}}+{\frac{2H^3S\log(HSA/\delta)}{N}}\right] \\
\leq &C^{\prime\prime}\left[\sqrt{\frac{H^3S\log(HSA/\delta)}{N}}+{\frac{H^3S\log(HSA/\delta)}{N}}\right] \\
\leq &O\left[\sqrt{\frac{H^2S\log(HSA/\delta)}{nd_m}}+{\frac{H^2S\log(HSA/\delta)}{nd_m}}\right], \\
\end{align*}
where the third inequality uses Lemma~\ref{lem:H3toH2} and the last one uses $N\geq \frac{1}{2} nd_m$ with high probability. Similar result holds for $\widehat{Q}_h^{{\pi}^\star}- {Q}_h^{{\pi}^\star}$. Combing those results we have reward-free bound (for any reward simultaneously)
\[
O\left[\sqrt{\frac{H^2S\log(HSA/\delta)}{nd_m}}+{\frac{H^2S\log(HSA/\delta)}{nd_m}}\right],
\]
which finishes the proof of Theorem~\ref{thm:offline_rf}. 

\begin{remark}
	Note above result is tight in both the dominate term AND the higher order term. Therefore this result cannot be further improved even in the higher order term.
\end{remark}

\section{Discussion of Section~\ref{sec:application}}\label{sec:app_application}

In this section we explain why Theorem~\ref{thm:offline_ta} and Theorem~\ref{thm:offline_rf} are optimal in the offline RL. 

We begin with the offline task-agnostic setting. For the exquisite readers who check the proof of Theorem~5 of \cite{zhang2020task}, the proving procedure of their lower bound follows the standard reduction to best-arm identification in multi-armed bandit problems. More specifically, to incorporate the dependence of $\log(K)$, they rely on the Theorem~10 of \cite{zhang2020task} (which is originated from \cite{mannor2004sample}) to show in order to be $(\epsilon,\delta)$-correct for a problem with $A$ arms and with $K$ tasks, it need at least $\Omega(\frac{A}{\epsilon^2}\log(\frac{K}{\delta}))$ samples. Such a result updates the Lemma~G.1. in \cite{yin2021nearoptimal} by the extra factor $\log(K)$ for the bandit problem with $K$ tasks. With no modification, the rest of the proof in Section~E of \cite{yin2021nearoptimal} follows though and one can end up with the lower bound $\Omega(H^2\log(K)/d_m\epsilon^2)$ over the problem class $\mathcal{M}_{d_{m}}:=\left\{(\mu, M) \mid \min _{t, s_{t}, a_{t}} d_{t}^{\mu}\left(s_{t}, a_{t}\right) \geq d_{m}\right\}$. The case for the offline reward-free setting is also similar. Indeed, the $\Omega(SA/\epsilon^2)$ trajectories in Lemma~4.2 in \cite{jin2020reward} could be replaced by $\Omega(1/d_m\epsilon^2)$ by choosing some hard \emph{near-uniform} behavior policy instances (see Section~E.2 in \cite{yin2021nearoptimal}) and the rest follows since by forcing $S$ such instances (Section~4.2 of \cite{jin2020reward}) to obtain $\Omega(S/d_m\epsilon^2)$ and create a chain of $\Omega(H)$ rewards for $\Omega(H^2S/d_m\epsilon^2)$.


\section{Proof of the linear MDP with anchor representations (Section~\ref{sec:discussion})}\label{sec:app_anchor}

Recall that we assume a generative oracle here. Sometimes we abuse the notation $\mathcal{K}$ for either anchor point set or the anchor point indices set. The meaning should be clear in each context.

\subsection{Model-based Plug-in Estimator for Anchor Representations}\label{subsec:method_anchor}

\textbf{Step 1:} For each $(s_k,a_k)$ where index $k\in\mathcal{K}$, collect $N$ samples from $P(\cdot|s_k,a_k)$; compute 
\[
\widehat{P}_\mathcal{K}(s'|s_k,a_k)=\frac{count(s,a,s')}{N};
\]
\textbf{Step 2:} Compute the linear combination coefficients $\lambda_k^{s,a}$ satisfies $\phi(s,a)=\sum_{k\in\mathcal{K}}\lambda_k^{s,a}\phi(s_k,a_k)$;

\textbf{Step 3:} Estimate transition distribution 
\[
\widehat{P}(s'|s,a)=\sum_{k\in\mathcal{K}} \lambda_k^{s,a}\cdot \widehat{P}_\mathcal{K}(s'|s_k,a_k).
\]

We need to check such $\widehat{P}(s'|s,a)$ is a valid distribution.
This is due to:

\begin{align*}
\sum_{k\in\mathcal{K}}{\lambda_k^{s,a}}=&\sum_{k\in\mathcal{K}}\sum_{s'}{\lambda_k^{s,a}}P(s'|s_k,a_k)=\sum_{s'}\sum_{k\in\mathcal{K}}{\lambda_k^{s,a}}P(s'|s_k,a_k)\\
=&\sum_{s'}\sum_{k\in\mathcal{K}}{\lambda_k^{s,a}}\langle \phi(s_k,a_k),\psi(s')\rangle=\sum_{s'}\langle \phi(s,a),\psi(s')\rangle=\sum_{s'}P(s'|s,a)=1\\
\end{align*}
and 
\begin{align*}
\sum_{s'}\widehat{P}(s'|s,a)=&\sum_{s'}\sum_{k \in \mathcal{K}} \lambda_{k}^{s, a} \widehat{P}_{\mathcal{K}}\left(s^{\prime} \mid s_{k}, a_{k}\right)=\sum_{k \in \mathcal{K}}\sum_{s'} \lambda_{k}^{s, a} \widehat{P}_{\mathcal{K}}\left(s^{\prime} \mid s_{k}, a_{k}\right)\\
=&\sum_{k \in \mathcal{K}}\lambda_{k}^{s, a} \frac{N}{N}=1.\\
\end{align*}

\textbf{Step 4:} construct empirical model $\widehat{M}=(\mathcal{S},\mathcal{A},\widehat{P},r,H)$ and output $\widehat{\pi}^\star=\argmax_\pi \widehat{V}_1^\pi$.

Similarly, Bellman (optimality) equations hold\footnote{We use the integral only to denote $\mathcal{S}$ could be exponentially large. }

\begin{align*}
V^\star_t(s)&=\max_a \left\{r(s,a)+\int_{s'}V^\star_{t+1}(s') dP(s'|s,a)\right\},\quad \forall s\in\mathcal{S}.\\
\widehat{V}^\star_t(s)&=\max_a \left\{r(s,a)+\int_{s'}\widehat{V}^\star_{t+1}(s')d\widehat{P}(s'|s,a)\right\},\quad \forall s\in\mathcal{S}.
\end{align*}

\subsection{General absorbing MDP} 

The definition of the general absorbing MDP remains the same: \emph{i.e.} for a fixed state $s$ and a sequence $\{u_t\}_{t=1}^H$, MDP $M_{s,\{u_t\}_{t=1}^H}$ is identical to $M$ for all states except $s$, and state $s$ is absorbing in the sense $P_{M_{s,\{u_t\}_{t=1}^H}}(s|s,a)=1$ for all $a$, and the instantaneous reward at time $t$ is $r_t(s,a)=u_t$ for all $a\in\mathcal{A}$. Also, we use the shorthand notation $V^\pi_{\{s,u_t\}}$ for $V^\pi_{s,{M_{s,\{u_t\}_{t=1}^H}}}$ and similarly for $Q_{\{s,u_t\}}$ and transition $P_{\{s,u_t\}}$. Then the following properties mirroring the Lemma~\ref{lem:absorb_value}  and Lemma~\ref{lem:q_diff} with nearly identical proof but for the integral version (which we skip):
\vspace{1em}
\begin{lemma}\label{lem:absorb_value_anchor}
	\[
	V^\star_{h,\{s,u_t\}}(s)=\sum_{t=h}^H u_t.
	\]
\end{lemma}

\begin{lemma}\label{lem:q_diff_anchor}
	Fix state $s$. For two different sequences $\{u_t\}_{t=1}^H$ and $\{u_t^\prime\}_{t=1}^H$, we have 
	\[
	\max_h\norm{Q^\star_{h,\{s,u_t\}}-Q^\star_{h,\{s,u_t^\prime\}}}_\infty\leq H\cdot\max_{t\in[H]}\left|u_t-u^\prime_t\right|.
	\]
	
\end{lemma}

\subsection{Singleton-absorbing MDP}

The well-definedness of singleton-absorbing MDP for linear MDP with anchor points depends on the following two lemmas whose proofs are still nearly identical to Lemma~\ref{lem:pos_diff} and Lemma~\ref{lem:smdp_prop} which we skip.
\begin{lemma}\label{lem:pos_diff_anchor}
	$ V^\star_t(s)-V^\star_{t+1}(s)\geq 0$, for all state $s\in\mathcal{S}$ and all $t\in[H]$.	
\end{lemma}

\begin{lemma}\label{lem:smdp_prop_anchor}
	Fix a state $s$. If we choose $u_t^\star:= V^\star_t(s)-V^\star_{t+1}(s)$ $\forall t\in[H]$, then we have the following vector form equation 
	\[
	V^\star_{h,\{s,u_t^\star\}}=V^\star_{h,M}\quad \forall h\in[H].
	\]
	Similarly, if we choose $\widehat{u}_t^\star:= \widehat{V}^\star_t(s)-\widehat{V}^\star_{t+1}(s)$, then $\widehat{V}^\star_{h,\{s,\hat{u}_t^\star\}}=\widehat{V}^\star_{h,M}$, $\forall h\in[H]$.
\end{lemma}

The singleton MDP we used is exactly $M_{s,\{u^\star_t\}_{t=1}^H}$ (or $\widehat{M}_{s,\{u^\star_t\}_{t=1}^H}$). 

\subsection{Proof for the optimal sample complexity}\label{sec:dec_anchor}

For $\widehat{\pi}^\star$, by (empirical) Bellman equation we have element-wisely:
\begin{align*}
\widehat{Q}_h^{\widehat{\pi}^\star}- {Q}_h^{\widehat{\pi}^\star}&=r_h+\widehat{P}^{\widehat{\pi}^\star_{h+1}} \widehat{Q}_{h+1}^{\widehat{\pi}^\star}-r_h-{P}^{\widehat{\pi}^\star_{h+1}} {Q}_{h+1}^{\widehat{\pi}^\star}\\
&=\left(\widehat{P}^{\widehat{\pi}^\star_{h+1}}-{P}^{\widehat{\pi}^\star_{h+1}}\right)\widehat{Q}_{h+1}^{\widehat{\pi}^\star} + {P}^{\widehat{\pi}^\star_{h+1}} \left(\widehat{Q}_{h+1}^{\widehat{\pi}^\star}-{Q}_{h+1}^{\widehat{\pi}^\star}\right)\\
&=\left(\widehat{P}-{P}\right)\widehat{V}_{h+1}^{\widehat{\pi}^\star}+ {P}^{\widehat{\pi}^\star_{h+1}} \left(\widehat{Q}_{h+1}^{\widehat{\pi}^\star}-{Q}_{h+1}^{\widehat{\pi}^\star}\right)\\
&=\ldots=\sum_{t=h}^H\Gamma_{h+1:t}^{\widehat{\pi}^\star}\left(\widehat{P}-{P}\right)\widehat{V}_{t+1}^{\widehat{\pi}^\star}\leq \underbrace{\sum_{t=h}^H\Gamma_{h+1:t}^{\widehat{\pi}^\star}\left|\left(\widehat{P}-{P}\right)\widehat{V}_{t+1}^{\widehat{\pi}^\star}\right|}_{(\text{$\star$})}\\
\end{align*}

where $\Gamma_{h+1:t}^{\pi^\star}=\prod_{i=h+1}^t P^{\pi^\star_i}$ is multi-step state-action transition and $\Gamma_{h+1:h}:=I$.

\subsection{Analyzing (\text{$\star$})} 
\textbf{Concentration on $\left(\widehat{P}-P\right)\hoV_h$.} Since $\hP$ aggregates all data from different step so that $\hP$ and $\hoV_h$ are on longer independent. We use the singleton-absorbing MDP ${M_{s,\{u_t^\star\}_{t=1}^H}}$ to handle the case (recall $u_t^\star:= V^\star_t(s)-V^\star_{t+1}(s)$ $\forall t\in[H]$).  \textbf{Here, we fix the state action $(s,a)\in\mathcal{K}$.} Then we have:
\begin{equation}\label{eqn:bern_absorb_anchor}
\begin{aligned}
&\left( \hP_{s,a}-P_{s,a} \right)\hoV_h=\left( \hP_{s,a}-P_{s,a} \right)\left(\hoV_h-\hoV_{h,\{s,u^\star_t\}}+\hoV_{h,\{s,u^\star_t\}}\right)\\
=& \left( \hP_{s,a}-P_{s,a} \right)\left(\hoV_h-\hoV_{h,\{s,u^\star_t\}}\right)+ \left( \hP_{s,a}-P_{s,a} \right)\hoV_{h,\{s,u^\star_t\}}\\
\leq& \norm{\hP_{s,a}-P_{s,a}}_1\norm{\hoV_h-\hoV_{h,\{s,u^\star_t\}}}_\infty+\sqrt{\frac{2\log(4/\delta)}{N}}\sqrt{\Var_{s,a}(\hoV_{h,\{s,u^\star_t\}})}+\frac{2H\log(1/\delta)}{3N}\\
\leq& \norm{\hP_{s,a}-P_{s,a}}_1\norm{\hoV_h-\hoV_{h,\{s,u^\star_t\}}}_\infty+\sqrt{\frac{2\log(4/\delta)}{N}}\left(\sqrt{\Var_{s,a}(\hoV_{h})}+\sqrt{\Var_{s,a}(\hoV_{h,\{s,u^\star_t\}}-\hoV_{h})}\right)+\frac{2H\log(1/\delta)}{3N}\\
\leq& \norm{\hP_{s,a}-P_{s,a}}_1\norm{\hoV_h-\hoV_{h,\{s,u^\star_t\}}}_\infty+\sqrt{\frac{2\log(4/\delta)}{N}}\left(\sqrt{\Var_{s,a}(\hoV_{h})}+\sqrt{\norm{\hoV_{h,\{s,u^\star_t\}}-\hoV_{h}}^2_\infty}\right)+\frac{2H\log(1/\delta)}{3N}\\
=& \left(\norm{\hP_{s,a}-P_{s,a}}_1+\sqrt{\frac{2\log(4/\delta)}{N}}\right)\norm{\hoV_h-\hoV_{h,\{s,u^\star_t\}}}_\infty+\sqrt{\frac{2\log(4/\delta)}{N}}\sqrt{\Var_{s,a}(\hoV_{h})}+\frac{2H\log(1/\delta)}{3N}\\
\end{aligned}
\end{equation}

where the first inequality uses Bernstein inequality (Lemma~\ref{lem:bernstein_ineq}) (\textbf{note here $P_{s,a}V=\int_{s'} V(s')dP(s'|s,a)$ since $\mathcal{S}$ could be continuous space, but this does not affect the availability of Bernstein inequality!}), the second inequality uses $\sqrt{\Var(\cdot)}$ is norm (norm triangle inequality). Now we treat $\norm{\hP_{s,a}-P_{s,a}}_1$ and $\norm{\hoV_h-\hoV_{h,\{s,u^\star_t\}}}_\infty$ separately.

\textbf{For $\norm{\hP_{s,a}-P_{s,a}}_1$.} Recall here $(s,a)\in\mathcal{K}$. By Lemma~\ref{lem:l1_upper} we obtain w.p. $1-\delta$
\begin{equation}\label{eqn:l1_anchor}
\norm{\hP_{s,a}-P_{s,a}}_1\leq C\sqrt{\frac{|\mathcal{S}|\log(1/\delta)}{N}}.
\end{equation}
where $C$ absorbs the higher order term and constants.

\textbf{For $\norm{\hoV_h-\hoV_{h,\{s,u^\star_t\}}}_\infty$.} Note if we set $\widehat{u}^\star_t=\widehat{V}^\star_t(s)-\widehat{V}^\star_{t+1}(s)$, then by Lemma~\ref{lem:smdp_prop_anchor}
\[
\hoV_h=\hoV_{h,\{s,\hat{u}^\star_t\}}
\]
Next since $\hoV_{h,\{s,\hat{u}^\star_t\}}(\tilde{s})=\max_a \widehat{Q}^\star_{h,\{s,\hat{u}^\star_t\}}(\tilde{s},a)$ $\forall \tilde{s}\in\mathcal{S}$, by generic inequality $|\max f-\max g|\leq \max|f-g|$, we have $|\hoV_{h,\{s,\hat{u}^\star_t\}}(\tilde{s})-\hoV_{h,\{s,{u}^\star_t\}}(\tilde{s})|\leq \max_a |\widehat{Q}^\star_{h,\{s,\hat{u}^\star_t\}}(\tilde{s},a)-\widehat{Q}^\star_{h,\{s,{u}^\star_t\}}(\tilde{s},a)|$, taking $\max_{\tilde{s}}$ on both sides, we obtain exactly
\[
\norm{\hoV_{h,\{s,\hat{u}^\star_t\}}-\hoV_{h,\{s,{u}^\star_t\}}}_\infty\leq \norm {\widehat{Q}^\star_{h,\{s,\hat{u}^\star_t\}}-\widehat{Q}^\star_{h,\{s,{u}^\star_t\}}}_\infty
\]
then by Lemma~\ref{lem:q_diff_anchor}, 
\begin{equation}\label{eqn:higher_diff_anchor}
\norm{\hoV_h-\hoV_{h,\{s,u^\star_t\}}}_\infty\leq \norm {\widehat{Q}^\star_{h,\{s,\hat{u}^\star_t\}}-\widehat{Q}^\star_{h,\{s,{u}^\star_t\}}}_\infty\leq H\max_t\left|\hat{u}^\star_t-u^\star_t\right|,
\end{equation}
Recall 
\[
\hat{u}^\star_t-u^\star_t=\widehat{V}^\star_t(s)-\widehat{V}^\star_{t+1}(s)-\left(V^\star_t(s)-V^\star_{t+1}(s)\right).
\]
Now we denote 
\[
\Delta_s:= \max_t|\hat{u}^\star_t-u^\star_t|=\max_t\left|\widehat{V}^\star_t(s)-\widehat{V}^\star_{t+1}(s)-\left(V^\star_t(s)-V^\star_{t+1}(s)\right)\right|,
\]
then $\Delta_s$ itself is a scalar and a random variable.

To sum up, by \eqref{eqn:bern_absorb_anchor}, \eqref{eqn:l1} and \eqref{eqn:higher_diff_anchor} and a union bound over all $(s,a)\in\mathcal{K}$ we have 
\begin{lemma}\label{lem:inter_bern_anchor}
	Fix $N>0$. With probability $1-\delta$, element-wisely, for all $h\in[H]$ and all $(s_{k},a_k)\in\mathcal{K}$,
	\begin{align*}
	\left|\left( \hP_{s_{k},a_k}-P_{s_{k},a_k} \right)\hoV_h\right|\leq& C\sqrt{\frac{|\mathcal{S}|\log(HK/\delta)}{N}}\cdot H\max_{s_k}\Delta_{s_k}\\
	+&\sqrt{\frac{2\log(4HK/\delta)}{N}}\sqrt{\Var_{P_{s_{k},a_k}}(\hoV_{h})}+\frac{2H\log(HK/\delta)}{3N}
	\end{align*}
	
\end{lemma}

Now we extend Lemma~\ref{lem:inter_bern_anchor} to any arbitrary $(s,a)$ by proving the following lemma:

\begin{lemma}[recover lemma]\label{lem:recover}
	For any function $V$ and any state action $(s,a)$, we have 
	\[
	\sum_{k\in\mathcal{K}}\lambda_{k}^{s,a}\sqrt{\Var_{P_{s_{k},a_k}}(V)}\leq \sqrt{\Var_{P_{s,a}}(V)}
	\]
\end{lemma}

\begin{proof}[Proof of Lemma~\ref{lem:recover}]
	Since $\lambda_{k}^{s,a}$ are probability distributions, by Jensen's inequality twice
	\begin{align*}
		&\sum_{k\in\mathcal{K}}\lambda_{k}^{s,a}\sqrt{\Var_{P_{s_{k},a_k}}(V)}\leq \sqrt{\sum_{k\in\mathcal{K}}\lambda_{k}^{s,a}\Var_{P_{s_{k},a_k}}(V)}\\
		=&\sqrt{\sum_{k\in\mathcal{K}}\lambda_{k}^{s,a}\Var_{P_{s_{k},a_k}}(V)}=\sqrt{\sum_{k\in\mathcal{K}}\lambda_{k}^{s,a}(P_{s_{k},a_k}V^2-(P_{s_{k},a_k}V)^2)}\\
		\leq &\sqrt{\sum_{k\in\mathcal{K}}\lambda_{k}^{s,a}\cdot P_{s_{k},a_k}V^2-(\sum_{k\in\mathcal{K}}\lambda_{k}^{s,a}P_{s_{k},a_k}V)^2}\\
		=&\sqrt{P_{s,a}V^2-(P_{s,a}V)^2}=\sqrt{\mathrm{Var}_{P_{s,a}}(V)},
	\end{align*}
	where we use $P_{s,a}=\sum_{k\in\mathcal{K}}\lambda_{k}^{s,a}P_{s_k,a_k}$.
\end{proof}

Therefore for all $(s,a)$, using Lemma~\ref{lem:inter_bern_anchor} and Lemma~\ref{lem:recover} we obtain w.p. $1-\delta$,
\begin{align*}
&\left|\left( \hP_{s,a}-P_{s,a} \right)\hoV_h\right|\leq \sum_{k\in\mathcal{K}}\lambda_{k}^{s,a}
\left|\left( \hP_{s_{k},a_k}-P_{s_{k},a_k} \right)\hoV_h\right|\\
&\leq C\sum_{k\in\mathcal{K}}\lambda_{k}^{s,a}\sqrt{\frac{S\log(HK/\delta)}{N}}\cdot H\max_{s_k}\Delta_{s_k}
+\sum_{k\in\mathcal{K}}\lambda_{k}^{s,a}\sqrt{\frac{2\log(4HK/\delta)}{N}}\sqrt{\Var_{P_{s_{k},a_k}}(\hoV_{h})}\\
&+\sum_{k\in\mathcal{K}}\lambda_{k}^{s,a}\frac{2H\log(HK/\delta)}{3N}\\
&=C\sqrt{\frac{S\log(HK/\delta)}{N}}\cdot H\max_{s_k}\Delta_{s_k}
+\sum_{k\in\mathcal{K}}\lambda_{k}^{s,a}\sqrt{\frac{2\log(4HK/\delta)}{N}}\sqrt{\Var_{P_{s_{k},a_k}}(\hoV_{h})}\\
&+\frac{2H\log(HK/\delta)}{3N}\\
&\leq C\sqrt{\frac{S\log(HK/\delta)}{N}}\cdot H\max_{s_k}\Delta_{s_k}
+\sqrt{\frac{2\log(4HK/\delta)}{N}}\sqrt{\Var_{P_{s,a}}(\hoV_{h})}+\frac{2H\log(HK/\delta)}{3N}\\
\end{align*}

Now plug above back into ($\star$), we receive:

\begin{align*}
&\left|\widehat{Q}^{\widehat{\pi}^\star}_h-Q^{\widehat{\pi}^\star}_h\right|\\
\leq& 
\sum_{t=h}^H\Gamma_{h+1:t}^{\widehat{\pi}^\star}\left(C\sqrt{\frac{S\log(HK/\delta)}{N}}\cdot H\max_{s_k}\Delta_{s_k}\cdot \mathbf{1}+\sqrt{\frac{2\log(4HK/\delta)}{N}}\sqrt{\Var_{P}(\hoV_{t+1})}+\frac{2H\log(HK/\delta)}{3N}\cdot\mathbf{1}\right)\\
\leq& \sum_{t=h}^H\Gamma_{h+1:t}^{\widehat{\pi}}\sqrt{\frac{2\log(4HK/\delta)}{N}}\sqrt{\Var_{P}(\hoV_{t+1})} +CH^2\sqrt{\frac{S\log(HK/\delta)}{N}}\cdot \max_s\Delta_s\cdot \mathbf{1}+\frac{2H^2\log(HK/\delta)}{3N}\cdot\mathbf{1}\\
\end{align*}

Similar to before, we get
\begin{equation}\label{eqn:var_anchor}
\begin{aligned}
&\sqrt{\Var_{P}(\hoV_{h})}:=\sqrt{\Var_{P}\left(\widehat{V}^{\widehat{\pi}^\star}_{h}\right)}\leq \sqrt{\Var_{P}\left(V^{\widehat{\pi}^\star}_h\right)}+\norm{\widehat{Q}^{\widehat{\pi}^\star}_{h}-{Q}^{\widehat{\pi}^\star}_h}_\infty\\
\end{aligned}
\end{equation}

Plug \eqref{eqn:var_anchor} back to above we obtain $\forall h\in[H]$,
\begin{equation}\label{eqn:decomp_anchor}
\begin{aligned}
&\left|\widehat{Q}^{\widehat{\pi}^\star}_h-Q^{\widehat{\pi}^\star}_h\right|
\leq \sqrt{\frac{2\log(4HK/\delta)}{N}}\sum_{t=h}^H\Gamma_{h+1:t}^{\widehat{\pi}^\star}\left(\sqrt{\Var_{P}\left(V^{\widehat{\pi}^\star}_{t+1}\right)}+\norm{\widehat{Q}^{\widehat{\pi}^\star}_{t+1}-{Q}^{\widehat{\pi}^\star}_{t+1}}_\infty\right)\\
& +CH^2\sqrt{\frac{S\log(HK/\delta)}{N}}\cdot \max_{s_k}\Delta_{s_k}\cdot \mathbf{1}+\frac{2H^2\log(HK/\delta)}{3N}\cdot\mathbf{1}\\
&\leq \sqrt{\frac{2\log(4HK/\delta)}{N}}\sum_{t=h}^H\Gamma_{h+1:t}^{\widehat{\pi}^\star}\sqrt{\Var_{P}\left(V^{\widehat{\pi}^\star}_{t+1}\right)}+\sqrt{\frac{2\log(4HK/\delta)}{N}}\sum_{t=h}^H\norm{\widehat{Q}^{\widehat{\pi}^\star}_{t+1}-{Q}^{\widehat{\pi}^\star}_{t+1}}_\infty\\
& +CH^2\sqrt{\frac{S\log(HK/\delta)}{N}}\cdot \max_{s_k}\Delta_{s_k}\cdot \mathbf{1}+\frac{2H^2\log(HK/\delta)}{3N}\cdot\mathbf{1}\\
\end{aligned}
\end{equation}
Apply Lemma~\ref{lem:H3toH2} and the (anchor version using recover lemma~\ref{lem:recover}) coarse uniform bound (Lemma~\ref{lem:crude_u_b}) we obtain the following lemma:

\begin{lemma}\label{lem:recursion_anchor}
	With probability $1-\delta$, for all $h\in[H]$, 
	\begin{align*}
	\norm{\widehat{Q}^{\widehat{\pi}^\star}_h-Q^{\widehat{\pi}^\star}_h}_\infty
	\leq \sqrt{\frac{C_0 H^3\log(4HK/\delta)}{N}}+\sqrt{\frac{2\log(4HK/\delta)}{N}}\sum_{t=h}^H\norm{\widehat{Q}^{\widehat{\pi}^\star}_{t+1}-{Q}^{\widehat{\pi}^\star}_{t+1}}_\infty +C'H^4\frac{S\log(HK/\delta)}{N}
	\end{align*}
\end{lemma}

\begin{proof}
	Since
	\begin{equation}\label{eqn:delta_cb_anchor}
	\begin{aligned}
	\Delta_{s_k}&:= \max_t|\hat{u}^\star_t-u^\star_t|=\max_t\left|\widehat{V}^\star_t(s_k)-\widehat{V}^\star_{t+1}(s_k)-\left(V^\star_t(s_k)-V^\star_{t+1}(s_k)\right)\right|\\
	&\leq 2\cdot \max_t \left|\widehat{V}^\star_t(s_k)-V^\star_t(s_k)\right|\\
	&= 2\cdot \max_t \left|\max_\pi \widehat{V}^\pi_t(s_k)-\max_\pi V^\pi_t(s_k)\right|\\
	&\leq 2\cdot \max_{\pi\in\Pi_g,t\in[H]}\norm{\widehat{V}_t^\pi-V_t^\pi}_\infty\leq C\cdot H^2 \sqrt{\frac{|\mathcal{S}|\log(HK/\delta)}{N}}
	\end{aligned}
	\end{equation}
	where the last inequality uses (the anchor version) of Lemma~\ref{lem:crude_u_b}.\footnote{Here the anchor version means for any $(s,a)$ we can apply $||\widehat{P}_{s,a}-P_{s,a}||_1=||\sum_{k}\lambda_{k}^{s,a}(\widehat{P}_{s,a}-P_{s,a})||_1\leq \sum_{k}\lambda_{k}^{s,a}||\widehat{P}_{s,a}-P_{s,a}||_1$.} Then apply union bound w.p. $1-\delta/2$, we obtain $\max_{s_k} \Delta_{s_k}\leq C\cdot H^2 \sqrt{\frac{|\mathcal{S}|\log(HK^2/\delta)}{N}}$. Note \eqref{eqn:decomp_anchor} holds with probability $1-\delta/2$, therefore plug above into \eqref{eqn:decomp_anchor} and uses Lemma~\ref{lem:H3toH2} and take $||\cdot||_\infty$ we obtain w.p. $1-\delta$, the result holds.
\end{proof}

\begin{lemma}\label{lem:inter_final_anchor}
	Given $N>0$. Define $C^{\prime\prime}:=2\cdot\max(\sqrt{C_0},C')$ where $C'$ is the universal constant in Lemma~\ref{lem:recursion_anchor}. When $N\geq 8H^2|\mathcal{S}|\log(4HK/\delta)$, then with probability $1-\delta$, $\forall h\in[H]$,
	\begin{equation}\label{eqn:inter_final_anchor}
	\begin{aligned}
	\norm{\widehat{Q}^{\widehat{\pi}^\star}_h-Q^{\widehat{\pi}^\star}_h}_\infty
	\leq C^{\prime\prime}\sqrt{\frac{H^3\log(4HK/\delta)}{N}}+C^{\prime\prime}\frac{H^4S\log(HK/\delta)}{N}.\\
	\norm{\widehat{Q}^{{\pi}^\star}_h-Q^{{\pi}^\star}_h}_\infty
	\leq C^{\prime\prime}\sqrt{\frac{H^3\log(4HK/\delta)}{N}}+C^{\prime\prime}\frac{H^4S\log(HK/\delta)}{N}.\\
	\end{aligned}
	\end{equation}
	
\end{lemma}

\begin{proof}
	The proof is the same as  Lemma~\ref{lem:inter_final}.
\end{proof}

\begin{remark}
	Note the higher order term has dependence $H^4S$. Use the same \emph{self-bounding} trick, we can reduce it to $H^{3.5}S^{0.5}$. 
\end{remark}

\begin{lemma}\label{lem:final_anchor}
	Given $N>0$. There exists universal constants $C_1,C_2$ such that when $N\geq C_1H^2|\mathcal{S}|\log(HK/\delta)$, then with probability $1-\delta$, $\forall h\in[H]$,
	\begin{equation}\label{eqn:final_anchor}
	\norm{\widehat{Q}^{\widehat{\pi}^\star}_h-Q^{\widehat{\pi}^\star}_h}_\infty
	\leq C_2\sqrt{\frac{H^3\log(HK/\delta)}{N}}+C_2\frac{H^3\sqrt{HS}\log(HK/\delta)}{N}.
	\end{equation}
	and 
	\begin{align*}
	\norm{\widehat{Q}^{{\pi}^\star}_h-Q^{{\pi}^\star}_h}_\infty
	\leq C_2\sqrt{\frac{H^3\log(HK/\delta)}{N}}+C_2\frac{H^3\sqrt{HS}\log(HK/\delta)}{N}.
	\end{align*}
\end{lemma}

\begin{proof}
	The proof is similar to Lemma~\ref{lem:final}.
\end{proof}

\subsection{Proof of Theorem~\ref{thm:anchor}}
\begin{proof}
By the direct computing of the suboptimality, 
\[
Q^\star_1-Q^{\widehat{\pi}^\star}_1=Q^\star_1-\widehat{Q}^{{\pi}^\star}_1+\widehat{Q}^{{\pi}^\star}_1-\widehat{Q}^{\widehat{\pi}^\star}_1+\widehat{Q}^{\widehat{\pi}^\star}_1-Q^{\widehat{\pi}^\star}_1\leq |Q^\star_1-\widehat{Q}^{{\pi}^\star}_1|+|\widehat{Q}^{\widehat{\pi}^\star}_1-Q^{\widehat{\pi}^\star}_1|,
\]
then by Lemma~\ref{lem:final_anchor} we can finish the proof.
\end{proof}

\subsection{Take-away in the linear MDP with anchor setting.}

Under the setting $S$ could be exponential large, $\mathcal{A}$ could be infinite (or even continuous space), with anchor representations ($K\ll |\mathcal{S}|$), our Theorem~\ref{thm:anchor} has order $\widetilde{O}(\sqrt{H^3/N})$ when $N$ is sufficiently large. This translate to $N=\widetilde{O}(H^3/\epsilon^2)$ and the total sample used is $KN=\widetilde{O}(KH^3/\epsilon^2)$. This improves the total complexity $\widetilde{O}(KH^4/\epsilon^2)$ in \cite{cui2020plug} and is optimal.

\section{The computational efficiency for the model-based offline plug-in estimators}

For completeness, we discuss the computational and storage aspect of our model-based method. Its computational cost is $\widetilde{O}(H^4/d_m\epsilon^2)$ for computing $\widehat{P}$, the same as its sample complexity in steps ($H$ steps is an episode), and running value iteration causes $O(HS^2A)$ time (here we assume the bit complexity $L(P,r,H)=1$, see \cite{agarwal2019reinforcement} Section~1.3). The total computational complexity is $\widetilde{O}(H^4/d_m\epsilon^2)+O(HS^2A)$. The memory cost is $O(HS^2A)$.

\section{Assisting lemmas}

\begin{lemma}[Multiplicative Chernoff bound \cite{chernoff1952measure}]\label{lem:chernoff_multiplicative}
	Let $X$ be a Binomial random variable with parameter $p,n$. For any $1\geq\theta>0$, we have that 
	$$
	\mathbb{P}[X<(1-\theta) p n]<e^{-\frac{\theta^{2} p n}{2}} . \quad \text { and } \quad \mathbb{P}[X \geq(1+\theta) p n]<e^{-\frac{\theta^{2} p n}{3}}
	$$
\end{lemma}

\begin{lemma}[Hoeffding’s Inequality \cite{sridharan2002gentle}]\label{lem:hoeffding_ineq}
	Let $x_1,...,x_n$ be independent bounded random variables such that $\E[x_i]=0$ and $|x_i|\leq \xi_i$ with probability $1$. Then for any $\epsilon >0$ we have 
	$$
	\P\left( \frac{1}{n}\sum_{i=1}^nx_i\geq \epsilon\right) \leq e^{-\frac{2n^2\epsilon^2}{\sum_{i=1}^n\xi_i^2}}.
	$$
\end{lemma}

\begin{lemma}[Bernstein’s Inequality]\label{lem:bernstein_ineq}
	Let $x_1,...,x_n$ be independent bounded random variables such that $\E[x_i]=0$ and $|x_i|\leq \xi$ with probability $1$. Let $\sigma^2 = \frac{1}{n}\sum_{i=1}^n \mathrm{Var}[x_i]$, then with probability $1-\delta$ we have 
	$$
	\frac{1}{n}\sum_{i=1}^n x_i\leq \sqrt{\frac{2\sigma^2\cdot\log(1/\delta)}{n}}+\frac{2\xi}{3n}\log(1/\delta)
	$$
\end{lemma}

\begin{lemma}[Freedman's inequality \cite{tropp2011freedman}]\label{lem:freedman}
	Let $X$ be the martingale associated with a filter $\mathcal{F}$ (\textit{i.e.} $X_i=\E[X|\mathcal{F}_i]$) satisfying $|X_i-X_{i-1}|\leq M$ for $i=1,...,n$. Denote $W:=\sum_{i=1}^n\Var(X_i|\mathcal{F}_{i-1})$  then we have 
	\[
	\P(|X-\E[X]|\geq\epsilon,W\leq \sigma^2)\leq 2 e^{-\frac{\epsilon^2}{2(\sigma^2+M\epsilon/3)}}.
	\]
	Or in other words, with probability $1-\delta$,
	\[
	|X-\E[X]|\leq \sqrt{{8\sigma^2\cdot\log(1/\delta)}}+\frac{2M}{3}\cdot\log(1/\delta), \quad\text{Or} \quad W\geq \sigma^2.
	\]
\end{lemma}

\begin{lemma}[Sum of expectation of conditional variance of value; Lemma~F.3 of \cite{yin2021near}]\label{lem:H3toH2}

	\begin{align*}
	&\operatorname{Var}_{\pi}\left[\sum_{t=h}^{H} r_{t}^{(1)} \mid s_{h}^{(1)}=s_{h}, a_{h}^{(1)}=a_{h}\right]\\
	=&\sum_{t=h}^{H}\Bigg(\mathbb{E}_{\pi}\left[\operatorname{Var}\left[r_{t}^{(1)}+V_{t+1}^{\pi}\left(s_{t+1}^{(1)}\right) \mid s_{t}^{(1)}, a_{t}^{(1)}\right] \mid s_{h}^{(1)}=s_{h}, a_{h}^{(1)}=a_{h}\right]\\
	+&\mathbb{E}_{\pi}\left[\operatorname{Var}\left[\mathbb{E}\left[r_{t}^{(1)}+V_{t+1}^{\pi}\left(s_{t+1}^{(1)}\right) \mid s_{t}^{(1)}, a_{t}^{(1)}\right] \mid s_{t}^{(1)}\right] \mid s_{h}^{(1)}=s_{h}, a_{h}^{(1)}=a_{h}\right]\Bigg)
	\end{align*}
	
	By apply above, one can show
	\[
	\sum_{t=h}^H\Gamma_{h+1:t}^{{\pi}}\sqrt{\Var_{P}\left(V^{{\pi}}_{t+1}\right)}\leq \sqrt{(H-h)^3}\cdot\mathbf{1}.
	\]
\end{lemma}

\begin{remark}
	The infinite horizon discounted setting counterpart result is $(I-\gamma P^\pi)^{-1}\sigma_{V^\pi}\leq (1-\gamma)^{-3/2}$.
\end{remark}

\subsection{Minimax rate of discrete distributions under $l_1$ loss.}

This Section provides the minimax rate for $\norm{\widehat{P}-P}_1$ for any model-based algorithms and is based on \cite{han2015minimax}. Let $P$ be $S$ dimensional distribution. 

\begin{lemma}[Minimax lower bound for $\norm{\widehat{P}-P}_1$]\label{lem:l1_lower}
	Let $n$ be the number of data-points sampled from $P$. If $n>\frac{e}{32}S$, then there exists a constant $p>0$, such that
	\[
	\inf_{\widehat{P}}\sup_{P\in\mathcal{M}_S} \P\left[\norm{\widehat{P}-P}_1\geq \frac{1}{8}\sqrt{\frac{eS}{2n}}-o(e^{-n})-o(e^{-S})\right]\geq p,
	\]
	where $\mathcal{M}_S$ denotes the set of distributions with support size $S$ and the infimum is taken over \textbf{ALL} estimators.
\end{lemma} 

\begin{remark}
	Note the $\widehat{P}$ in above carries over all estimators but not just empirical estimator. This provides the minimax result.
\end{remark}

\begin{proof}
	The proof comes from Theorem~2 of \cite{han2015minimax}, where we pick $\zeta=1$. Note they establish the minimax result for $\mathbb{E}_{P}\|\hat{P}-P\|_{1}$. However, by a simple contradiction we can get the above. Indeed, suppose 
	\[
	\inf_{\widehat{P}}\sup_{P\in\mathcal{M}_S}\P\left[\norm{\widehat{P}-P}_1< \frac{1}{8}\sqrt{\frac{eS}{2n}}-o(e^{-n})-o(e^{-S})\right]=1,
	\]
	then this implies $\inf_{\widehat{P}}\sup_{P\in\mathcal{M}_S}\mathbb{E}_{P}\|\hat{P}-P\|_{1}<\frac{1}{8}\sqrt{\frac{eS}{2n}}-o(e^{-n})-o(e^{-S})$ which contradicts Theorem~2 of \cite{han2015minimax}.
\end{proof}

\begin{lemma}[Upper bound for $\norm{\widehat{P}-P}_1$]\label{lem:l1_upper}
	Let $n$ be the number of data-points sampled from $P$. Then with probability $1-\delta$
	\[
	 \norm{\widehat{P}-P}_1\leq C\left(\sqrt{\frac{S\log(S/\delta)}{n}}+\frac{S\log(S/\delta)}{n}\right)
	\]
	for any $P\in\mathcal{M}_S$. Here $\widehat{P}$ is the empirical (MLE) estimator.
\end{lemma} 

\begin{proof}
	First fix a state $s$. Let $X_i= \mathbf{1}[s_i=s]$, then $X_i\sim Bern(p_s(1-p_s))$ and $X_s=\sum_{i=1}^n X_i\sim Binomial(n,p_i)$. By Bernstein inequality,
	\[
	\left|\frac{X_s}{n}-P_s\right|\leq \sqrt{\frac{2p_s(1-p_s)\log(1/\delta)}{n}}+\frac{3}{n}\log(1/\delta)
	\]
	Apply a union bound we obtain w.p. $1-\delta$
	\[
	\left|\frac{X_s}{n}-P_s\right|\leq \sqrt{\frac{2p_s(1-p_s)\log(S/\delta)}{n}}+\frac{3}{n}\log(S/\delta)\quad \forall s \in\mathcal{S} 
	\]
	which implies
	\begin{align*}
	\norm{\widehat{P}-P}_1&=\sum_{s\in\mathcal{S}}\left|\frac{X_s}{n}-P_s\right|\\
	&\leq \sum_{s\in\mathcal{S}}\sqrt{\frac{2p_s(1-p_s)\log(S/\delta)}{n}}+\frac{3S}{n}\log(S/\delta)\\
	&=\sqrt{\frac{1}{n}}\sum_{s\in\mathcal{S}}\frac{1}{S}\cdot \sqrt{2S^2p_s(1-p_s)\log(S/\delta)}+\frac{3S}{n}\log(S/\delta)\\
	&\leq \sqrt{\frac{1}{n}}\sqrt{2S^2\cdot\frac{\sum_{s\in\mathcal{S}}p_s}{S}\left(1-\frac{\sum_{s\in\mathcal{S}}p_s}{S}\right)\log(S/\delta)}+\frac{3S}{n}\log(S/\delta)\\
	&=\sqrt{\frac{2(S-1)\log(S/\delta)}{n}}+\frac{3S}{n}\log(S/\delta).\\
	\end{align*}
	where the last inequality uses the concavity of $\sqrt{x(1-x)}$.
	
	Finally, we can absorb the higher order term using the mild condition $n>c\cdot S\log(S/\delta)$.
\end{proof}

\subsection{A crude uniform convergence bound}
Here we provide a crude bound for $\sup_{\pi\in\Pi_g}\norm{\widehat{V}^\pi_1-V^\pi_1}_\infty$, which is the finite horizon counterpart of Section~2.2 of \cite{jiang2018notes} and is a form of simulation lemma.

\begin{lemma}[Crude bound by Simulation Lemma]\label{lem:crude_u_b}
	Fix $N>0$ to be number of samples for each coordinates. Recall $\Pi_g$ is the global policy class. Then w.p. $1-\delta$,
	\[
	\sup_{\pi\in\Pi_g,h\in[H]}\norm{\widehat{Q}_h^{{\pi}}- {Q}_h^{{\pi}}}_\infty\leq C\cdot H^2 \sqrt{\frac{S\log(SA/\delta)}{N}},
	\]
	which further implies 
	\[
	\sup_{\pi\in\Pi_g,h\in[H]}\norm{\widehat{V}_h^{{\pi}}- {V}_h^{{\pi}}}_\infty\leq C\cdot H^2 \sqrt{\frac{S\log(SA/\delta)}{N}},
	\]
\end{lemma}

\begin{proof}
	\begin{align*}
	\widehat{Q}_h^{{\pi}}- {Q}_h^{{\pi}}&=r_h+\widehat{P}^{{\pi}_{h+1}} \widehat{Q}_{h+1}^{{\pi}}-r_h-{P}^{{\pi}_{h+1}} {Q}_{h+1}^{{\pi}}\\
	&=\left(\widehat{P}^{{\pi}_{h+1}}-{P}^{{\pi}_{h+1}}\right)\widehat{Q}_{h+1}^{{\pi}} + {P}^{{\pi}_{h+1}} \left({Q}_{h+1}^{{\pi}}-{Q}_{h+1}^{{\pi}}\right)\\
	&=\left(\widehat{P}-{P}\right)\widehat{V}_{h+1}^{{\pi}}+ {P}^{{\pi}_{h+1}} \left(\widehat{Q}_{h+1}^{{\pi}}-{Q}_{h+1}^{{\pi}}\right)\\
	&=\ldots=\sum_{t=h}^H\Gamma_{h+1:t}^{{\pi}}\left(\widehat{P}-{P}\right)\widehat{V}_{t+1}^{{\pi}}\\
	&\leq \sum_{t=h}^H\Gamma_{h+1:t}^{{\pi}}\left|\left(\widehat{P}-{P}\right)\widehat{V}_{t+1}^{{\pi}}\right|\\
	&\leq \sum_{t=h}^H 1\cdot\max_{s,a}\norm{(\widehat{P}-P)(\cdot|s,a)}_1\cdot  \norm{\widehat{V}_{t+1}^{{\pi}}}_\infty\cdot\mathbf{1}\\
	&\leq H^2\cdot \max_{s,a}\norm{(\widehat{P}-P)(\cdot|s,a)}_1\cdot\mathbf{1}\leq C\cdot H^2 \sqrt{\frac{S\log(SA/\delta)}{N}}\mathbf{1}\\
	\end{align*}
	with probability $1-\delta$, where the last inequality is by Lemma~\ref{lem:l1_upper}. By symmetry and taking the $\norm{\cdot}_\infty$, we obtain w.p. $1-\delta$
	\[
	\sup_{\pi\in\Pi_g,h\in[H]}\norm{\widehat{Q}_h^{{\pi}}- {Q}_h^{{\pi}}}_\infty\leq C\cdot H^2 \sqrt{\frac{S\log(SA/\delta)}{N}}.
	\] 
	The above holds for $\forall \pi\in\Pi_g$ since Lemma~\ref{lem:l1_upper} acts on $\norm{\widehat{P}-P}_1$ and is irrelevant to $\pi$. 
\end{proof}

\end{document}